\pdfoutput=1
\documentclass[lettersize,journal]{IEEEtran}
\usepackage{amsmath,amsfonts}
\usepackage{array}
\usepackage[caption=false,font=normalsize,labelfont=bf,textfont=normalfont]{subfig}
\usepackage{textcomp}
\usepackage{stfloats}
\usepackage{url}
\usepackage{verbatim}
\usepackage{graphicx}
\usepackage[hidelinks]{hyperref}
\RequirePackage[dvipsnames]{xcolor}

\usepackage[utf8]{inputenc}

\newcommand{\term}[1]{\mbox{\emph{#1}}}

\usepackage{amsthm}
\usepackage{mathtools}
\usepackage{pifont} 

\usepackage[ruled,vlined,nofillcomment]{algorithm2e}

\SetCommentSty{MyCommentStyle}
\SetProgSty{}
\SetAlCapHSkip{0pt}
\usepackage{tikz}
\usetikzlibrary{arrows}
\usetikzlibrary{arrows.meta} 
\usetikzlibrary{fit}
\usetikzlibrary{decorations.pathreplacing}
\usetikzlibrary{shapes}
\usetikzlibrary{calc}

\definecolor{tab20blue}{HTML}{1f77b4}
\definecolor{tab20red}{HTML}{d62728}
\definecolor{tab20green}{HTML}{2ca02c}
\definecolor{tab20orange}{HTML}{ff7f0e}
\definecolor{tab20purple}{HTML}{9467bd}
\definecolor{tab20gray}{HTML}{7f7f7f}

\usepackage{pgf}

\usepackage{adjustbox}
\usepackage{multirow}
\usepackage{array}

\usepackage{pmboxdraw} 

\usepackage{titlesec}
\titleformat{\subsubsection}[runin]
  {\normalfont\bfseries}
  {}
  {0pt}
  {}
  []

\titleformat{\paragraph}[runin]
  {\normalfont\itshape}
  {}
  {0pt}
  {}
  []

\titlespacing*{\subsubsection}{0pt}{*1}{0.6em}
\titlespacing{\paragraph}{\parindent}{*1}{0.6em}
  
\let\Subsubsection\subsubsection
\let\Paragraph\paragraph

\usepackage{chngcntr} 

\RequirePackage{xspace}
\RequirePackage{amssymb}
\RequirePackage{booktabs}
\RequirePackage[numbers,sort&compress]{natbib}
\usepackage[capitalize]{cleveref}
\crefname{section}{Sec.}{Secs.}
\Crefname{section}{Section}{Sections}
\Crefname{table}{Table}{Tables}
\crefname{table}{Tab.}{Tabs.}
\crefname{appendix}{Suppl.}{Suppls.}
\Crefname{appendix}{Supplement}{Supplements}
\crefname{algorithm}{Alg.}{Alg.}
\Crefname{algorithm}{Algorithm}{Algorithms}

\newcounter{figpoint}[figure]

\crefformat{figpoint}{(#2#1#3)}
\Crefformat{figpoint}{(#2#1#3)}

\usepackage{titletoc}
\usepackage{placeins}
\usepackage{makecell}
\usepackage{longtable}
\newcolumntype{C}{>{\centering\arraybackslash}p{65pt}}

\usepackage[normalem]{ulem}

\makeatletter
\newlength{\caplegendskip}
\DeclareRobustCommand{\caplegend}[1]{%
  \par
  \begingroup
    \normalfont\footnotesize
    \leftskip\z@skip \rightskip\z@skip \parfillskip\@flushglue
    \parindent\z@
    \vskip\caplegendskip
    #1\par
  \endgroup}
\makeatother

\def\eg{\emph{e.g.}\xspace} 
\def\ie{\emph{i.e.}\xspace}

\def\wrt{w.r.t.\xspace}

\def\st{\emph{s.t.}\xspace} 

\newcommand\SA{\mathcal{A}}
\newcommand\SB{\mathcal{B}}

\newcommand\SD{\mathcal{D}}

\newcommand\SQ{\mathcal{Q}}
\newcommand\SR{\mathcal{R}}

\newcommand\SV{\mathcal{V}}

\newcommand\BN{\mathbb{N}}

\newcommand\BQ{\mathbb{Q}}
\newcommand\BR{\mathbb{R}}

\newcommand\FC{\mathfrak{C}}

\newcommand\tV{\widetilde{V}}
\newcommand\td{\widetilde{d}}
\newcommand\tn{\widetilde{n}}
\newcommand\tC{\widetilde{C}}
\newcommand\tQ{\widetilde{Q}}
\newcommand\tR{\widetilde{R}}
\newcommand\tSQ{\widetilde{\SQ}}
\newcommand\tSR{\widetilde{\SR}}
\newcommand\tBQ{\widetilde{\BQ}}

\DeclareMathOperator*\argmin{arg\,min}
\DeclareMathOperator*\argmax{arg\,max}
\DeclareMathOperator*\merge{\bf merge}

\newcommand{\xmark}{{\color{tab20red}\ding{55}}}
\newcommand{\cmark}{{\color{tab20green}\ding{51}}}

\newtheorem{theorem}{Theorem}
\newtheorem{lemma}{Lemma}
\newtheorem{definition}{Definition}

\newcommand{\overbar}[1]{\mkern 1.5mu\overline{\mkern-1.5mu#1\mkern-1.5mu}\mkern 1.5mu}

\newif\ifincludesupplement
\includesupplementtrue

\begin{document}

\onecolumn
\pagenumbering{roman}
\thispagestyle{empty}
\null\vfill
\begin{center}
  \large
  This work has been submitted to the IEEE for possible publication.
  Copyright may be transferred without notice, after which this version may no longer be accessible.
\end{center}
\vfill
\twocolumn
\pagenumbering{arabic}

\title{Optimizing Incomplete, Large-Scale and Sparse Multi-Graph Matching in Bioimaging}

\author{
	Max Kahl$^{*1}$\quad 
	Sebastian Stricker\thanks{${}^{*}$ Equal contribution; author order may be altered in citations.}$^{*2}$\quad
	Lisa Hutschenreiter$^2$\quad
	Florian Bernard$^3$\quad\\
	Carsten Rother$^2$\quad
	Bogdan Savchynskyy$^2$ \\
	$^1$Max Planck Institute for Informatics \quad $^2$Heidelberg University \quad $^3$University of Bonn
}

\maketitle

\begin{abstract}
Multi-graph matching is a fundamental problem in computer vision. 
Our work is motivated by a challenging application in bioimaging, where \emph{dozens} or even \emph{hundreds} of 3D microscopy images of worms must be brought into correspondence. 
Existing datasets do not cover this large-scale regime, and virtually all existing methods are inapplicable because they assume a \emph{complete} or \emph{dense} problem setting.
To support further research, our first contribution is a new large-scale dataset based on problem instances from bioimaging. 
Our second contribution is a comprehensive analysis of the two main multi-graph matching paradigms: \emph{direct} and \emph{permutation synchronization}-based formulations. 
We argue, in part by proof, that practical large-scale methods must explicitly address problem \emph{sparsity} and \emph{incompleteness}. 
Since standard permutation synchronization approaches fail in this setting, we further introduce a \emph{sparse} permutation synchronization paradigm.
Our final contribution is GREEDA, a general method for sparse and incomplete problems that can be instantiated across cost orders and paradigms. 
While our paper focuses on objective functions up to quadratic order, GREEDA is inherently generalizable to arbitrary orders.
On larger, sparse instances, GREEDA outperforms competing methods in both objective value and runtime. 
For example, for moderately sized problems based on 30 worm images GREEDA produces a high-quality solution within 2 minutes, whereas competitors require at least an hour and yield far worse results. 
On smaller dense problems, GREEDA remains on par with leading methods while being an order of magnitude faster.
\end{abstract}


\section{Introduction}
\label{sec:intro}

\IEEEPARstart{E}{stablishing} correspondences across multiple finite sets
is a fundamental combinatorial problem important for, \eg,
3D model retrieval \cite{nie2018view},
shape matching \cite{sahilliouglu2020recent,gao2021isometric},
statistical shape modeling~\cite{yadav2023population,heimann2009statistical},
federated learning~\cite{liu2022deep},
and genomic data analysis~\cite{chen2019identification}.
Typically, each \emph{object}, such as an image or shape, is represented by a set of \emph{keypoints}, which stand for semantically meaningful parts of an object.
The task is  to bring these keypoints into correspondence, or \emph{match} them. 
This matching must satisfy several
conditions, see \cref{fig:problem-description}:
\begin{itemize}
  \item \emph{Uniqueness}. 
        Each keypoint of a given object can be matched to \emph{at most} one keypoint of any other object. 
        If ``\emph{at most}'' is substituted by ``\emph{exactly}'', one speaks of a \emph{complete} and otherwise of an \emph{incomplete} (or \emph{partial}) matching. 
        Due to occlusions or noise during the keypoint extraction process, the incomplete setting is prevalent in practice.
  \item \emph{Cycle consistency}.
        If keypoint $ 1 $ in object $ V^p $ is matched to keypoint $ 2 $ in object $ V^q $ and keypoint $ 3 $ in object $ V^r $, then keypoint $ 2 $ must be matched to keypoint $ 3 $. Similar transitivity conditions must hold for all matching cycles across arbitrary object subsets.
  \item \emph{Costs}.
        Matchings must be \emph{minimal} \wrt the given costs, quantifying the similarity between all keypoints.
        Costs for $d$ objects decompose into a sum of costs for each of the $ d(d-1)/2 $ object pairs.
        The latter, in turn, are sums of \emph{linear}, keypoint-to-keypoint costs and \emph{quadratic}, keypoint-pair-to-keypoint-pair costs.
        Quadratic costs allow the modeling of mutual geometric relations between keypoints, considerably improving matching accuracy~\cite{haller2022comparative}.
\end{itemize}
\begin{figure}[t!]
  \centering
  \begin{tikzpicture}[scale=0.7]
      \tikzset{vertex/.style = {shape=circle,draw}, inner sep=0.04cm}
      \tikzset{blacked/.style = {shape=circle,draw,fill=black,text=black, inner sep=0.06cm}}
      \tikzset{edge/.style = {-, dotted, tab20gray}}
      \tikzset{greenvertex/.style = {draw=tab20green,shape=circle, text=black, line width=0.5mm}, inner sep=0.04cm}
      \tikzset{bluevertex/.style = {draw=tab20blue,shape=circle, text=black, line width=0.5mm}, inner sep=0.04cm}

      \node[greenvertex] (1_p) at  (-2,2) {$ 1 $};
      \node[bluevertex] (2_p) at  (-1,1) {$ 2 $};
      \node[vertex] (3_p) at  (0,0) {$ 3 $};

      \node[vertex] (1_q) at  (5,2) {$ 1 $};
      \node[greenvertex] (2_q) at  (4,1) {$ 2 $};
      \node[vertex] (3_q) at  (3,0) {$ 3 $};

      \node[bluevertex] (1_r) at  (0,4) {$ 1 $};
      \node[vertex] (2_r) at  (1,4) {$ 2 $};
      \node[greenvertex] (3_r) at  (2,4) {$ 3 $};
      \node[vertex] (4_r) at  (3,4) {$ 4 $};

      \node (V_p) at (-2, 0) {$V^{p}$};
      \node (V_q) at (5, 0) {$V^{q}$};
      \node (V_r) at (1.5, 5) {$V^{r}$};

      \node[draw, rectangle, minimum height=0.75cm, minimum width=3.75cm, rotate=0] (box_r) at (1.5,4) {};
      \node[draw, rectangle, minimum height=0.75cm, minimum width=3.5cm, rotate=-45]  (box_p) at (-1,1) {};
      \node[draw, rectangle, minimum height=0.75cm, minimum width=3.5cm, rotate=45] (box_q) at (4,1) {};

      \draw[edge] (1_p) to (1_q);
      \draw[edge] (1_p) to (2_q);
      \draw[edge] (1_p) to (3_q);
      \draw[edge] (2_p) to (1_q);
      \draw[edge] (2_p) to (2_q);
      \draw[edge] (2_p) to (3_q);
      \draw[edge] (3_p) to (1_q);
      \draw[edge] (3_p) to (2_q);
      \draw[edge] (3_p) to (3_q);

      \draw[edge] (1_p) to (1_r);
      \draw[edge] (1_p) to (2_r);
      \draw[edge] (1_p) to (3_r);
      \draw[edge] (1_p) to (4_r);
      \draw[edge] (2_p) to (1_r);
      \draw[edge] (2_p) to (2_r);
      \draw[edge] (2_p) to (3_r);
      \draw[edge] (2_p) to (4_r);
      \draw[edge] (3_p) to (1_r);
      \draw[edge] (3_p) to (2_r);
      \draw[edge] (3_p) to (3_r);
      \draw[edge] (3_p) to (4_r);

      \draw[edge] (1_q) to (1_r);
      \draw[edge] (1_q) to (2_r);
      \draw[edge] (1_q) to (3_r);
      \draw[edge] (1_q) to (4_r);

      \draw[edge] (2_q) to (1_r);
      \draw[edge] (2_q) to (2_r);
      \draw[edge] (2_q) to (3_r);
      \draw[edge] (2_q) to (4_r);

      \draw[edge] (3_q) to (1_r);
      \draw[edge] (3_q) to (2_r);
      \draw[edge] (3_q) to (3_r);
      \draw[edge] (3_q) to (4_r);

      \node[fill=white, opacity=.8,text opacity=1] (C^pr) at (1.2, 2.85) {$ C^{p,r}_{13,34} $};
      \node[fill=white, opacity=.8,text opacity=1] (C^qr) at (3.78,2.62) {$ C^{q,r}_{34,34} $};

      \node (C^pr_anchor_1) at (0.1,3.1) {};
      \node (C^pr_anchor_2) at (1.5,2) {};
      \draw[<->, dashed, line width=0.3mm] (C^pr_anchor_1) to[bend right=20] (C^pr_anchor_2);

      \node (C^qr_anchor) at (3,3.17) {};
      \draw[->, dashed, line width=0.3mm] (C^qr) to[bend right=20] (C^qr_anchor);

      \draw[tab20blue, line width=0.4mm] (2_p) to (1_r);

      \draw[tab20orange, line width=0.4mm] (3_p) to (1_q);
      \draw[tab20orange, line width=0.4mm] (3_p) to (4_r);
      \draw[tab20orange, line width=0.4mm] (3_q) to (4_r);

      \draw[tab20green, line width=0.4mm] (1_p) to (2_q);
      \draw[tab20green, line width=0.4mm] (1_p) to (3_r);
      \draw[tab20green, line width=0.4mm] (2_q) to (3_r);

\end{tikzpicture}%
  \caption{
    \textbf{Multi-graph matching and cycle consistency. }
    Shown are three objects
    (solid rectangles) $ V^p, V^q, V^r $, whose keypoints
    (circles) must be matched (edges).
    By construction, the setting is incomplete since $V^r$ has $4$ keypoints, whereas $ V^p, V^q$ have only $3$.
    The blue and green matchings are cycle-consistent and form cliques in the shown 3-partite graph. 
    The orange matching is not cycle-consistent and as such does not form a clique.
    Examples of linear $ C^{q,r}_{34,34} $ and quadratic $ C^{p,r}_{13,34} $ costs (dashed arrows) and their associated matchings are also shown.
  }
  \label{fig:problem-description}
\end{figure}

\indent The problem combining these conditions is known as \emph{multi-graph matching (MGM)}, and it is NP-hard \cite{crama1992approximation,tang2017initialization} in general.
Note that the term \emph{graph} in multi-graph matching originally stems from interpreting objects as graphs and deriving quadratic costs from differences of respective edge attributes.
We call an MGM problem \emph{sparse} if most keypoint-to-keypoint matchings are \emph{forbidden} (otherwise \emph{allowed}). An MGM problem where all matchings are allowed is called \emph{dense}. 

\begin{figure*}
  \centering
  \refstepcounter{figpoint}\label{fig:worms:incompleteness}%
  \refstepcounter{figpoint}\label{fig:worms:sparsity}%
  \refstepcounter{figpoint}\label{fig:worms:cycle-consistency}%
  \refstepcounter{figpoint}\label{fig:worms:costs}%
  \input{floats/fig-worms-tikz}
  \caption{
\textbf{Three exemplary worms} from the considered datasets, shown from the side.
Depicted are 3D renderings of the corresponding nuclei segmentation masks.
Since \textit{C.\ elegans} has a fixed number of cells in a given developmental stage (558 in the L1 larval stage \cite{richards2013quantitative}), there exists a 1-to-1 mapping between the cells of any two worms.
However, due to occlusions and noise, the number of segmented nuclei varies across masks (top to bottom: 554, 532, 539), making correct correspondences challenging even for humans. 
\textbf{\cref{fig:worms:incompleteness}}~Because the number of segmented cells varies across worms, the formulation must be \emph{incomplete}, \ie, not every cell is required to be matched. 
This can be handled either \emph{explicitly}, by adding a dummy node $\delta$ for each cell, or \emph{implicitly}, by the underlying matching algorithm.
We argue in \cref{sec:incomplete-transformation} that the former is prohibitively expensive in practice.
\textbf{\cref{fig:worms:sparsity}}~Due to the large number of cells, only \emph{sparse} formulations allow for tractable optimization, see \cref{sec:real_world_appl}.
We restrict each cell to a local set of candidate matches, which both sparsifies the problem and prevents implausible long-range, \eg, head-to-tail, assignments.
In practice, only about $3\%$ of all possible keypoint-to-keypoint matchings are allowed, \ie, have finite costs.
\textbf{\cref{fig:worms:cycle-consistency}}~Any valid solution must be \emph{cycle-consistent}. The green matching is cycle-consistent, the orange one is not.
\textbf{\cref{fig:worms:costs}}~Costs for the corresponding multi-graph matching problem are computed from the segmentation masks.
After centering each worm in a common coordinate system, linear terms $C^{p,q}_{is,is}$ are defined from the distance $ \text{d}(i,s) $ between the coordinates of cells $i$ and $s$, favoring matchings at similar absolute positions. Additional linear terms can be derived from visual features such as cell shape~\cite{kainmueller2014active} or color~\cite{yemini2021neuropal} similarity (not shown).
Quadratic terms $C^{p,q}_{is,jt}$ are defined from the difference between the distances $\text{d}(i,j)$ and $\text{d}(s,t)$ between cell pairs and favor matchings that preserve relative geometry. That is, when matching $i$ to $s$ and $j$ to $t$, the distance $\text{d}(i,j)$ should equal the distance $\text{d}(s,t)$.
For details on the cost definitions and pre-processing steps such as straightening, see~\cite{kainmueller2014active}. \\
  }
  \label{fig:c-elegans}
\end{figure*}

\cref{fig:c-elegans} illustrates the primary application of our work, where 100 C.~elegans worms have to be matched, and the keypoints are candidates for nuclei. 
This is a large-scale problem with about $45$M possible keypoint correspondences, 
about two orders of magnitude more than usually considered~\cite{swoboda2019convex}.

Assigning cell identities to worm nuclei based on a matching allows cell-specific gene activity to be quantified via reporter gene fluorescence~\cite{long20093d,li2024full}.

The three exemplary worms shown in \cref{fig:c-elegans} have different numbers of keypoints, and only a small fraction of keypoint-to-keypoint matchings are allowed.
Hence, by construction, we are given an incomplete MGM problem with a sparse cost structure.

Despite finding a good matching, certainly runtime also plays a key role.
Firstly, a fast method may scale better for problems  with a larger number of objects.
Secondly, better applied results may be obtained when the costs are learned, \ie, the matchings are forced to fit ground truth in supervised (\eg, \cite{rolinek2020deep}) or certain structural criteria in unsupervised cases (\eg, \cite{karg2025fully,tourani_unsupervised_2024}). 
This means that the MGM task runs inside a learning loop and hence must be solved in reasonable runtime. 
Thirdly, MGM has so far been used rarely in practical applications, apart from bioimaging.
We conjecture that the reason is a lack of good and fast MGM solvers.

Existing methods mainly follow two paradigms: direct and permutation synchronization-based formulations. 
Direct formulations address the quadratic MGM problem itself, whereas synchronization approaches first solve all pairwise GM problems independently and then find a cycle-consistent approximation by changing a minimal number of matchings. 
In the following, we refer to these as \emph{synchronization} methods and omit the word \emph{permutation}.

\Subsubsection{Our contributions} are as follows: 
\begin{itemize}
\item We analyze direct and synchronization-based MGM \wrt \emph{sparsity} and \emph{incompleteness}. 
In particular, we adapt the synchronization paradigm to \emph{sparse} problems and prove that methods designed for the \emph{complete} MGM problem are not practically applicable in larger \emph{incomplete} settings due to prohibitive runtime scaling.
\item We present a new algorithm, \emph{GREEDA}, for efficiently solving the incomplete MGM problem under sparse costs. 
It handles linear, quadratic, and, theoretically, higher-order objective functions, and can be instantiated as both a direct and a synchronization method through formulation-specific subroutines. 
Across both paradigms, it considerably and consistently outperforms the state-of-the-art on large-scale datasets in both runtime and objective value. 
On small-scale data, it remains competitive with or superior to leading methods in terms of objective value while being an order of magnitude faster.
\item We construct and publish a new large-scale dataset based on existing problem instances \cite{worms_large_dataset}. Unlike previous datasets featuring at most \emph{10} worms, our dataset contains up to \emph{100} worms, comprises 11 problem instances, and is available at \url{https://doi.org/10.11588/DATA/RTINUU}. 
\end{itemize}

This work builds upon the conference paper~\cite{kahl2025towards}.
The central addition is the extension of our theoretical and experimental analysis to the synchronization paradigm, establishing \emph{GREEDA} as a general-purpose method for MGM across problem sizes, sparsity regimes, and both direct and synchronization formulations.
Additionally, we improve our swap local search, the main bottleneck of the earlier version, evaluate on 11 new large-scale instances, and publish the related dataset.

Furthermore, we now provide our algorithm as an easy-to-use Python package to facilitate the adoption of our method.
Our code is available at \url{https://pypi.org/project/pylibmgm/}.

\section{Real-World Applicability of MGM Methods}
\label{sec:real_world_appl}
Large-scale MGMs in bioimaging applications, see \cref{fig:c-elegans}, have two decisive properties hardly covered by existing MGM algorithms: \emph{Incompleteness} and \emph{sparsity}. Whereas incomplete MGMs naturally arise due to noise in images and pre-processing steps like segmentation, sparsity allows one to limit the size of the problem and encode important spatial relationships. 
Limiting the size is important as the number of quadratic costs grows as $O(n^4 d^2)$, where $n$ is the number of keypoints per object and $d$ the number of objects.
Each of the
$d(d-1)/2 = O(d^2)$ object pairs admits $n^2$ matchings and therefore $O(n^4)$ pairs of
matchings, each carrying one quadratic cost.

Hence, dense MGMs with $n>100$ become practically infeasible. 
Additionally, forbidding for example the matching of nuclei in the head of a worm to those in its tail is an important natural prior (see \cref{fig:c-elegans}).

\Subsubsection{Complete vs. incomplete MGM.}
Many approaches~\cite{bernard2018ds,tang2017initialization,yan2016constrained,yan2013joint,yan2015consistency}
only treat complete MGM.
Despite popular claims to the contrary, there is no straightforward and
efficient way to apply them to the incomplete case.
While a polynomial transformation from incomplete to complete graph matching exists
\cite{haller2022comparative},
its often-mentioned generalization to MGM, \eg, \cite{yan2016constrained,yan2013joint,yan2015consistency},
is prohibitively expensive.
Transforming an \emph{incomplete} MGM problem with $ d $
objects, each with  $ n $ keypoints,
into a \emph{complete} one results in $nd$ keypoints in \emph{each} of the $d$ objects, see \cref{sec:incomplete-transformation}. 
This makes any complete MGM algorithm impractical.
The following theorem, proven in~\cref{sec:incomplete-transformation}, essentially states that there exists \emph{no significantly better} transformation.
\begin{theorem}
  Let $B$ be a complete MGM problem instance that is a
  \emph{clique-wise} transformation of an incomplete MGM problem instance $A$.
  Let $A$ have $|V|$ keypoints in total. Then \emph{each object} in $B$ has at least $|V|$ keypoints.
  \label{thm:minimality-of-transformation}
\end{theorem}

\noindent In turn, our approach applies to incomplete MGMs \emph{directly}.

\Subsubsection{Sparse problems.}
As defined above, an MGM problem is sparse if most keypoint-to-keypoint matchings are forbidden.
In practice, forbidden matchings can be modeled by \emph{infinite} costs that are only implicitly present in the problem description~\cite{swoboda2022structured,graphmatchingddtorresanietal,haller2022comparative}, rendering the latter \emph{sparse}.
However, many methods, especially the solution construction and synchronization ones, discussed in the following \cref{sec:related_work}, do not apply to infinite costs.
In particular, synchronization methods often fail to reconcile infinite costs with their assumptions, \eg, by heavily relying on spectral methods~\cite{bernard2021sparse,arrigoni2017synchronization,pachauri2013solving}.
In \cref{sec:synchronization}, we discuss this incompatibility with infinite costs in more detail and propose the \emph{sparse synchronization} task as a remedy.
Overall, our method is guaranteed to return only \emph{allowed}, \emph{finite-cost} matchings for \emph{sparse} problems, across direct \emph{and} synchronization formulations.

\section{Related Work}
\label{sec:related_work}

\Subsubsection{Graph matching (GM)} refers to the well-studied
\cite{haller2022comparative} special case of matching \emph{two} objects.
Similar to MGM, one distinguishes \emph{complete} and \emph{incomplete} GM.
Complete GM is also known as the NP-hard~\cite{pardalos1993uo}
\emph{quadratic assignment problem (QAP)}~\cite{lawler1963quadratic}
or, if quadratic costs are zero,
the polynomially solvable~\cite{kuhn1955hungarian}
\emph{linear assignment problem (LAP)}~\cite{burkard2012assignment}.
MGM can be viewed as $d(d-1)/2$ GM problems between all object pairs,
coupled via cycle consistency constraints.
If we wish to distinguish between cost orders similar to GM,
we refer to (complete \emph{or} incomplete) MGM as Multi-QAP and, respectively,
as Multi-LAP if quadratic costs are zero.
Unlike the LAP, the Multi-LAP is
NP-hard~\cite{crama1992approximation,tang2017initialization}.
As a result of this coupled structure, GM solvers are an essential but
interchangeable component of many MGM methods, including those considered in this work.
We utilize the GM solver~\cite{hutschenreiter_fusionmoves_2021} as a GM subroutine since it shows the best results in the recent GM benchmark~\cite{haller2022comparative}.

\Subsubsection{Multi-graph matching (MGM).}
Although MGM has not yet received the same attention as GM, several types of solvers have been proposed in recent years. In the following, we review their ability to deal with incomplete and sparse MGMs. The works most closely related to our algorithm are additionally addressed in \cref{sec:our-method}.

\Paragraph{GM-solver-based primal heuristics} are probably the largest class of MGM algorithms, including ours. These methods usually consist of \emph{construction} and \emph{local search} subroutines that employ a GM solver as their decisive component. Construction methods are usually based on the \emph{composition principle}: They assign a matching between two objects by composition via a third one. This third object is either iteratively chosen \wrt a metric combining costs \emph{and} a (pairwise) consistency measure~\cite{jiang2020unifying,yan2015multi} or is a node in a \emph{spanning tree}~\cite{tang2017initialization,yan2016constrained,yan2015consistency} of the complete graph connecting all objects as its nodes. Whereas the first subclass does not even guarantee cycle consistency of the result, the second cannot address sparse problems as forbidden matchings are often selected through composition. Moreover, existing methods of this type consider only complete MGMs.

In contrast, our approach extends a feasible solution consisting of $k\leq d$ objects by one object in each iteration until the solution includes all objects, \ie, $k=d$. 
In combination with the incomplete GM solver~\cite{hutschenreiter_fusionmoves_2021}, it guarantees cycle-consistent and allowed matchings.

\Paragraph{Local search} subroutines are usually based on the observation that cycle-consistent matchings can be improved by re-matching one object to the remaining, already matched $ (d-1) $ objects, see~\cref{fig:gm-ls} in \cref{sec:our-method}.
It turns out that this re-matching constitutes a GM problem.
Hence, one can iteratively re-match different objects utilizing a GM solver.
We refer to this algorithm as \emph{GM local search} (GM-LS), see \cref{sec:qap-local-search} for details.
This idea was initially proposed by~\cite{bandelt2004local} for the closely related \emph{multi-dimensional assignment problem (MDAP)} and then independently for different MGM variants \cite{tang2017initialization,yan2016constrained,yan2013joint,yan2015consistency}.
Similar to the construction subroutine, existing works explicitly consider local search only for the complete problem.\\
\indent In this work, we \emph{parallelize} GM local search and make it applicable to incomplete sparse MGMs. We also propose another local search method, referred to as \emph{(clique) swap local search}, which is so far unknown in the computer vision literature. The related work from operations research is discussed in \cref{sec:swap-local-search}.
\indent \Paragraph{Synchronization approaches}
\cite{chen2018projected,shen2016normalized,chen2014near}
first independently solve all pairwise GM problems, which generally induces a \emph{cycle-inconsistent} matching.
Subsequently, they try to find a \emph{cycle-consistent approximation}
by changing a \emph{minimal} number of matchings~\cite{arrigoni2017synchronization,bernard2021sparse,maset2017practical}.
Such synchronization methods often serve as a subroutine for convex-relaxation-based MGM solvers, \eg,
\cite{swoboda2019convex,bernard2018ds}.
However, they are \emph{expensive} and prone to
\emph{suboptimal} decisions as they require the solution of $d(d-1)/2$
GM problems and ignore costs during the approximation.
Due to the latter, existing synchronization methods also fail to deal with
sparse problems because they often introduce \emph{``blunders''}
by choosing forbidden matchings.
In contrast, our algorithm considers the problem costs in all of its stages and thus avoids blunders.
At the same time, if required, it can also be used as a sparsity-aware synchronization method, see~\cref{sec:synchronization}. \\
\indent \Paragraph{Convex relaxation-based methods} are probably the smallest subclass among existing MGM techniques. The work~\cite{bernard2018ds} considers a lifting-free quadratic relaxation, but addresses complete problems only.  The work~\cite{kezurer2015tight} proposes a powerful semi-definite relaxation for incomplete MGMs, but its scalability is strongly limited due to variable lifting, as mentioned by the authors.
The closest competitor for our method is the Lagrange relaxation-based method~\cite{swoboda2019convex}, which is able to deal with incomplete sparse problems. In \cref{sec:experiments} we demonstrate, however, that our algorithm significantly outperforms~\cite{swoboda2019convex} in terms of runtime and objective value.
\begin{figure}[t!]
  \centering
    \begin{tikzpicture}[scale=0.63]
      \tikzset{vertex/.style = {shape=circle,draw}, inner sep=0.04cm}
      \tikzset{blacked/.style = {shape=circle,draw,fill=black,text=black, inner sep=0.06cm}}
      \tikzset{edge/.style = {-, dotted, tab20gray}}
      \tikzset{greenvertex/.style = {draw=tab20green,shape=circle, text=black, line width=0.5mm}, inner sep=0.04cm}
      \tikzset{bluevertex/.style = {draw=tab20blue,shape=circle, text=black, line width=0.5mm}, inner sep=0.04cm}

      \node[rotate=90, align=center] (costs) at (-0.2,0) {Costs};
      \node[align=center, draw, inner sep=5pt] (construction) at (2.4,0) {Feasible Sol. \\Construction\\\cref{sec:construction}};
      \node[align=center, draw, inner sep=5pt, text width=1.5cm] (gm-ls) at (6.2,0) {GM-LS\\\cref{sec:qap-local-search}};
      \node[align=center, draw, inner sep=5pt, text width=1.5cm] (swap-ls) at (9.5,0) {SWAP-LS\\\cref{sec:swap-local-search}};
      \node[rotate=90, align=center] (solution) at (12,0) {Solution};

      \draw[->] (costs) -- (construction);
      \draw[->] (construction) -- (gm-ls);
      \draw[->] (swap-ls) -- (solution);

      \node (gm-ls-anchor-north) at (6.2, 0.8) {};
      \node (swap-ls-anchor-north) at (9.5,  0.8) {};

      \node (gm-ls-anchor-south) at (6.2, -0.8) {};
      \node (swap-ls-anchor-south) at (9.5,  -0.8) {};

      \draw[->] (gm-ls-anchor-north) to[bend left=45] (swap-ls-anchor-north);
      \draw[->] (swap-ls-anchor-south) to[bend left=45] (gm-ls-anchor-south);
      
\end{tikzpicture}%
  \caption{ 
    \textbf{Conceptual diagram of our MGM method GREEDA.}
    \emph{GREEDA} is composed of three parts --
    a construction heuristic (\cref{sec:construction}) and two local search heuristics, \term{GM-LS} (\cref{sec:qap-local-search}) and \term{SWAP-LS} (\cref{sec:swap-local-search}). 
    }
  \label{fig:pipeline}
\end{figure}

\section{Our Method: GREEDA}\label{sec:our-method}
In \cref{sec:definition}, we formalize the MGM problem. 
In \Cref{sec:construction,sec:qap-local-search,sec:swap-local-search}, we describe the individual building blocks of our method, summarized in \cref{fig:pipeline}, together with the corresponding technical contributions. 
We term our method \emph{GREEDA}, alluding to the greedy nature of our construction algorithm, see \cref{sec:construction}.

\subsection{Formal Problem Definition}
\label{sec:definition}
\Subsubsection{Graph matching} concerns itself with matching two 
finite keypoint sets  $  V^{1} $ and $  V^{2} $,
further referred to as
\emph{objects}.
Their elements are called \emph{keypoints} or \emph{vertices}.
GM considers the undirected complete bipartite graph
$
\overbar{ G} =
\left(
V 
\coloneqq
V^{1} \cup  V^{2},
\overbar{E}
\coloneqq
V^{1} \times  V^{2}
\right)
$ 
with objects $  V^{1} $ and $  V^{2} $ as independent sets,
where an edge
$ ij\coloneqq(i, j) \in \overbar{E} $ corresponds to matching vertex
$ i\in V^{1} $ to vertex $ j\in V^{2} $.
Although discussing undirected graphs,
we write the edge set $ \overbar{E} $ as a Cartesian product $V^{1} \times  V^{2}$ 
to emphasize the two independent sets $ V^1 $ and $ V^2 $.
Due to their independence, we can always identify directed with undirected 
edges.
An \emph{incomplete} matching between objects $ V^1 $ and $ V^2 $ is defined
as a subset of edges $  E \subset \overbar{E} $ containing
\emph{at most one} incident edge for each vertex.
\emph{Complete} matchings, conversely,
contain \emph{exactly one} incident edge for each vertex,
which demands equal cardinalities of both objects $ | V^{1}|=| V^{2}| $.
The goal of GM is to find minimal matchings \wrt given costs
$
{C:
\left(
V^{1} \times  V^{2} 
\right)^{2}
\to
\BR \cup \{-\infty, \infty\}}
$.
In the matrix identification $ C(i,s,j,t)=C_{is,jt}$,
diagonal entries $ C_{is,is}$ describe \emph{linear} costs, and off-diagonal entries \emph{quadratic} costs, see \cref{fig:problem-description}.
Linear costs $ C_{is,is} $ penalize \emph{vertex-to-vertex} correspondences, \ie, matching vertex $ i\in V^{1} $ to vertex $ s\in V^{2} $, whereby 
infinite costs $ C_{is, is} =\infty $ \emph{forbid} such a matching.
Quadratic costs $ C_{is,jt} $, in turn, penalize 
\emph{vertex-pair-to-vertex-pair} correspondences, 
\ie, matching the pair $ (i,j)\in V^{1}\times V^{1} $ to
the pair $ (s,t)\in V^{2}\times V^{2} $.

\Subsubsection{Multi-graph matching} generalizes GM to matching
$ d\in\BN_{\geq3} $ objects $V^{p}$, $ p\in[d]\coloneqq[1,d]\cap\BN$, 
\wrt costs
$ C^{p, q}: \left(  V^{p}\times V^{q} \right)^{2} \to \BR \cup \{-\infty, \infty\} $
between all object pairs
$ p,q\in[d] $, $ p\neq q $.
Instead of a bipartite graph, it considers the undirected complete 
\emph{$ d $-partite} graph 
$ \overbar{ G} = \left(  V\coloneqq \cup_{p\in[d]}  V^{p}, \overbar{E} \right) $
with objects $  V^{p} $ as independent disjoint sets.
Since its index uniquely identifies the object $V^p$,
we also refer to object $V^p$ directly as object $p$.

\emph{Incomplete} multi-matchings 
are subsets of edges
$  E\subset \overbar{E} $ \st any vertex is incident to \emph{at most one} edge
connecting the same objects.
Similar to GM, \emph{complete} multi-matchings require \emph{exactly one} 
such edge and equal cardinalities of objects, \ie $|V^p| = |V^q|$ for all $p,q\in[d]$.

\Subsubsection{Cycle consistency.}
A multi-matching $  E\subset\overbar{E} $ is \emph{cycle-consistent} if each
\emph{path} in $  E $
of length $n\geq 2$, where $n$ is the number of edges,
can be extended within $  E $ to a
\emph{cycle} of length $n+1$ with at most one vertex per object.
Enforcing this for all paths of length 2 is sufficient \cite{swoboda2019convex}. 
That is, a multi-matching
$  E\subset\overbar{E} $ is cycle-consistent
iff $ ij\in E $ and $ jk\in E $ imply $ ik\in E $ for all 
$ i,j,k\in V $, see 
\cref{fig:problem-description}.

\Subsubsection{Clique representation.}
In \cite{tron2017fast}, it was shown that a
multi-matching $  E\subset\overbar{E} $ is cycle-consistent iff the corresponding
subgraph $  G=( V, E) $ is a \emph{disjoint union of cliques},
\ie, there exists a vertex partition $ \SQ $ of $V$
such that for each part $Q\in\SQ$ the induced subgraph $G|_{Q}=(Q,E|_{Q\times Q})$
is complete, that is, a \emph{clique},
and the induced edge sets $E|_{Q\times Q}$ fully partition~$E$.
Therefore, cycle-consistent multi-matchings are induced by vertex partitions
$ \SQ\subset 2^{V} $ where any part $  Q\in\SQ $
contains at most one vertex per object $  V^{p} $ for all $p\in[d]$,
\ie, $ | Q\cap V^{p}| \leq 1 $, see \cref{fig:problem-description}.
We call such vertex partitions \emph{feasible}
and denote the set of all feasible vertex partitions or \emph{solutions} as
$ \BQ $.

Since each part $  Q\in\SQ $ of a solution uniquely induces a clique in $G$,
we also refer to parts $Q$ as \emph{cliques}.
The object subset actually covered by such a clique $  Q\in\SQ $ is denoted as
$ D(Q)\subseteq[d] $, and a vertex of a clique belonging to object $ p\in D(Q) $
as $  Q^{p} $, \ie, $ \{ Q^{p}\}\coloneqq Q\cap V^{p} $, see 
\cref{fig:gm-ls}.
The set of feasible partitions over the vertices
$ V^{D} \coloneqq \bigcup_{p\in D} V^p $
of an object \emph{subset} $ D\subseteq[d] $ is denoted by $ \BQ^D $.

\Subsubsection{The MGM objective}
is to find cycle-consistent multi-matchings 
minimizing the sum of all costs, \ie,
\begin{equation}
	\hspace{-5pt}\min_{\SQ\in\BQ}
	\left[ 
	C(\SQ) \coloneqq
	\sum_{ Q, R\in\SQ}
  \;
  \sum_{\substack{p,q\in D(Q, R) \\ p < q}}
	C^{p,q}_{ Q^p Q^q, R^p R^q}
	\right]\,,
	\label{eq:MGM}
\end{equation}
where $ D(Q, R) \coloneqq D(Q) \cap D(R) $ are all objects covered by both cliques $ Q, R \in \SQ $.
We assume $C^{p,q}=C^{q,p}$ and count this cost only once for each pair of objects $p,q \in[d]$ by requiring $p<q$.
The sum in \cref{eq:MGM} decomposes into linear costs for $ {Q = R} $, \ie, the diagonal entries of $ C^{p,q} $, and quadratic costs for $  {Q\neq R} $.
Higher-order costs, beyond the scope of this work, analogously involve more than two cliques.
Note that the formulation in Eq. \eqref{eq:MGM} implicitly assumes zero costs for \emph{not} matching a vertex.

\Subsubsection{Alternative formulations.}
The clique formulation from Eq. \eqref{eq:MGM} is non-standard.
Many formulations~\cite{swoboda2019convex,zhou2015multi,chen2014near}
represent (multi-)matchings $  E\subset\overbar{E} $ by
\emph{partial permutation matrices} $ X $,
where $ X_{is}=1 $ iff $ is\in E $.
Others~\cite{nurlanov2023universe,bernard2019hippi,pachauri2013solving}
view MGM as a \emph{labeling problem},
where each vertex of the same object must be assigned a different label -- 
commonly called \emph{universe point}.
Vertices with the same label are matched to
each other, \ie, comprise a clique.
We use the clique formulation because it allows the most natural 
description of our algorithm.
\begin{figure}[t!]
  \centering
  \begin{tikzpicture}[scale=0.85]
      \tikzset{vertex/.style = {shape=circle,draw}, inner sep=0.04cm}
      \tikzset{edge/.style = {-, dotted, tab20gray}}
      \tikzset{black/.style = {shape=circle,draw,fill=black,text=black, inner sep=0.09cm}}
      \tikzset{orange/.style = {shape=circle,draw,fill=tab20orange,text=tab20orange, inner sep=0.09cm}}
      \tikzset{green/.style = {shape=circle,draw,fill=tab20green,text=tab20green, inner sep=0.09cm}}
      \tikzset{blue/.style = {shape=circle,draw,fill=tab20blue,text=tab20blue, inner sep=0.09cm}}
      \tikzset{red/.style = {shape=circle,draw,fill=tab20red,text=tab20red, inner sep=0.09cm}}
      \tikzset{purple/.style = {shape=circle,draw,fill=tab20purple,text=tab20purple, inner sep=0.09cm}}
      \tikzset{purple/.style = {shape=circle,draw,fill=tab20purple,text=tab20purple, inner sep=0.09cm}}

      \tikzset{orangesquare/.style = {draw=tab20orange, rectangle, line width=0.4mm, minimum width=0.5cm, minimum height=0.5cm}}
      \tikzset{greensquare/.style = {draw=tab20green, rectangle, line width=0.4mm, minimum width=0.5cm, minimum height=0.5cm}}
      \tikzset{redsquare/.style = {draw=tab20red, rectangle, line width=0.4mm, minimum width=0.5cm, minimum height=0.5cm}}
      \tikzset{bluesquare/.style = {draw=tab20blue, rectangle, line width=0.4mm, minimum width=0.5cm, minimum height=0.5cm}}
      \tikzset{purplesquare/.style = {draw=tab20purple, rectangle, line width=0.4mm, minimum width=0.5cm, minimum height=0.5cm}}

      \node[greensquare] (S) at (0, 3) {$ Q $}; 
      \node[purplesquare] (Q) at (0, 2) {$ S $};
      \node[redsquare] (R) at (0, 1) {$ T $};
      \node[bluesquare] (T) at (0, 0) {$ R $};

      \node[green] (S_p) at (-3, 3) {};
      \node[green] (S_q) at (-2, 3) {};
      \node[green] (S_r) at (-1, 3) {};
      
      \node[purple] (Q_q) at (-2, 2) {};
      \node[purple] (Q_r) at (-1, 2) {};
      
      \node[red] (R_r) at (-1, 1) {};
      
      \node[blue] (T_p) at (-3, 0) {};
      \node[blue] (T_q) at (-2, 0) {};

      \node[vertex, minimum size=0.5cm] (i) at (3, 2.5) {$ i $};
      \node[vertex, minimum size=0.5cm] (j) at (3, 1.5) {$ j $};
      \node[vertex, minimum size=0.5cm] (l) at (3, 0.5) {$ l $};

      \node[draw, rectangle, minimum width=0.75cm, minimum height=3.4cm, rotate=0] (box_p) at (-3,1.5) {};
      \node[draw, rectangle, minimum width=0.75cm, minimum height=3.4cm, rotate=0] (box_q) at (-2,1.5) {};
      \node[draw, rectangle, minimum width=0.75cm, minimum height=3.4cm, rotate=0] (box_r) at (-1,1.5) {};
      \node[draw, rectangle, minimum width=0.75cm, minimum height=3.4cm, dashed] (box_SQ) at (0,1.5) {};
      \node[draw, rectangle, minimum width=0.75cm, minimum height=2.6cm, rotate=0] (box_w) at (3,1.5) {};

      \node[above] (V^p) at (box_p.north) {$ V^p $};
      \node[above] (V^q) at (box_q.north) {$ V^q $};
      \node[above] (V^r) at (box_r.north) {$ V^r $};
      \node[above] (V^w) at (box_w.north) {$ V^w $};
      \node[above] (SQ) at (box_SQ.north) {$ \SQ $};

      \draw[-, black] (S) to (i);
      \draw[edge] (S) to (j);
      \draw[edge] (S) to (l);
      \draw[edge] (Q) to (i);
      \draw[edge] (Q) to (j);
      \draw[edge] (Q) to (l);
      \draw[edge] (R) to (i);
      \draw[edge] (R) to (j);
      \draw[edge] (R) to (l);
      \draw[edge] (T) to (i);
      \draw[-, black] (T) to (j);
      \draw[edge] (T) to (l);

      \draw[line width=0.03cm, tab20green] (-2,3) ellipse (1.5cm and 0.25cm);
      \draw[line width=0.03cm, tab20purple] (-1.5,2) ellipse (1cm and 0.25cm);
      \draw[line width=0.03cm, tab20red] (-1,1) ellipse (0.5cm and 0.25cm);
      \draw[line width=0.03cm, tab20blue] (-2.5,0) ellipse (1cm and 0.25cm);

      \node (T^p_label) at (-3.85,-0.4) {$ R^{p} $};
      \draw[->, dashed, line width=0.03cm, shorten >= 1pt]  (T^p_label) to[bend left=30] (T_p.west);

      \node (S^p_label) at (-3.85,3.4) {$ Q^{p} $};
      \draw[->, dashed, line width=0.03cm, shorten >= 1pt]  (S^p_label) to[bend right=30] (S_p.west);

      \draw [decorate, decoration={brace, amplitude=5pt, mirror}] ([yshift=-2pt]box_p.south west) -- ([yshift=-2pt]box_q.south east)
            node[midway, below=5pt, inner sep=1pt] (DT) {$ D(R)=\{p,q\} $};

      \path (S) -- (i) coordinate[pos=0.25] (iS_anchor);
      \path (T) -- (j) coordinate[pos=0.25] (jT_anchor);
      \draw[<->, dashed, line width=0.03cm, shorten <=3pt, shorten >=3pt]  (iS_anchor) to[bend left=20]
            node[midway, right=1pt, fill=white, fill opacity=.7, text opacity=1, inner sep=1.5pt] (C^wQ) {$ C^{w,\SQ}_{iQ,jR} $} (jT_anchor);

      \draw[-, line width=0.2mm, tab20green] (-0.5,3) to (S);
      \draw[-, line width=0.2mm, tab20purple] (-0.5,2) to (Q);
      \draw[-, line width=0.2mm, tab20red] (-0.5,1) to (R);
      \draw[-, line width=0.2mm, tab20blue] (-1.5,0) to (T);

      \draw[<->, line width=0.03cm, preaction={draw, white, line width=3pt, -}] (-3.1,2.72) to[bend right=35] (-3.1,0.28);
      \node[rotate=90] at (-3.9, 1.5) {\textbf{Swap}};

\end{tikzpicture}
  \caption{
  	\textbf{Feasible solution construction and swap local search.}
  	Depicted are three cycle-consistently matched objects  
  	$ V^p, V^q, V^r $.  Matched vertices are of the same color,
  	horizontally aligned, and decomposed into cliques $ Q,S,T,R $ (ellipses),
  	yielding a partial solution $ \SQ $ (dashed rectangle).
  	For example, the clique $ R $ spans the objects $ D(R)=\{p,q\} $, where 
  	its vertex $ R^p $ belongs to object $ V^p $.
  	The costs $ C^{w,\SQ}_{iQ,jR} $ for matching vertices
  	$ i $ and $ j $ of object $ V^w $ to cliques $ Q $ and $ R $ of the partial solution $\SQ$ are also shown.
    We consider two types of swaps, see \cref{sec:swap-local-search}: The \emph{vertex swap} fixes the object $p$ and interchanges the vertices $ Q^p $ and $ R^p $ (solid arrow). 
    The \emph{2-clique swap} is more global and jointly optimizes over all possible vertex swaps between cliques $ Q $ and $ R $.
  }
  \label{fig:gm-ls}
\end{figure}

\subsection{Feasible Solution Construction}
\label{sec:construction}
\Subsubsection{Basic algorithm.}
The basic construction \cref{alg:construction} obtains feasible
solutions by solving a \emph{chain} of GM problems.
Given a \emph{random ordering} of objects $ [d] $,
it iteratively extends a \emph{partial solution} $\SQ$, which,
in the $k$-th iteration, matches the objects
$  V^{1}, \dots,  V^{k} $.
The matching $  E $ is the solution of the GM problem with costs
$ C^{k+1,\SQ}_{iS,jT},\ {i,j\in V^{k+1}},\ {S,T\in\SQ}$,  
stemming from the \emph{summation} over a clique's individual elements
\begin{equation}
	C^{k+1,\SQ}_{i S,j T} 
	\coloneqq 
	\sum_{q\in D(S,T)}
	C^{k+1, q}_{i S^q,j T^q}\,.
	\label{eq:basic-costs}
\end{equation}
The partial solution $\SQ$ is extended by 
adding matched vertices $ i\in V^{k+1}, i S\in E $ 
to their assigned cliques 
$ \{i\}\cup S$.

We denote the (partial) solution resulting from this \emph{merge} as
$ {\SQ'=\merge( V^{k+1}, \SQ;  E)} $.
If a vertex $ i\in V^{k+1} $ is unmatched, that is
$ \forall j S\in E: j\neq i $, it is added as a singleton, \ie $ \{i\}\in\SQ' $.
If a clique $  S\in\SQ $ is unmatched, that is $ \forall i T\in E:  S\neq  T $,
it remains unchanged $  S\in\SQ' $.

\cref{alg:construction} has six notable properties:
(1) It guarantees cycle consistency. 
(2) It can use virtually any GM solver as a black box. 
(3) It is independent of the cost order and thus applicable to higher-order problems as long as the underlying GM solver is applicable. 
(4) It is parameter-free \wrt the number of cliques. 
(5) It scales linearly in the number of objects $ d\in\BN_{\geq3} $. 
(6) It is randomized, hence the best solution from multiple parallel runs can be kept.\\
\Subsubsection{Extensions and parallelization.}
We propose two extensions of \cref{alg:construction}: \emph{Incremental} and 
\emph{parallel} construction.
\Paragraph{Incremental construction} addresses the problem of 
\emph{error propagation}.
Errors during early matchings ``propagate'' along the chain of pairwise problems,
sway later matchings, and worsen the final solution.
Because MGM problems' restrictions to object subsets are again MGM problems,
we propose to weaken such propagations by using a \emph{better} 
(but potentially more expensive) \emph{multi-matching solver} for the first 
$ s\in[d] $ objects $ [s] $.

\Paragraph{Parallel construction} improves the linear scaling \wrt the number of 
objects. 
It generalizes the chain of pairwise problems to \emph{binary trees}, 
specifically \emph{leaf-labeled, ordered, binary trees} with 
\emph{leaf label set} $ [d] $.
Leaf labels associate leaves with objects.
Each tree vertex corresponds to a (partial) solution $ \SQ\in\BQ^{D} $ 
matching its descendant leaves $ D\subseteq[d] $.
In analogy to the basic variant, a tree vertex's (partial) solution is
obtained by first solving a GM problem that matches the cliques of its children 
$ \SA\in\BQ^{D_A}$, 
$\SB\in\BQ^{D_B} $
and then merging matched cliques.
Pairwise costs  $ C^{\SA,\SB}_{IS,JT},I,J\in\SA,S,T\in\SB $,
are again obtained by summing costs of a clique's individual
elements,
\begin{equation}
	C^{\SA, \SB}_{ I S, J T}
	\coloneqq
	\sum_{p\in D( I, J)}
	\sum_{q\in D( S, T)}
	C^{p, q}_{ I^p S^q, J^p T^q} \,.
	\label{eq:parallel-costs}
\end{equation}
Ultimately, the root corresponds to a final solution.
GM problems of tree vertices at the same level can be solved in 
\emph{parallel}.
Therefore, \emph{balanced} trees yield the best acceleration --
in sequence only $ O(\log(d)) $ instead of $ O(d) $ problems need solving.
Additionally, most properties of the basic construction transfer to the 
parallel version:
It is \emph{randomized} \wrt the leaf order,
allows \emph{plug-\&-play} of pairwise solvers, 
guarantees \emph{cycle consistency},
and is \emph{parameter-free} \wrt the number of cliques.
Pseudocode and further details are provided in \cref{suppl:construction}.
\begin{algorithm}[t]
\addtolength{\hsize}{1.5em}%
\DontPrintSemicolon
  $
    \SQ
    \leftarrow
    \bigcup_{i\in V^{1}}  \left\{\left\{ i \right\}\right\}
  $\;
  \SetInd{0.5em}{0.5em}
  \For{$ k\in \left\{ 1,\dots, d-1 \right\} $}{
    \tcp{compute matching $E$ between object $(k+1)$ and the already matched $[k]$}
    $ E  \leftarrow $ (Approx.) solve GM with costs 
    $ 
      C^{k+1, \SQ}
    $
    (Eq. \eqref{eq:basic-costs})\; 
    $ 
      \SQ
      \leftarrow 
      \merge
      ( 
        V^{k+1},
        \SQ; \
        E
      )
    $\;
  }
\caption{Basic solution construction.}
\label{alg:construction}
\end{algorithm}

\Subsubsection{Our novelty and related work.}
The basic solution construction \cref{alg:construction}
was first mentioned in \cite{bandelt1994approximation} 
and later used in~\cite{bandelt2004local}
for complete Multi-LAPs seen as a special case of the MDAP.
Experiments in \cite{bandelt2004local} also demonstrate its superiority 
over star-shaped spanning tree approaches.
Its application to \emph{Multi-QAPs}
and its \emph{parallelization} using binary trees are new.
Incremental construction conceptually resembles 
an idea proposed in \cite{tang2017initialization}, where,
as part of a spanning tree heuristic, 
a local search is run on the hitherto constructed partial solution.
To summarize, our contribution is to introduce the prior art from operations research \cite{bandelt1994approximation,bandelt2004local}, apply it to \emph{Multi-QAPs} instead of Multi-LAPs,
and propose new extension and parallelization schemes.

\subsection{GM Local Search}
\label{sec:qap-local-search}
\Subsubsection{Basic algorithm.}
The basic GM local search step comprises \emph{splitting} and \emph{merging}
an object $  V^{p} $ from and to a given solution 
$ \SQ\in\BQ $ as defined by \cref{alg:qap-ls}.
An object $  V^{p} $ is split by \emph{subtracting}
its vertices from all cliques $  Q\in\SQ $, yielding a partial solution 
$ \SQ^{-p}\coloneqq \{Q\setminus V^p \mid Q\in\SQ \} $.
Merging it back into the split solution works just as in \cref{sec:construction}
by solving a GM problem with costs from Eq.~\eqref{eq:basic-costs}
whose solution dictates the merge.
While only accepting improvements, 
the basic GM local search \cref{alg:qap-ls} repeats this step 
along an \emph{object sequence} $ \left( p_k \right)_{k\in\BN}, p_k\in[d] $,
until a suitable stopping criterion is met.
In practice, we run \cref{alg:qap-ls} cyclically along the random object 
ordering $ [d] $ used in \cref{alg:construction}
as long as the solution improves.
Note that splitting may yield empty sets for single-vertex cliques, \ie $Q\setminus V^p = \emptyset$ for some $Q \in \SQ$.
For the sake of notation, here we allow for the empty set in the solution.
Like \cref{alg:construction},
the pairwise solver is treated as a \emph{black box} and
solutions are always \emph{cycle-consistent} by construction.
\begingroup
\begin{algorithm}[t]
\addtolength{\hsize}{1.5em}%
\DontPrintSemicolon
\KwIn{Solution $ \SQ\in\BQ $, Object Sequence $ (p_k)_{k\in\BN}$}
\SetInd{0.7em}{0.7em}
  $ k \leftarrow 1 $\; 
  \While{stopping criterion not met}{
  \tcp{split object $p_k$ from the matching $\SQ$}
    $
      \SQ'
      \leftarrow
      \{\ Q\setminus V^{p_k} \mid Q\in\SQ \}
    $\;
    \tcp{compute matching $E$ between object $p_k$ and the rest}
    $ E  \leftarrow$ (Approx.) solve GM with costs 
    $ C^{p_k, \SQ'} $ (Eq. \eqref{eq:basic-costs})\; 
    $ 
      \SQ'
      \leftarrow 
      \merge
      ( 
        V^{p_k},
        \SQ'; \
        E
      )
    $\;
    \If{$ C ( \SQ' ) < C ( \SQ ) $}{
      $ \SQ \leftarrow \SQ'$ \tcp{accept if profitable}
    }
    $ k \leftarrow k+1 $\;
  }
\caption{Basic GM local search.}
\label{alg:qap-ls}
\end{algorithm}
\endgroup

\Subsubsection{Extensions and parallelization.}
A known extension~\cite{bandelt2004local} of this local search is to split a solution 
$ \SQ\in\BQ  $
along a subset of objects $ D\subset[d] $.
As a result, two partial solutions 
$ \{  Q\cap V^D \mid  Q\in\SQ \}\in\BQ^D $,
$ \{  Q\cap V^{[d]\setminus D} \mid  Q\in\SQ \}\in\BQ^{[d]\setminus D} $
are merged using Eq. \eqref{eq:parallel-costs}.
Since these extensions performed worse than the basic version
in preliminary experiments, we did not investigate them further.

Our proposed \emph{parallelized} search step consists of two passes.
First, it runs the basic search step for \emph{all}
objects $ p\in[d] $ in \emph{parallel},
splitting $ d $ objects,
generating $ d $ matchings $  E^p \subset  V^{p}\times\SQ^{-p}$, 
yielding $ d $ proposed solutions, but ultimately leaving the input solution 
$ \SQ\in\BQ $ untouched.
Second, it splits and merges objects according to the
collected matchings $  E^p, p\in[d], $ in \emph{ascending}
order of their proposed solutions' 
objective value -- only accepting profits.
The more matchings yield profit despite being ``outdated'' 
(because prior matchings changed the solution), the more 
\emph{acceleration} is generated over the basic search step.
Both the basic and the parallel search steps \emph{sequentially}
solve only \emph{one} GM problem.
Additionally, the parallelized step always accepts the 
\emph{most profitable} proposal, though this is only a
\emph{secondary} source of acceleration,
as demonstrated in \cref{sec:ablation-study}.
Further details can be found in \cref{suppl:qap-local-search}.

\Subsubsection{Our novelty and related work.}
Again, the basic \cref{alg:qap-ls}
was first mentioned in \cite{bandelt2004local} 
targeting complete Multi-LAPs as a special case of the MDAP.
In \cite{bandelt2004local}, its extension to object subsets is 
investigated and our preliminary findings are confirmed 
for the Multi-LAP.
In computer vision, it was independently proposed in 
\cite{tang2017initialization,yan2016constrained}
for Multi-LAPs and in 
\cite{yan2013joint,yan2015consistency} for Multi-QAPs,
both only for the complete case.
Our parallelized extension is novel.%
\subsection{Swap Local Search}
\label{sec:swap-local-search}

\begingroup
\begin{algorithm}[t]
\addtolength{\hsize}{1.5em}%
\DontPrintSemicolon
\KwIn{Solution $ \SQ\in\BQ $}
\SetInd{0.7em}{0.7em}
  \While{stopping criterion not met}{
  	\For{$Q, R \in \SQ$}{
  		$x^* \leftarrow$ Approximately solve Eq. \eqref{eq:QPBO-swap} for cliques $Q, R$\;
  		\tcp{swap the vertices according to $x^*$}
  		$Q', R' \leftarrow \operatorname{swap}(x^{*}, Q, R)$ \;
        \tcp{update the clique representation}
  		$ \SQ^\prime \leftarrow (\SQ \setminus \{Q, R\}) \cup \{Q^\prime,R^\prime\}$ 	
 		
	  	\If{$ C ( \SQ' ) < C ( \SQ ) $}{
	  		$ \SQ \leftarrow \SQ'$ \tcp{accept if profitable}
	  	}
  	}
  }
\caption{Swap local search.}
\label{alg:swap-ls}
\end{algorithm}
\endgroup

\Subsubsection{Vertex swap.}
Given a solution $ \SQ\in\BQ $,
\emph{swapping} vertices of the same object between
two cliques $ Q,R\in\SQ $ maintains feasibility.
While fixing an object $ p\in[d] $, a \emph{vertex swap} subtracts
a vertex in $V^p$ covered by either clique from said clique and
adds it to the other, yielding two new cliques, $ Q' $ and $ R' $, 
see \cref{fig:gm-ls}.
For example, if both cliques $ Q $ and $ R $ cover the fixed object 
$ p\in[d] $ and
the vertices $ Q^p $ and $ R^p $
are swapped, it results in 
$ Q' = (Q\setminus \{Q^p\}) \cup \{R^p\} $
and
$ R' = (R\setminus \{R^p\}) \cup \{Q^p\} $.
If only one clique covers the fixed object $p$, only one vertex changes its clique, \ie,
$ Q' = Q \setminus \{Q^p\}$
and
$ R' = R \cup \{Q^p\} $.
Finally, both cliques remain unchanged if none covers the object.
After a swap, cliques still cover at most one vertex per object.
The new solution 
$ \SQ'= (\SQ\setminus\{Q,R\})\cup\{Q',R'\}$
is, thus, also feasible.
Furthermore, the objective value change induced by a vertex swap
decomposes into a sum of objective value changes
$ \delta C^{p,q}_{\text{swap}}(Q,R) $
between the fixed object $ p\in[d] $ and all other objects 
$ q\in[d]\setminus\{p\} $, \ie,
\begin{equation}
	C(\SQ^\prime) - C(\SQ) =
	\sum_{
		q \in [d] \setminus \{p\}
	}
	\delta C^{p,q}_{\textrm{swap}}(Q, R) \,. 
	\label{eq:single-swap-energy}
\end{equation}
For each other object $ q\in [d] \setminus \{ p \} $, the corresponding objective value change is the difference between the cost terms involving the cliques $ Q', R' $ \wrt the solution $ \SQ' $ and the cost terms involving the cliques $ Q, R $ \wrt the solution $ \SQ $, \ie,
\begin{equation}
  \delta C^{p,q}_{\textrm{swap}}(Q, R) \coloneqq C^{p, q}_{\left\{ Q', R' \right\}} (\SQ') - C^{p, q}_{\left\{ Q, R \right\}} (\SQ) \ ,
  \label{eq:single-swap-energy-1}
\end{equation}
where the sum of the respective cost terms $ C^{p,q}_{\left\{ Q, R \right\}}(\SQ) $ for the solution $ \SQ $ is given by
\begin{equation}
\begin{aligned}
C^{p,q}_{\left\{ Q, R \right\}} (\SQ) \coloneqq 
    &\sum_{S \in \{Q,R\}} \Biggl( \sum_{T \in \{Q,R\}} C^{p,q}(S,T) \\
    &+ \sum_{T \in \SQ \setminus \{Q,R\}}
        \bigl(C^{p,q}(S,T) + C^{p,q}(T,S)\bigr)
\Biggr) \ .
\label{eq:single-swap-energy-2}
\end{aligned}
\end{equation}
Here, $ C^{p,q}(S,T) $ denotes the costs of matching the objects $ p, q $ through the cliques $ S,T\in\SQ $, given by
\begin{equation}
  C^{p, q} (S,T) = \left\{  
\begin{array}{ll}
      C^{p,q}_{S^pS^q,T^pT^q}&;\, p<q \wedge p, q\in D(S, T) \\
      C^{q,p}_{S^qS^p,T^qT^p}&;\, q<p \wedge p, q\in D(S, T) \\
      0&;\text{\,otherwise}.
  \end{array} \right.
  \label{eq:single-swap-energy-3}
\end{equation}
A more detailed derivation is given in \cref{suppl:swap-local-search}. 

\Subsubsection{2-clique swap.}
We leverage the decomposition in \cref{eq:single-swap-energy} to
\emph{jointly} consider \emph{all possible} swaps
between two cliques $ Q,R\in\SQ $.
A 2-clique swap is a set of binary decision variables
$ x=\{x^p\}_{p\in[d]}\in\{0,1\}^{d} $,
deciding for every object $ p\in[d] $, whether the vertex swap 
fixing it should be performed ($ x^p=1 $), or not ($ x^p=0 $).
Performing all vertex swaps specified by a 2-clique swap $ x\in\{0,1\}^d $
yields, again, two new cliques $ Q' $ and $ R' $, which we denote by 
$ Q',R'=\operatorname{swap}(x; Q, R) $.
The optimal 2-clique swap $ x^* $ is a solution of the following 
quadratic pseudo-Boolean optimization (QPBO) problem~\cite{boroshammer02}:
\begin{equation}
	x^* \in 
	\argmin_{x \in \{0,1\}^d}
	\sum_{p \in [d]} 
	x^{p}
	\left(
	\sum_{q \in [d]} 
	\overline{x^{q}}	
	\cdot \delta C^{p,q}_{\textrm{swap}}(Q, R)
	\right) \,,
	\label{eq:QPBO-swap}
\end{equation}
where $\overline{x} = (1-x)$ denotes negation.
The implicit product $x^{p}\overline{x^{q}}$ ensures that only those $\delta C^{p,q}_{\textrm{swap}}(Q, R)$ are counted where exactly one vertex pair, 
either $ Q^p, R^p $ or $ Q^q, R^q $, was swapped between the cliques, but not both, as the latter does not change the costs between the objects $V^p$ and $V^q$.
This problem is NP-hard~\cite{rother07-cvpr},
which is why we address it using a
state-of-the-art
approximate solver \mbox{QPBO-I}~\cite{rother07-cvpr}.
Our swap local search \cref{alg:swap-ls} iterates over all
clique pairs, solving the problem from Eq. \eqref{eq:QPBO-swap} with 
\mbox{QPBO-I} and performing the obtained 2-clique swap if profitable.

Note that both the vertex and the 2-clique swap may
result in $Q'=\emptyset$ or $R'=\emptyset$,
whenever the objects covered by the cliques do not overlap, \ie $D(Q,R) = \emptyset$.
More specifically, this occurs when \emph{merging} two cliques into a single one is more profitable \wrt the costs than keeping them separate.
Conversely, calculating a 2-clique swap against an empty set can be used to split a single clique into two.

\Subsubsection{Sparsity constraints} can \emph{couple} decision variables in \cref{eq:QPBO-swap}, since forbidden matchings \emph{force} vertices of different graphs to be swapped \emph{jointly}, see \cref{fig:swap-partitions}.
If a vertex swap between two cliques $ Q, R\in \SQ $ along an object $ p\in D(Q) $ results in a forbidden matching, because there exists an object $ q\in D(R) $ \st either the matching of $ Q^p $ to $ R^q $ or of $ R^p $ to $ Q^q $ is forbidden, then the vertex swaps along both objects $ p, q $ must be performed jointly. 
We also say the objects $ p,q $ are coupled, \ie, $ x^{p}=x^{q} $.
This relationship is \emph{transitive}.
If the objects $ p,q $ and the objects $ q, r $ are coupled, then the objects $ p,r $ are also coupled, 
\ie, the vertex swaps along all three objects $ p, q, r $ must be performed jointly. 
Hence, the set of covered graphs $ D(Q) \cup D(R)\subset [d] $ admits a partition $ \SD\subset 2^{[d]} $ into sets of objects $ D\in \SD $ whose vertices must be swapped jointly.
Instead of \cref{eq:QPBO-swap}, one can, therefore, solve the \emph{smaller} quadratic pseudo-Boolean optimization problem:
\begin{equation}
	x^* \in 
	\argmin_{x \in \{0,1\}^{\SD}}
	\sum_{A \in \SD} 
	x^{A}
	\left(
	\sum_{B \in \SD} 
	\overline{x^{B}}	
	\cdot \delta C^{A,B}_{\textrm{swap}}(Q, R)
	\right) \,
	\label{eq:QPBO-sparse-swap}
\end{equation}
with costs 
\begin{equation}
	\delta C^{A,B}_{\textrm{swap}}(Q, R) =
	\sum_{p\in A} \sum_{q\in B} \delta C^{p,q}_{\textrm{swap}}(Q, R) \ .
	\label{eq:QPBO-sparse-swap-cost}
\end{equation}
In practice, we build the \emph{swap partition} $ \SD \subset 2^{[d]}$ from the \emph{bottom up} via \cref{alg:swap-partition}. The partition is initialized as singletons and then iteratively updated by merging two parts if each contain objects to be coupled.
If \emph{all} covered objects are coupled, problem \eqref{eq:QPBO-sparse-swap} contains a single decision variable and is, hence, \emph{trivially solvable}. 
In practice, this is, nonetheless, a common case, \eg, when one clique exclusively matches vertices in the head and the other in the tail of the worm, see \cref{sec:real_world_appl}.
For this reason alone, solving problem \eqref{eq:QPBO-sparse-swap} instead of \eqref{eq:QPBO-swap}  \emph{significantly accelerates} our previously proposed swap local search \cite{kahl2025towards}.

\begingroup
\begin{algorithm}[t]
  \addtolength{\hsize}{1.5em}%
  \DontPrintSemicolon
  \KwIn{Cliques $ Q,R\in \SQ $}
  \SetInd{0.7em}{0.7em}
  \tcp{initialize as singletons}
  $   \SD \leftarrow \left\{ \left\{ p \right\} : p\in D(Q)\cup D(R) \right\} $\;
  \For{$ p,q\in D(Q)\cup D(R) $ \st $ p<q $}{
    \If{$ C^{p,q}_{Q^pR^q,Q^pR^q} = \infty $ or $ C^{p,q}_{R^pQ^q,R^pQ^q} = \infty  $}{
      \tcp{merge object sets to be coupled}
      $ A \leftarrow A\in \SD  $ \st $ p\in A $ \;
      $ B \leftarrow B\in \SD  $ \st $ q\in B $ \;
      $ \SD \leftarrow \left( \SD \setminus \left\{ A,B \right\} \right) \cup \left\{ A \cup B \right\} $ \;
    }
  }
  \caption{Swap partition construction.}
  \label{alg:swap-partition}
\end{algorithm}
\endgroup

\begin{figure}[t!]
  \centering
  \begin{tikzpicture}[scale=0.85]
  \tikzset{vertex/.style = {shape=circle,draw}, inner sep=0.04cm}
  \tikzset{edge/.style = {-, dotted, tab20gray}}
  \tikzset{black/.style = {shape=circle,draw,fill=black,text=black, inner sep=0.09cm}}
  \tikzset{orange/.style = {shape=circle,draw,fill=tab20orange,text=tab20orange, inner sep=0.09cm}}
  \tikzset{green/.style = {shape=circle,draw,fill=tab20green,text=tab20green, inner sep=0.09cm}}
  \tikzset{blue/.style = {shape=circle,draw,fill=tab20blue,text=tab20blue, inner sep=0.09cm}}
  \tikzset{red/.style = {shape=circle,draw,fill=tab20red,text=tab20red, inner sep=0.09cm}}
  \tikzset{purple/.style = {shape=circle,draw,fill=tab20purple,text=tab20purple, inner sep=0.09cm}}
  \tikzset{purple/.style = {shape=circle,draw,fill=tab20purple,text=tab20purple, inner sep=0.09cm}}

  \tikzset{orangesquare/.style = {draw=tab20orange, rectangle, line width=0.4mm, minimum width=0.5cm, minimum height=0.5cm}}
  \tikzset{greensquare/.style = {draw=tab20green, rectangle, line width=0.4mm, minimum width=0.5cm, minimum height=0.5cm}}
  \tikzset{redsquare/.style = {draw=tab20red, rectangle, line width=0.4mm, minimum width=0.5cm, minimum height=0.5cm}}
  \tikzset{bluesquare/.style = {draw=tab20blue, rectangle, line width=0.4mm, minimum width=0.5cm, minimum height=0.5cm}}
  \tikzset{purplesquare/.style = {draw=tab20purple, rectangle, line width=0.4mm, minimum width=0.5cm, minimum height=0.5cm}}

  \node (SQ) at              (3.0, 3.60) {$ \SQ $};
  \node[greensquare]  (Q) at (3.0, 3.00) {$ Q $};
  \node[purplesquare] (R) at (3.0, 2.00) {$ R $};

  \node (V^p) at (-3, 3.60) {$ V^p $};
  \node (V^q) at (-2, 3.60) {$ V^q $};
  \node (V^r) at (-1, 3.60) {$ V^r $};
  \node (V^p) at ( 0, 3.60) {$ V^s $};
  \node (V^q) at ( 1, 3.60) {$ V^t $};
  \node (V^r) at ( 2, 3.60) {$ V^u $};

  \node[green] (Q_p) at (-3, 3) {};
  \node[green] (Q_q) at (-2, 3) {};
  \node[green] (Q_r) at (-1, 3) {};
  \node[green] (Q_s) at ( 0, 3) {};
  \node[green] (Q_t) at ( 1, 3) {};
  \node[green] (Q_u) at ( 2, 3) {};

  \node[purple] (R_p) at (-3, 2) {};
  \node[purple] (R_q) at (-2, 2) {};
  \node[purple] (R_r) at (-1, 2) {};
  \node[purple] (R_t) at ( 1, 2) {};
  \node[purple] (R_u) at ( 2, 2) {};

  \node[draw, rectangle, minimum width=2.25cm, minimum height=2.05cm] (box_A) at (-2.0, 2.70) {};
  \node[draw, rectangle, minimum width=0.60cm, minimum height=2.05cm] (box_B) at ( 0.0, 2.70) {};
  \node[draw, rectangle, minimum width=1.45cm, minimum height=2.05cm] (box_C) at ( 1.5, 2.70) {};

  \node (A) at (-2.0, 1.2) {$A$};
  \node (B) at ( 0.0, 1.2) {$B$};
  \node (C) at ( 1.5, 1.2) {$C$};

  \draw[line width=0.03cm, tab20green] (-0.5,3) ellipse (3.00cm and 0.35cm);
  \draw[line width=0.03cm, tab20purple] (-0.5,2) ellipse (3.00cm and 0.35cm);

  \draw[-, black] (Q_p) to (R_q);
  \draw[-, black] (R_q) to (Q_r);

  \draw[-, black] (Q_t) to (R_u);
  \draw[-, black] (R_t) to (Q_u);

  \node[fill=white, opacity=0, text opacity=1, align=left] (C) at (-5.,2.25) {$C^{p,q}_{Q^pR^q,Q^pR^q}=\infty$};
  \node (anchor_l) at (-3.5, 2.5) {};
  \node (anchor_r) at (-2.55, 2.5) {};
  \draw[->, dashed, line width=0.03cm]  (anchor_l) to[bend left=10] (anchor_r);

\end{tikzpicture}
  \caption{
    \textbf{Swap partition.}
    Shown are two cliques $ Q, R\in \SQ $ spanning six objects $ V^{p} $ to $ V^{u} $. Solid lines depict sparsity constraints, \ie, infinite linear costs for matching two vertices, \eg, $ Q^{p} $ and $ R^{q} $. These constraints induce the \emph{swap partition} with object subsets $ A=\left\{ p, q, r \right\}, B=\left\{ s \right\}$, and $ C=\left\{ t, u \right\} $ (solid rectangles). Swaps along objects within each subset must be performed \emph{jointly} to avoid \emph{forbidden} matchings, reducing the problem size from six to three decision variables.
  }
  \label{fig:swap-partitions}
\end{figure}

\Subsubsection{Our novelty and related work.}
The vertex swap was originally introduced for the 
MDAP with three objects~\cite{balas1991algorithm}
and later extended~\cite{murphey1997greedy,robertson2001set,karapetyan2011local}
to consider multiple swaps at the same time.
However, these propositions come without a method to solve the 2-clique swap problem efficiently.
Our contributions are: Formulating and solving the 2-clique swap as a quadratic pseudo-Boolean optimization problem, and efficiently adapting it to sparse problems.

\section{Sparse Synchronization}
\label{sec:synchronization}
\emph{GREEDA}, as an MGM meta-heuristic, applies to not just quadratic but arbitrary, \eg, linear, costs due to the
interchangeability of the GM solver.
It is therefore \emph{also} a synchronization method.
Synchronization can be seen as a constructive heuristic that reduces
MGM from arbitrary cost orders to a special case of the Multi-LAP.
It has two stages.
First, all $d(d-1)/2$ GM subproblems between all object pairs $ p,q\in[d] $ (with costs $ C^{p, q} $)
are \emph{independently} solved, potentially in parallel,
yielding matchings $  E^{p,q} \subset V^{p}\times V^{q}$.
These matchings induce an, a priori, cycle-\emph{inconsistent}
multi-matching $  E \coloneqq \bigcup_{p, q\in[d]} E^{p,q} $.
Second, a cycle-consistent multi-matching $  E^* $ sharing the largest possible
number of matchings with $ E $ is (approximately) recovered~\cite{arrigoni2017synchronization,bernard2021sparse,maset2017practical}, \ie,
\begin{equation}\label{eq:sync}
  E^*\in \argmax_{ F } | E\cap F| \qquad
  \text{s.t.}\  F\subset\overbar{E} \text{ cycle-consistent}.
\end{equation}
The recovery problem~\eqref{eq:sync} can be formulated as a \emph{synchronization Multi-LAP} between all objects $  V^{p} $ with costs
\begin{equation}
  M^{p,q}_{is,is}
  =
  \left\{
  \begin{array}{ rl }
    -1, & \text{ if } is\in E\,, \\
    0,  & \text{ otherwise}.       \\
  \end{array}
  \right.
  \label{eq:sync-dense-costs}
\end{equation}
An alternative, which agrees with Eq.~\eqref{eq:sync} for complete problems,
is to recover the projection \wrt the \emph{Hamming distance}, \ie, minimizing the
symmetric difference $ | E\Delta E^*| $~\cite{bernard2019synchronisation,pachauri2013solving}.
For details and the corresponding Multi-LAP, see \cref{suppl:hamm-dist}.

Synchronization's Achilles' heel is its \emph{ignorance of the original MGM objective}~\eqref{eq:MGM}
during recovery, \ie the second stage~\eqref{eq:sync}.
Especially for \emph{sparse} models with a majority of \emph{forbidden} matchings
of \emph{infinite} costs $ C^{p,q}_{is,is}=\infty $,
including just one in the final solution $  E^* $ renders it forbidden.
Instead, we propose to take \emph{forbidden matchings} into account by solving the
above-mentioned synchronization Multi-LAP with costs
\begin{equation}
  N^{p,q}_{is,is}
  =
  \left\{
  \begin{array}{ ll }
    \infty          & ,\text{ if } C^{p,q}_{is,is}=\infty\,, \\
    M^{p,q}_{is,is} & ,\text{ otherwise},                       \\
  \end{array}
  \right.
  \label{eq:sync-sparse-costs}
\end{equation}
instead of~\eqref{eq:sync-dense-costs}.
When used within our algorithms, the sparsity-aware costs \eqref{eq:sync-sparse-costs} always lead to \emph{feasible} solutions, containing only \emph{allowed} matchings.
We refer to this sparse synchronization variant of our algorithm as \term{GREEDA (Sync)}. 
In \term{GREEDA (Sync)}, we use SciPy's \cite{2020SciPy-NMeth} modified Jonker-Volgenant algorithm as described in \cite{crouse2016implementing} for the LAPs. 
  \cref{tab:greeda-overview-single} recapitulates both variants of our method, \emph{GREEDA} and \term{GREEDA (Sync)}.
\makeatletter
\@ifundefined{NC@rewrite@M}{%
  \newcolumntype{M}[1]{>{\raggedright\arraybackslash}m{#1}}}{}
\makeatother

\providecommand{\psrule}{\par\vspace{1.5pt}\hrule\vspace{1.5pt}}

\begin{table}[t]
  \centering
  \caption{%
    \textbf{Overview of GREEDA and GREEDA (Sync).}
    \caplegend{%
    Both variants consist of three stages: construction, followed by alternating GM-LS and SWAP-LS.
    \emph{GREEDA} minimizes \cref{eq:MGM}, while \term{GREEDA (Sync)} minimizes \cref{eq:sync-sparse-costs}.
    Since the latter is a sparse Multi-LAP, \term{GREEDA (Sync)} replaces the GM solver with an LAP solver.%
    }%
  }
  \label{tab:greeda-overview-single}
  \footnotesize
  \renewcommand{\arraystretch}{1.15}
  \begin{tabular}{@{}>{\centering\arraybackslash}m{17pt}@{\hspace{3pt}}M{63.5pt}M{71.5pt}M{73pt}@{}}
    \toprule
     & \textbf{Stage}
     & \textbf{Subproblem}\psrule\textbf{Subproblem solver}
     & \textbf{Decision variables} \\
    \midrule
     & \textbf{Construction}\newline(\cref{sec:construction}, \cref{alg:construction})
     & GM~\cref{eq:basic-costs}\psrule GM solver~\cite{hutschenreiter_fusionmoves_2021}
     & matching $E \subset V^{k+1} \times \SQ$ \\
    \addlinespace[2pt]
     & \textbf{GM-LS}\newline(\cref{sec:qap-local-search}, \cref{alg:qap-ls})
     & GM~\cref{eq:basic-costs}\psrule GM solver~\cite{hutschenreiter_fusionmoves_2021}
     & matching $E \subset V^{p} \times \SQ^{-p} $ \\
    \addlinespace[2pt]
    \multirow{-7}{*}[-2.47mm]{\rotatebox[origin=c]{90}{\shortstack{\textbf{\emph{GREEDA}}\\[-1pt] \cref{eq:MGM}}}}
     & \textbf{SWAP-LS}\newline(\cref{sec:swap-local-search}, \cref{alg:swap-ls})
     & QPBO~\cref{eq:QPBO-sparse-swap}\psrule QPBO-I~\cite{rother07-cvpr}
     & binary labeling\newline $x \in \{0,1\}^{\SD}$ \\
    \midrule
     & \textbf{Pairwise}\newline \textbf{pre-processing}\newline(\cref{sec:synchronization})
     & GM independently for costs~$C^{p,q}$\psrule GM solver~\cite{hutschenreiter_fusionmoves_2021}
     & matchings $E^{p,q} \subset V^{p} \times V^{q}$ \\
    \addlinespace[2pt]
    \multirow{-7}{*}[-5.66mm]{\rotatebox[origin=c]{90}{\shortstack{\textbf{\term{GREEDA (Sync)}}\\[-1pt] \cref{eq:sync-sparse-costs}}}}
     & \textbf{Synchronization}\newline(\cref{sec:synchronization})
     & Sparse Multi-LAP \cref{eq:sync-sparse-costs}\psrule \cref{alg:construction,alg:qap-ls,alg:swap-ls} with\newline LAP solver~\cite{2020SciPy-NMeth,crouse2016implementing}
     & multi-matching\newline $\SQ \in \BQ$ \\
    \bottomrule
  \end{tabular}
\end{table}

\Subsubsection{Our novelty and related work.}
Usually, synchronization methods neither discuss sparsity nor apply to
infinite costs%
\footnote{The only sparse synchronization method we could find, and this is only
through communication with the authors, is a greedy subroutine
of \cite{swoboda2019convex} not described in their paper.}%
.
In contrast, our method, applied to the synchronization Multi-LAP, efficiently and accurately handles the costs from Eq. \eqref{eq:sync-sparse-costs} and, therefore, respects forbidden matchings.

\section{Incomplete to Complete Transformation}
\label{sec:incomplete-transformation}
As outlined in \cref{sec:intro},
many algorithms address only complete MGM problems. Their application to the incomplete problems arising in practice
is often justified by a \emph{polynomial-time transformation} from incomplete to complete MGM,
typically without reference or further elaboration, \eg, \cite{yan2016constrained,yan2013joint,yan2015consistency}.
In the following, we formally state and analyze this transformation, and show that it is prohibitively expensive.

To transform an incomplete problem with objects $ V^p$, $ p\in[d] $, and
costs $ C^{p,q} $, $ p,q\in[d] $, into a complete one with objects 
of equal cardinalities, we introduce \emph{dummy} vertices, similar to the respective GM transformation in~\cite{haller2022comparative}.
More precisely, 
given that $|V|$ stands for the \emph{total} number of vertices in the original incomplete problem,
we add $ |V|-|V^p|\eqqcolon|\Delta^p| $
dummy vertices $ \Delta^p $ to each object $ V^p $, $p\in [d]$. As a result, we
construct a complete MGM problem with objects
\begin{equation}
	\widetilde{V}^{p} = V^p \cup \Delta^p, \quad  p \in[d]\,,
	\label{eq:complete-objects}
\end{equation}
and costs
\begin{equation}
	\widetilde{C}^{p,q}_{is,jt} =
	\begin{cases}
		C^{p,q}_{is,jt} &, \text{if $ i,s,j,t\in V  $ }\\
		0 &, \text{otherwise}.
	\end{cases}
	\label{eq:complete-costs}
\end{equation}
In this way, each vertex $i\in V^p$, $p\in [d]$, of the incomplete problem receives
a dummy vertex $\delta\in\Delta^q$ in each object $q\in[d]\backslash\{p\}$ of the complete problem. 
These dummy vertices turn each clique of the incomplete problem maximal, \ie they span all objects of the (complete) problem.
In turn, omitting matchings to dummy vertices defines the inverse transformation,
which leads us to the following

\begin{theorem}  \label{thm:transformation}
	Incomplete MGM is polynomial-time transformable to complete MGM.
\end{theorem}
The proof is based on the transformation \eqref{eq:complete-objects}--\eqref{eq:complete-costs} and given in \cref{suppl:details_incomplete_to_complete_transformation}.
While such a transformation could, in principle, be used to apply solvers for complete MGM to incomplete problems, doing so without further reasoning might be overly simplistic.
The transformation \eqref{eq:complete-objects}--\eqref{eq:complete-costs} adds 
$ \sum_{p\in[d]}(|V|-|V^p|)=(d-1)|V| $ dummy vertices, \ie,
$ d-1 $ times the size $|V|$ of the incomplete problem, 
which renders most existing algorithms impractical by considerably slowing them down.

Yet, the described transformation is \emph{minimal} within the class of \emph{clique-wise} transformations,
that is, the size of the respective complete problems cannot be decreased in general:
\begin{definition}\label{def:general-transformation}
	A problem class $P^A$ is transformable in polynomial time to the problem class $P^B$,
	if (i) there exists a mapping $T_{inst}\colon P^A\to P^B$ and
	(ii) for any $A\in P^{A}$ there exists a mapping $T_{sol}$ from the 
	optimal solutions of $T_{inst}(A)$ to the optimal solutions of $A$, 
	where all mappings are computable in polynomial time.
\end{definition}

Henceforth, problem always means \emph{optimization} problem.
Being very general, \cref{def:general-transformation} does not consider several practical issues: First, an optimal solution is rarely known a priori.
Therefore, the mapping $T_{sol}$ is usually defined for \emph{all} solutions that can potentially be optimal.
For MGM problems, however, \emph{any} solution can potentially be optimal, as there always exist costs that make it optimal. 
Second, if approximate algorithms are used to solve the reduced problem, then the optimum-to-optimum mapping, as required in \cref{def:general-transformation}, is insufficient. Instead, the mapping $T_{sol}$ must be \emph{monotone} \wrt the objective values of the mapped feasible solutions: Feasible solutions with lower objective values in $T_{inst}(A)$ correspond to feasible solutions with lower objective values in $A$. These properties are summarized in the following

\begin{definition}
	We say that there is a \emph{robust} polynomial-time \emph{transformation} of a problem class $P^A$ to the problem class $P^B$, if (i) there exists a mapping ${T_{inst}\colon P^A\to P^B}$ and (ii) for any $A\in P^{A}$ there exists a (\wrt the objective) \emph{monotone} mapping $T_{sol}$ from the feasible solutions of $T_{inst}(A)$ to the feasible solutions of $A$, where all mappings are computable in polynomial time.
\end{definition}

Let now $P^A$ and $P^B$ correspond respectively to the problem classes of incomplete and complete MGM.
Any feasible solution of an arbitrary MGM problem $A'$  is given by its vertex partition $\SQ^{A'}\in\BQ^{A'}$.
In the following, we define a subset of robust transformations, where each clique in the partition $\SQ^A$, $A\in P^A$, has its counterpart in the \emph{respective} partition $\SQ^B$, $B=T_{inst}(A)$:
\begin{definition}\label{def:clique-wise-transformation} 
	Assume $B=T_{inst}(A)$ and $\SQ^A:=T_{sol}(\SQ^B)$.
	We call the robust transformation $P^A\to P^B$ \emph{clique-wise},
	if the mapping $ T_{sol} $ is induced by a clique-wise mapping $ T_{clique}:\SQ^B\to\SQ^A 
	$ \st $ T_{clique}\big|_{\SQ^B\setminus T^{-1}_{clique}(\emptyset)} $ is injective,
	which directly implies $|\SQ^A| \le |\SQ^B|$.
\end{definition}
As it follows from the proof of \cref{thm:transformation}, the transformation \eqref{eq:complete-objects}--\eqref{eq:complete-costs} is clique-wise.
We are not aware of any non-clique-wise transformations, although we cannot exclude their existence.
\cref{def:clique-wise-transformation} leads to \cref{thm:minimality-of-transformation}, see \cref{sec:real_world_appl},
which we prove in the following.
\begin{proof}[Proof of \cref{thm:minimality-of-transformation}]
	Consider the feasible multi-matching in $A$ which leaves all vertices unmatched. This has $|\SQ^A|=|V|$ cliques representing individual vertices. On the other hand, the number of (maximal) cliques in any (complete) MGM solution in $B$ is equal to the number of vertices in each of its objects. Hence, the latter cannot be less than $|V|$ due to \cref{def:clique-wise-transformation}.
\end{proof}
\section{Experimental Validation}
\label{sec:experiments}
Our experiments comprise an \emph{ablation study} (\cref{sec:ablation-study}) and \emph{a comparison to other methods}, for \emph{direct MGM} (\cref{sec:comparison}) and \emph{sparse synchronization} (\cref{sec:synchronization-exp}).
Additionally, we consider another bioimaging application, the labeling of cortical sulcal pits~\cite{yadav2023population}, which we treat separately in \cref{sec:sulcal-graphs}, as it is a single large-scale instance for which we compare against the solutions published with~\cite{yadav2023population} instead of running competing methods ourselves.
Whereas the ablation study compares different variants of \emph{GREEDA}'s \emph{constructive heuristic} and \term{GM-LS}, see \cref{fig:pipeline}, in the comparison to others, we only utilize \emph{sequential} variants of all algorithms to ensure a fair runtime comparison. 
Contrary to most existing works, we only compare \emph{objective values} instead of \emph{ground-truth-based accuracies} since none of the considered algorithms explicitly optimize the latter or has access to the ground truth.
  For a more detailed discussion, see \cref{suppl:accuracy}.
We leave accuracy comparisons to the works addressing modeling and learning questions.
\Subsubsection{Datasets.}
We evaluate methods on the established datasets:
\emph{synthetic}, CMU  \emph{hotel} and \emph{house}, and \emph{worms}. All of them are widely used in the MGM literature
and accessible through the archive accompanying~\cite{swoboda2022structured}.

Additionally, to showcase our approach, we create the \term{worms-large} dataset \cite{worms_large_dataset}.
While the original \emph{worms} dataset considers only 3 to 10 worm objects per matching problem,  
\term{worms-large} features problems with 20 to 100 objects%
\footnote{We also substituted the worms-29 dataset from~\cite{kahl2025towards} with worms-30.}%
:%
\begin{samepage}
\begin{itemize}
    \item \tikz[remember picture]{\coordinate (worms-20);} \term{worms-20}
    \item \tikz[remember picture]{\coordinate (worms-30);} \term{worms-30}
    \item \tikz[remember picture]{\coordinate (worms-50);} \term{worms-50}
    \item \tikz[remember picture]{\coordinate (worms-100);} \term{worms-100}
\end{itemize}
\begin{tikzpicture}[remember picture, overlay]
    \draw [decorate, decoration={brace, amplitude=8pt}]
        ([xshift=55pt, yshift=5pt]worms-20) -- ([xshift=55pt, yshift=-1pt]worms-30)
        node[midway, right=8pt] {$\sim\!510\text{-}570$ vertices each};
    \draw [decorate, decoration={brace, amplitude=8pt}]
        ([xshift=55pt, yshift=5pt]worms-50) -- ([xshift=55pt, yshift=-1pt]worms-100)
        node[midway, right=8pt] {\;\;\;\;\;\;$558$ vertices each};
\end{tikzpicture}
\end{samepage}
Solving problems of this scale became necessary as part of a related research line~\cite{karg2025fully} focused on the unsupervised training of matching cost hyperparameters via Bayesian optimization. 
This prompted us to create and publish this dataset to facilitate benchmark comparisons.
Compared to the \emph{worms} dataset, the \term{worms-large} instances are significantly more sparse: on average, each vertex is restricted to 14 potential matches instead of 23. 
This degree of sparsity was chosen to substantially reduce problem size and runtime without compromising the solution quality of downstream matching tasks. 
Furthermore, the distinct hyperparameters used for these instances result in different cost values compared to the original \emph{worms} dataset.

With approximately 45M keypoint correspondences (compared to 3M in the original \term{worms-29} instance from~\cite{kahl2025towards}), the \term{worms-100} instances represent the largest MGM problems considered in computer vision to date. To promote further research in this regime, we publish the \term{worms-large} dataset by adding it to the benchmark collection of~\cite{swoboda2022structured}.
It is also directly accessible at \url{https://doi.org/10.11588/DATA/RTINUU}.

\subsection{Ablation Study (\cref{fig:results-ablation-ls})}
\label{sec:ablation-study}
Since the ablation study focuses on the underlying meta-heuristic rather than a particular instantiation, we restrict our analysis to \emph{GREEDA} and extend conclusions to \term{GREEDA (Sync)}.
We compare all combinations of the \emph{sequential}, \emph{parallel}, and \emph{incremental} solution construction methods from \cref{sec:construction} with the \emph{sequential} and \emph{parallel} \emph{GM local search} methods from \cref{sec:qap-local-search}.
For lack of a better solver, \emph{incremental} construction uses \emph{GREEDA} for the first 5 objects.
To showcase our parallelization, we additionally consider a straightforward \emph{``best-improvement''} parallelization of the GM local search:
Although it performs all GM-LS steps in parallel, it accepts only the most profitable one.
In this respect, it differs from the parallel local search as described in \cref{sec:qap-local-search}, which aims to accept \emph{all possible} profitable changes.
The experiments were run on \term{worms-10}, containing the largest MGM instances addressed in computer vision~\cite{swoboda2019convex} so far,
using 10 processing cores in the parallel setting.

\cref{fig:results-ablation-ls} depicts objective value with respect to runtime.
While all methods converge to virtually the same value, 
\emph{parallel} GM-LS converges fastest.
It also notably outperforms the \emph{``best-improvement''} parallelization.
Further details about the experimental setup and solver settings are provided in \cref{suppl:experiment-setup}.

\begin{figure}
  \centering
  \resizebox{\columnwidth}{!}{\input{floats/pgf/ablation-local-search.pgf}}
  \caption{\textbf{Convergence} of sequential and parallel \emph{GM local search} variants (see \cref{sec:qap-local-search}) after sequential and parallel construction (see \cref{sec:construction}), respectively. 
  For the incremental construction, only the final objective value is given to maintain the readability of the plot. 
  Though taking notably more time, this method does not result in a better solution. 
  Each method was run once on the first 10 problem instances of \term{worms-10} and the plot shows the averaged results.
  }
  \label{fig:results-ablation-ls}
\end{figure}

\Subsubsection{Influence of randomization.}
In our experiments, we found the influence of randomization to be negligible:
Across $100$ runs with different object orderings,
the expected cost of a single run exceeds the best of these $100$ runs by only $0.59\,\%$ after construction, 
and by only $0.27\,\%$ after the sequential GM-LS. 
We therefore do not recommend restarts in general, see \cref{suppl:randomization} for details.

\subsection{Comparison to Other Methods: Direct MGM}
\label{sec:comparison}
\makeatletter
\@ifundefined{NC@rewrite@L}{%
  \newcolumntype{L}[1]{>{\raggedright\arraybackslash}p{#1}}}{}
\@ifundefined{marklen}{\newlength{\marklen}}{}
\@ifundefined{markwd}{\newlength{\markwd}}{}
\makeatother
\settowidth{\marklen}{\cmark}
\settowidth{\markwd}{\xmark}
\ifdim\markwd>\marklen\setlength{\marklen}{\markwd}\fi

\providecommand{\colhead}[1]{%
  \renewcommand{\arraystretch}{0.9}%
  \begin{tabular}[b]{@{}c@{}}#1\end{tabular}}

\providecommand{\mcell}[2][]{%
  \settowidth{\markwd}{#2}%
  \makebox[\marklen][c]{#2}%
  \rlap{\kern-0.5\dimexpr\marklen-\markwd\relax$^{#1}$}}


\begin{table}[t]
  \centering
  \setlength{\tabcolsep}{5pt}
  \caption{%
    \textbf{Properties of the evaluated algorithms.}
    \caplegend{%
    \emph{Incomplete problems}: solved directly, without the transformation of \cref{sec:incomplete-transformation}.
    \emph{Sparse costs}: guaranteed finite cost under~\cref{eq:MGM}, \ie no forbidden matches, see \cref{table:sync_worms_10}.
    \emph{Cycle consistency}: guaranteed cycle-consistent solution $\SQ\in\BQ$.
    \emph{Large scale}: approximately solves \term{worms-100} within $96$\,h given $236$\,GB of RAM.%
    }%
  }
  \label{tab:guarantees}
  \footnotesize
  \renewcommand{\arraystretch}{1.15}
  \begin{tabular}{@{}L{3.4cm}cccc@{}}
    \toprule
     & \colhead{incomplete\\problems}
     & \colhead{sparse\\costs}
     & \colhead{cycle\\consistency}
     & \colhead{large\\scale} \\
    \midrule
    \multicolumn{5}{@{}l@{}}{\textbf{Direct MGM}} \\
    \addlinespace[2pt]
    GREEDA (ours)  & \mcell{\cmark} & \mcell{\cmark} & \mcell{\cmark} & \mcell{\cmark} \\
    MP-T~\cite{swoboda2019convex}             & \mcell{\cmark} & \mcell{\cmark} & \mcell{\cmark} & \mcell{\xmark} \\
    DS*~\cite{bernard2018ds}              & \mcell{\xmark} & \mcell{\xmark} & \mcell{\cmark} & \mcell{\xmark} \\
    Cao~\cite{yan2015multi}, \emph{pygmtools}~\cite{wang2024pygmtools}
                          & \mcell{\xmark} & \mcell{\xmark} & \mcell[b]{\xmark} & \mcell{\xmark} \\
    Floyd~\cite{jiang2020unifying}, \emph{pygmtools}~\cite{wang2024pygmtools}
                          & \mcell{\xmark} & \mcell{\xmark} & \mcell[b]{\xmark} & \mcell{\xmark} \\
    \midrule
    \multicolumn{5}{@{}l@{}}{\textbf{Synchronization}} \\
    \addlinespace[2pt]
    GREEDA (Sync) (ours)
                          & \mcell{\cmark} & \mcell{\cmark} & \mcell{\cmark} & \mcell{\cmark} \\[1pt]
    MP-T (Sync)~\cite{swoboda2019convex}      & \mcell{\cmark} & \mcell{\cmark} & \mcell{\cmark} & \mcell{\xmark} \\
    PPI~\cite{chen2018projected}              & \mcell{\cmark} & \mcell{\xmark} & \mcell{\cmark} & \mcell[a]{(\cmark)} \\
    Stiefel~\cite{bernard2021sparse}          & \mcell{\cmark} & \mcell{\xmark} & \mcell{\cmark} & \mcell[a]{(\cmark)} \\
    Spectral~\cite{pachauri2013solving}         & \mcell[c]{\cmark} & \mcell{\xmark} & \mcell{\cmark} & \mcell[a]{(\cmark)} \\
    MST~\cite{tang2017initialization}              & \mcell[d]{\cmark} & \mcell{\xmark} & \mcell{\cmark} & \mcell[a]{(\cmark)} \\
    MST-Star~\cite{yan2015consistency}         & \mcell[d]{\cmark} & \mcell{\xmark} & \mcell{\cmark} & \mcell[a]{(\cmark)} \\
    \bottomrule
  \end{tabular}

  \smallskip
  \begin{flushleft}
    \scriptsize
    $^{a}$~Scalable \wrt \cref{eq:sync-dense-costs}, but the inability to handle sparse costs limits practical application as pairwise pre-processing becomes infeasible for larger complete problems.\\
    $^{b}$~Optimizes for cycle consistency, but relaxes the constraint. In our experiments, the returned matchings are in general not cycle-consistent.\\
    $^{c}$~Only through the extension~\cite{bernard2019synchronisation} to partial matchings.\\
    $^{d}$~Only in our re-implementation. The original works address complete problems.
  \end{flushleft}
\end{table}

\Subsubsection{Algorithms.}
We compare our method \emph{GREEDA} to the state-of-the-art MGM methods \term{MP-T}~\cite{swoboda2019convex} and \emph{DS*}~\cite{bernard2018ds}. 
The latter only applies to complete and dense MGM problems, while the former also applies to incomplete and sparse MGM problems. 
We also include the synchronization methods \emph{PPI}, \emph{Stiefel}, and \emph{Spectral}, see \cref{sec:synchronization-exp}.

We additionally evaluated the \emph{Cao}~\cite{yan2015multi} and \emph{Floyd}~\cite{jiang2020unifying} algorithms from the \emph{pygmtools} package~\cite{wang2024pygmtools}. 
Like \emph{DS*}, they are restricted to complete, dense MGM problems. 
However, their current implementation does not enforce cycle consistency, and their solutions were, in our experiments, in general cycle-inconsistent.
We therefore exclude them from the main comparison and report their results separately in \cref{suppl:detailed-results}. 
These results show that their solutions are still inferior to \emph{GREEDA} in terms of objective value despite being cycle-inconsistent.
\cref{tab:guarantees} gives an overview of the properties of all algorithms compared in \cref{sec:comparison,sec:synchronization-exp}.
\Paragraph{Implementation Details.}
We use the original author implementations for both \emph{DS*} (Matlab) and \term{MP-T} (C++). \emph{GREEDA} is also implemented in C++. 
Details regarding the synchronization methods can be found in \cref{sec:synchronization-exp}.

\begin{figure*}[!t]
  \captionsetup[subfloat]{margin={35pt, 0pt}}
  \centering
  \resizebox{\textwidth}{!}{\input{floats/pgf/fig-compare-legend.pgf}}%
  \\
  \subfloat[]{%
      \input{floats/pgf/fig-compare-a-worms30-full.pgf}%
    \label{subfig:results-compare-a}%
  }\hfill
  \subfloat[]{%
      \input{floats/pgf/fig-compare-b-worms30-zoomed.pgf}%
    \label{subfig:results-compare-b}%
  }\hfill
  \\
  \subfloat[]{%
      \input{floats/pgf/fig-compare-c-synthetic-density-12.pgf}%

    \label{subfig:results-compare-c}%
  }\hfill
  \subfloat[]{%
      \input{floats/pgf/fig-compare-d-hotel-8.pgf}%
    \label{subfig:results-compare-d}%
  }

  \caption{
    \textbf{Method Comparison.}
    Comparison of objective value \wrt runtime.
    The objective is plotted on a log scale and offset by the cycle-\emph{inconsistent} solution's objective $C^\mathrm{inc}$, obtained by independently solving the $d(d-1)/2$ pairwise GM problems.
    This value approximates, since \cite{hutschenreiter_fusionmoves_2021} is itself approximate, a lower bound to Eq.~\eqref{eq:MGM} given by the MGM relaxation that ignores cycle consistency.\newline
    For the sequential variant of \emph{GREEDA}, we report the best and worst result over 10 runs.
    \cref{subfig:results-compare-b} shows the first minutes of \cref{subfig:results-compare-a}, illustrating the considerable difference in construction time between \emph{GREEDA} and \term{MP-T} on large problems.
    In both, we only compare methods that can find allowed solutions for sparse problems.
    \cref{subfig:results-compare-c} and \cref{subfig:results-compare-d} report results averaged over all instances of the 12-object \emph{synthetic density} and the 8-object \emph{hotel} datasets, respectively.
    Since the pre-processing stage of synchronization methods makes runtime comparisons with direct MGM methods difficult, we report only their objective values.
    Colored dots mark the final objective and runtime of each direct MGM method.
    \emph{DS*} results are shown only in \cref{subfig:results-compare-c}, since the other datasets contain incomplete problems.
    Note that even if \emph{GREEDA} is run 10 times sequentially to obtain the result of \emph{GREEDA (best)}, it remains the fastest algorithm.
  }
  \label{fig:results-compare}
\end{figure*}

\Subsubsection{Results.} 
For the \term{worms-50} and \term{worms-100} datasets, the only true competitor, \term{MP-T}, runs out of memory on a machine with 236\,GB of RAM.
By contrast, \emph{GREEDA} finds a feasible solution after 8 minutes on \term{worms-50} and 30 minutes on \term{worms-100}, and converges after about 2 hours using 25\,GB of RAM and after about 16 hours using 100\,GB of RAM, respectively.

Otherwise, for brevity, we only provide the comparison for three representative datasets, and refer to \cref{suppl:detailed-results} for a full list of results with all algorithms and datasets (in total 331 problem instances).

\cref{fig:results-compare} shows the costs-runtime plots for different methods. 
Its counterpart in the preliminary version \cite{kahl2025towards} mistakenly reports results for \emph{GREEDA} without SWAP-LS. Here, we show results for the full method, including the improved SWAP-LS.
For synchronization methods, we consider only the final objective value, as the need to solve $ d(d-1)/2 $ pairwise problems during pre-processing makes runtime comparisons with \emph{direct} MGM methods difficult.
Again, \emph{GREEDA} consistently outperforms or is on par with competitors in terms of final objective value, and is an order of magnitude faster in terms of runtime. 
In particular, for our new, large-scale \term{worms-30} dataset, see \cref{subfig:results-compare-a}, \emph{GREEDA} achieves a better solution much faster than the direct competitor \term{MP-T}, improving construction time from an hour%
\footnote{Twice the time reported in~\cite{kahl2025towards}. 
Following correspondence with the \term{MP-T} authors~\cite{swoboda2019convex}, we now use solver settings that yield much better objective values at the cost of longer runtime.}
to just two minutes (see \cref{subfig:results-compare-b}).
While \term{MP-T} fails to improve upon its initial solution significantly, our local search subroutines \term{GM-LS} and \term{SWAP-LS} quickly and significantly improve its result.
\cref{subfig:results-compare-c} includes a comparison to \emph{DS*}, which yields high-quality results but is substantially slower than the other methods, partially due to its Matlab implementation. Since it is a solver for complete problems, it cannot be applied to the \emph{hotel}, \emph{house}, or \emph{worms} datasets, as the incomplete-to-complete problem transformation would increase the runtime even further, see \cref{sec:incomplete-transformation}.
All plots in \cref{fig:results-compare} additionally show the application of \term{GM-LS} and \term{SWAP-LS} to the results of the other methods.
Although the resulting objective values are at times comparable to \emph{GREEDA}, optimization takes significantly longer.
As mentioned in \cref{sec:related_work}, synchronization methods may introduce blunders that substantially impact the resulting cost. This can be seen in \cref{subfig:results-compare-d}, where all synchronization methods return a solution that is worse than the trivial one, which would leave all vertices unmatched. Note that the corresponding problem instance is \emph{dense}. 
For \emph{sparse} synchronization, see \cref{sec:synchronization-exp}.
\subsection{Comparison to Other Methods: Synchronization}
\label{sec:synchronization-exp}
Here, we compare algorithms \wrt the synchronization objective \cref{eq:sync-dense-costs}.

\Subsubsection{Algorithms.}
We compare our sparse synchronization algorithm \term{GREEDA (Sync)} to the following synchronization algorithms:
The projected power iteration (\emph{PPI}) \cite{chen2018projected}, \emph{Stiefel} synchronization \cite{bernard2021sparse}, and the \emph{Spectral} approach \cite{pachauri2013solving} extended by the Successive Block Rotation Algorithm~\cite{bernard2019synchronisation} to yield incomplete matchings.
We also compare to the conceptually similar constructive heuristics based on arbitrary (\emph{MST}) \cite{tang2017initialization} and star-shaped (\term{MST-Star}) \cite{yan2015consistency} spanning trees. 
Additionally, we adapt \term{MP-T}~\cite{swoboda2019convex} as a synchronization method \term{MP-T (Sync)} by applying it to the Multi-LAP~\eqref{eq:sync} with sparsity-aware costs~\eqref{eq:sync-sparse-costs}.

Except for \term{MP-T (Sync)}, no competitor is guaranteed to return only allowed matchings for sparse problems. 

For all synchronization methods, we solve the (potentially sparse) GM subproblems arising in the pre-processing stage using the state-of-the-art GM solver~\cite{hutschenreiter_fusionmoves_2021,fusionmovesprojectpage}, which we turn from a randomized algorithm into a deterministic one by fixing its randomization seed.
Further details about the experimental setup and solver settings are given in \cref{suppl:experiment-setup}.

\Paragraph{Implementation Details.}
While \term{GREEDA (Sync)} and \term{MP-T (Sync)} are available in C++, \emph{Stiefel} and \emph{Spectral} are Matlab implementations. We use the original software from the authors for all but \emph{PPI}, \emph{MST}, and \term{MST-Star}, which we re-implemented ourselves in Python.

\begin{figure*}[t!]
  \captionsetup[subfloat]{margin={45pt, 0pt}}
  \centering%
  \resizebox{\textwidth}{!}{\input{floats/pgf/sync-scatter/fig-sync-compare-legend.pgf}}%
  \\
  \subfloat[]{%

      \input{floats/pgf/sync-scatter/fig-sync-compare-hotel-16.pgf}%
    \label{subfig:results-sync-compare-a}%
  }\hfill
  \subfloat[]{%
      \input{floats/pgf/sync-scatter/fig-sync-compare-house-16.pgf}%
    \label{subfig:results-sync-compare-b}%
  }\hfill
  \subfloat[]{%
      \input{floats/pgf/sync-scatter/fig-sync-compare-synthetic-density-16.pgf}%
    \label{subfig:results-sync-compare-c}%
  }\hfill
  \subfloat[]{%
      \input{floats/pgf/sync-scatter/fig-sync-compare-synthetic-deform-16.pgf}%
    \label{subfig:results-sync-compare-d}%
  }

  \caption{
    \textbf{Synchronization Comparison on Smaller, Dense Problems.}
    Comparison of the \emph{sparse} \eqref{eq:sync-sparse-costs}, or equivalently, since all instances are dense, the \emph{standard} \eqref{eq:sync-dense-costs} synchronization Multi-LAP objective \wrt total runtime, averaged over all instances of the respective datasets.
    All methods except \emph{MST} and \term{MST-Star} directly optimize the Multi-LAP objective \eqref{eq:sync-dense-costs}.
    The dashed line marks the end of the pre-processing stage, in which all GM subproblems are solved independently.
  }
  \label{fig:sync-scatter-plot}
\end{figure*}

\begin{table}
  \centering
  \caption{%
    \textbf{Sparse synchronization \eqref{eq:sync-sparse-costs} on worms-10.}
    \caplegend{%
    See \cref{sec:synchronization}. Results averaged over 10 instances.
    We report the Multi-LAP (\term{M-LAP}), \cref{eq:sync-dense-costs}, and \emph{MGM}, \cref{eq:MGM}, objectives, as well as the number of forbidden matchings (\emph{\#forb.}) in the final solution.
    Since neither \emph{PPI}, \emph{Spectral}, nor \emph{Stiefel} are able to deal with the sparse synchronization problem~\eqref{eq:sync-sparse-costs} directly, we resort to \emph{soft constraints} by replacing infinite $N^{p,q}_{is,is}$ values in~\eqref{eq:sync-sparse-costs} with a finite value $\alpha>0$, for $ \alpha\in \left\{ 1,10,50 \right\} $.
    Among the three, only \emph{PPI} is able to reduce the number of forbidden matchings \emph{\# forb.} significantly.
    Nevertheless, only \term{MP-T (Sync)} and \term{GREEDA (Sync)} are able to obey sparsity constraints, which renders the MGM objective infinite for all other methods.
    As baselines, we also list results of direct optimization algorithms \term{MP-T} and \emph{GREEDA}, the latter over 10 runs to account for randomization.
    }%
  }
    \begin{tabular}{clcc}
      \toprule
      Inf.Value & Method & M-LAP\textdownarrow  & \# forb./MGM\textdownarrow  \\
      \cmidrule{1-4}
      \multirow{6}{*}{\shortstack{$ \alpha = 0 $ \\(\cref{eq:sync-dense-costs})}}
                       & PPI  \cite{chen2018projected} & -18428 & 791\phantom{00}/\phantom{00}$\infty$ \\
                       & Spectral  \cite{pachauri2013solving,bernard2019synchronisation} & -18521 & 713\phantom{00}/\phantom{00}$\infty$ \\
                       & Stiefel  \cite{bernard2021sparse} & -18461 & 580\phantom{00}/\phantom{00}$\infty$ \\
                       & MP-T (Sync)  \cite{swoboda2019convex} & -18497 & 599\phantom{00}/\phantom{00}$\infty$ \\
                       & GREEDA (Sync)(worst) & -18622 & 755\phantom{00}/\phantom{00}$\infty$ \\
                       & GREEDA (Sync)(best) & \bfseries -18669 & 692\phantom{00}/\phantom{00}$\infty$ \\
      \hline
      \hline
      \multirow{3}{*}{$\alpha = 1$}
                       & PPI  \cite{chen2018projected} & \textit{-18390} & 275\phantom{00}/\phantom{00}$\infty$ \\
                       & Spectral  \cite{pachauri2013solving,bernard2019synchronisation} & -3913 & 15426/\phantom{00}$\infty$ \\
                       & Stiefel  \cite{bernard2021sparse} & -14213 & 2493\phantom{0}/\phantom{00}$\infty$ \\
      \hline
      \multirow{3}{*}{$\alpha = 10$}
                       & PPI  \cite{chen2018projected} & \textit{-17576} & 303\phantom{00}/\phantom{00}$\infty$ \\
                       & Spectral  \cite{pachauri2013solving,bernard2019synchronisation} & -1118 & 17745/\phantom{00}$\infty$ \\
                       & Stiefel  \cite{bernard2021sparse} & -10531 & 3663\phantom{0}/\phantom{00}$\infty$ \\
      \hline
      \multirow{3}{*}{$\alpha = 50$}
                       & PPI  \cite{chen2018projected} & \textit{-16586} & 389\phantom{00}/\phantom{00}$\infty$ \\
                       & Spectral  \cite{pachauri2013solving,bernard2019synchronisation} & -1061 & 17731/$\phantom{00}\infty$ \\
                       & Stiefel  \cite{bernard2021sparse} & -10700 & 3643\phantom{0}/\phantom{00}$\infty$ \\
      \cmidrule{1-4}
      \multirow{5}{*}{\shortstack{ $ \alpha = \infty $ \\ (\cref{eq:sync-sparse-costs})}}
                       & MST  \cite{tang2017initialization} & -14740 & 360\phantom{00}/\phantom{00}$\infty$ \\
                       & MST-Star  \cite{yan2015consistency} & -16349 & 213\phantom{00}/\phantom{00}$\infty$ \\
                       & MP-T (Sync) & -18333 & {\bfseries 0}/-3108350 \\
                     & GREEDA (Sync)(worst) & -18451 & {\bfseries 0}/-3132603 \\
                     & GREEDA (Sync)(best) & \bfseries -18510 & \bfseries 0/-3136748 \\
      \hline
      \hline
      baseline & MP-T  \cite{swoboda2019convex} & -  & {\bfseries 0}/-2393022 \\
               & GREEDA (worst) & -  & {\bfseries 0}/-3201256 \\
               & GREEDA (best) & -  & \bfseries 0/-3208824 \\
      \hline
      \bottomrule
    \end{tabular}
  \label{table:sync_worms_10}
\end{table}

\Subsubsection{Results.}
Quantitative comparisons on the large-scale, sparse \term{worms-10} instances are summarized in \cref{table:sync_worms_10}, see \cref{suppl:detailed-sync-results} for detailed results. 
\term{GREEDA (Sync)} achieves the best performance across the reported metrics. 
The synchronization methods \emph{PPI}, \emph{Spectral}, and \emph{Stiefel} consistently produce forbidden matchings in all evaluated experiments, even when employing soft constraints that penalize infinite costs with a positive constant $\alpha > 0$. 
Similarly, the spanning-tree-based methods \emph{MST} and \term{MST-Star} fail to avoid forbidden matchings even for $\alpha=\infty$. 
Moreover, the objective values attained in this setting are inferior to those of \term{GREEDA (Sync)} and \term{MP-T (Sync)}, even with respect to the dense synchronization objective~\eqref{eq:sync-dense-costs}.

Additionally, we observed that using synchronization solutions, including those containing forbidden matchings, as starting points for the \emph{GREEDA} local search not only yields \emph{allowed} multi-matchings, but also converges to highly consistent MGM objective values~\eqref{eq:MGM}, regardless of the initialization.
We attribute this robustness to the ability of the local search to efficiently explore the exponentially large neighborhood of a current solution. This effectiveness, in turn, stems from the efficiency of the underlying graph matching solver~\cite{hutschenreiter_fusionmoves_2021} utilized in our method.

\cref{fig:sync-scatter-plot} shows the synchronization Multi-LAP objective \eqref{eq:sync-dense-costs} against runtime for all methods on the small-scale, dense datasets \term{house-16}, \term{hotel-16}, \term{synthetic deform-16}, and \term{synthetic density-16}, averaged over all 10 instances of each dataset.
Since all problems are dense, the \emph{sparse} synchronization Multi-LAP objective \eqref{eq:sync-sparse-costs} coincides with the \emph{standard} synchronization objective \eqref{eq:sync-dense-costs}, and all methods%
\footnote{\emph{Spectral} optimizes the Hamming distance; for dense problems, this is equivalent to~\eqref{eq:sync-dense-costs}.}
except \emph{MST} and \term{MST-Star} optimize this objective explicitly.
\term{GREEDA (Sync)} consistently outperforms the dedicated synchronization methods \emph{PPI}, \emph{Stiefel}, and \emph{Spectral} in terms of objective value.
The spanning-tree-based heuristics \emph{MST} and \term{MST-Star} underperform with respect to the Multi-LAP objective \eqref{eq:sync-dense-costs},
though they achieve better values for the MGM objective \eqref{eq:MGM}, for which they partially optimize, see \cref{suppl:detailed-sync-results} for details. 
The only method that provides competitive results \wrt the Multi-LAP objective value \eqref{eq:sync-dense-costs} is the synchronization adaptation of \term{MP-T}, namely \term{MP-T (Sync)}.
In terms of objective value, \term{GREEDA (Sync)} outperforms \term{MP-T (Sync)} on all datasets except \emph{hotel} and \emph{house}, where it is consistently second best and trails \term{MP-T (Sync)} only by a small margin, while being at least half an order of magnitude faster.
For example, on \term{hotel-16}, the average percentage difference relative to \term{MP-T (Sync)} is $0.84\%$ for \term{GREEDA (Sync)(best)}, $1.69\%$ for \term{GREEDA (Sync)(worst)}, and $3.59\%$ for the closest remaining competitor, \emph{PPI}.
In terms of runtime, \term{GREEDA (Sync)} outperforms \term{MP-T (Sync)} by an order of magnitude on most datasets, with exceptions limited to small or simple cases (\term{hotel-4}, \term{house-4}, \term{synthetic outlier-4 \& -8}, and \emph{synthetic complete}).
We provide results for all datasets in \cref{suppl:detailed-sync-results}.

Overall, \term{GREEDA (Sync)} consistently offers the best trade-off between Multi-LAP objective value \eqref{eq:sync-sparse-costs} and runtime across the evaluated datasets.
In particular, on the large \emph{worms} instances, it consistently outperforms its only competitor in terms of objective value while being an order of magnitude faster.

\subsection{Sulcal Graphs}
\label{sec:sulcal-graphs}
\begin{figure}
  \centering
  \input{floats/tikz-sulcal-graphs}
  \caption{%
      \textbf{Sulcal pit labeling as MGM.}
      A \emph{sulcal pit} is the deepest point of a sulcal basin in the brain's cortical grooves (sulci).
      Representing these pits as vertices and connecting any two whose basins meet with an edge yields a \emph{sulcal graph}~\cite{IMsulcal_graphs}.
      Sulci are organized in a common basic configuration~\cite{ribas2010cerebral}. 
      However, the precise sulcal pattern varies between individuals, so the graphs differ in size and connectivity.
      Matching them cycle-consistently across many brains (colored as in \cref{fig:c-elegans}) enables population-level analyses, \eg, for identifying signs of disorders~\cite{IMsulcal_graphs,yadav2023population}.
  }
  \label{fig:sulcal-graphs}
\end{figure}

To showcase our method's applicability to other bioimaging applications, we evaluate it on the task of labeling cortical sulcal pits~\cite{yadav2023population}.
The corresponding publication~\cite{yadav2023population} introduced this task as an MGM problem across $d = 137$ human brains with {$77$--$101$} keypoints each, using the OASIS database~\cite{marcus2007open}, see \cref{fig:sulcal-graphs}.

Apart from our \term{worms-large} dataset, this is one of the largest MGM problems addressed in the literature, and scalability is considered the primary obstacle in~\cite{yadav2023population}.
Most established MGM solvers had to be excluded from that study because their runtime or memory consumption was prohibitive for this many objects.
Furthermore, several of the evaluated solvers return matchings that are not cycle-consistent, which is a limitation the authors accepted merely to obtain a result.
Lastly, the problem is naturally incomplete, as objects vary in size due to inter-individual variability and imperfect segmentations.
Since most of their considered methods require complete MGM problems, they pad each object with zero-cost dummy vertices to the size of the largest one. 
This yields smaller objects than the transformation in \cref{sec:incomplete-transformation} but forces equally sized objects into a full bijection, 
potentially incorrectly matching spurious or subject-specific nodes without a counterpart. 
GREEDA, in contrast, handles incomplete problems directly, but we evaluate it on the padded instance to ensure comparability of results.

As \cref{fig:sulcal-graph-matching} shows, \emph{GREEDA} outperforms all competitors included in the codebase of~\cite{yadav2023population}, requiring about $10$\,h and up to $49$\,GB of memory to fully converge. 
However, results outperforming all competitors were already achieved after $30$\,min, when the \emph{construction} (\cref{alg:construction}) phase finished.

As the brain images in this dataset had already been roughly aligned to a common sphere~\cite{yadav2023population}, we also created a sparsified instance that restricts matches to the $k = 20$ geodesically nearest keypoints in every other object.
This reduced the runtime by a factor of $17$ and the memory consumption by a factor of $5$.

\begin{figure}[t]
  \centering
  \input{floats/pgf/sulcal_graphs_comparison.pgf}
  \caption{%
    \textbf{Sulcal graph matching.}
    Attained objective values of the MGM problem formed by the sulcal graphs of 137 human brains
    from the OASIS database~\cite{marcus2007open}, as posed in~\cite{yadav2023population}.
    We compare to a subset of results provided in~\cite{yadav2023population}, omitting non-cycle-consistent ones. The objective value is printed inside each bar. The runtimes of \emph{GREEDA} are given in brackets.
  }
  \label{fig:sulcal-graph-matching}
\end{figure}

\section{Conclusions and Future Work}
\label{sec:conclusion}

In this paper, we addressed the challenging problem of large-scale multi-graph matching under sparsity and incompleteness constraints, 
a setting motivated by real-world bioimaging applications. 
We introduced GREEDA, an efficient solver that operates natively on sparse and incomplete problem formulations, as a direct or a synchronization method. 
By decomposing the heavy multi-graph assignment task into localized graph matching subroutines, GREEDA successfully scales to hundreds of objects
while maintaining cycle consistency and yielding state-of-the-art optimization accuracy at an execution speed that is orders of magnitude faster than existing competitors.

A promising direction remains for future exploration. As discussed in \cref{sec:intro}, a modern approach to MGM includes learning the matching costs directly from data using neural networks. Integrating GREEDA's exceptionally fast solving routine directly within such an end-to-end learning loop represents an exciting path forward. This could unlock more accurate and unsupervised correspondences in biological volumetric scans while operating inside the tight time constraints typically required by training cycles.

\section*{Acknowledgements}
This work was supported by the German Research Foundation projects 498181230 and 539435352. The authors further acknowledge facilities for high-throughput calculations bwHPC of the state of Baden-Württemberg (DFG grant INST 35/1597-1 FUGG) as well as the Center for Information Services and High Performance Computing (ZIH) at TU Dresden.
\bibliographystyle{IEEEtran}
\bibliography{main}
\ifincludesupplement
    \clearpage
    \onecolumn
\setcounter{page}{1}
\appendices
\crefalias{section}{appendix}
\crefalias{subsection}{appendix}

\begin{center}
	\vspace*{0.5\baselineskip}
	{\LARGE\bfseries Supplementary Material\par}
	\vspace{1.5\baselineskip}
\end{center}

\titlecontents{section}
[0.0em] 
{\vspace{.5\baselineskip}} 
{\thecontentslabel\protect\hspace{1em}} 
{} 
{\titlerule*[0.5pc]{.}\contentspage} 

\setcounter{figure}{0}
\setcounter{table}{0}
\setcounter{algocf}{0}

\counterwithin{figure}{section}
\counterwithin{table}{section}
\counterwithin{algocf}{section}

\startcontents[appendix-toc]
\section*{Table of Contents}
\printcontents[appendix-toc]{}{0}[1]{}

\FloatBarrier

\section{Feasible Solution Construction: Details}
\label{suppl:construction}

\Subsubsection{Merge operation.}
The merge operation between two partial solutions
$ \SA\in\BQ^{D_A}  $ and $ \SB\in\BQ^{D_B} $ 
of disjoint $ D_A \cap D_B=\emptyset $ object subsets
$ D_A,D_B\subset[d] $
specified via the matching $ E\subset\SA\times\SB $ is formally defined as
\begin{equation}
  \merge(\SA,\SB;E)
  \coloneqq 
  \{A\cup B \, \mid \, AB\in E \}
  \ \cup \
  \{A \, \mid \, \forall B\in\SB: AB\notin E \}
  \ \cup \
  \{B \, \mid \, \forall A\in\SA: AB\notin E \}   \,.
  \label{eq:merge}
\end{equation}
The merge between an object $ V^p $ and a partial solution 
$ \SQ\in\BQ^D $, $ p\notin D\subset[d] $, specified by the matching 
$ E\subset V^p\times\SQ $ is a special case
\begin{equation}
  \merge(V^p,\SQ;E)
  \coloneqq
  \merge(
  \SV^p, \SQ; E') \,,
  \label{eq:single-object-merge}
\end{equation}
where
$ \SV^p = \left\{ \{i\} \,\mid\, i\in V^p \right\} $
and 
$ E' = \left\{ \{i\}\cup Q \,\mid\, iQ\in E \right\} $.
\Subsubsection{Extensions and parallelization.}
\begin{figure}
  \centering
		\subfloat[][Basic]{%
			\adjustbox{valign=b, scale=0.75}{\begin{tikzpicture}[
  vertex/.style = {shape=circle,draw, inner sep=0.04cm},
  blacked/.style = {shape=circle,draw,fill=black,text=black, inner sep=0.06cm},
  edge/.style = {-, dashed, tab20gray},
  greenvertex/.style = {draw=tab20green,shape=circle, text=black, line width=0.3mm, inner sep=0.04cm},
  bluevertex/.style = {draw=tab20blue,shape=circle, text=black, line width=0.3mm, inner sep=0.04cm},
  orangevertex/.style = {draw=tab20orange,shape=circle, text=black, line width=0.3mm, inner sep=0.04cm},
  greendummy/.style = {regular polygon, regular polygon sides=4, draw=tab20green,  line width=0.3mm, inner sep=-0.01cm},
  bluedummy/.style = {regular polygon,  regular polygon sides=4, draw=tab20blue,   line width=0.3mm, inner sep=-0.01cm},
  orangedummy/.style = {regular polygon,regular polygon sides=4, draw=tab20orange, line width=0.3mm, inner sep=-0.01cm},
]

\begin{scope}[scale=0.5]
    \node[vertex] (1) at (0,0) {$ 1 $};
    \node[vertex] (2) at (2,0) {$ 2 $};
    \node[vertex] (3) at (3,1) {$ 3 $};
    \node[vertex] (4) at (4,2) {$ 4 $};
    \node[vertex] (5) at (5,3) {$ 5 $};
    \node[vertex] (6) at (6,4) {$ 6 $};
    \node[vertex] (7) at (7,5) {$ 7 $};

    \node[blacked] (int_1) at (1,1) {};
    \node[blacked] (int_2) at (2,2) {};
    \node[blacked] (int_3) at (3,3) {};
    \node[blacked] (int_4) at (4,4) {};
    \node[blacked] (int_5) at (5,5) {};
    \node[blacked] (int_6) at (6,6) {};

    \draw (1) to (int_1);
    \draw (2) to (int_1);
    \draw (3) to (int_2);
    \draw (4) to (int_3);
    \draw (5) to (int_4);
    \draw (6) to (int_5);
    \draw (7) to (int_6);
    \draw (int_1) to (int_2);
    \draw (int_2) to (int_3);
    \draw (int_3) to (int_4);
    \draw (int_4) to (int_5);
    \draw (int_5) to (int_6);
\end{scope}

\end{tikzpicture}}%
			\label{fig:construction-tree-basic}%
		}%
    \hfill
		\subfloat[][Parallel]{%
			\adjustbox{valign=b, scale=0.75}{\begin{tikzpicture}[
  vertex/.style = {shape=circle,draw, inner sep=0.04cm},
  blacked/.style = {shape=circle,draw,fill=black,text=black, inner sep=0.06cm},
  edge/.style = {-, dashed, tab20gray},
  greenvertex/.style = {draw=tab20green,shape=circle, text=black, line width=0.3mm, inner sep=0.04cm},
  bluevertex/.style = {draw=tab20blue,shape=circle, text=black, line width=0.3mm, inner sep=0.04cm},
  orangevertex/.style = {draw=tab20orange,shape=circle, text=black, line width=0.3mm, inner sep=0.04cm},
  greendummy/.style = {regular polygon, regular polygon sides=4, draw=tab20green,  line width=0.3mm, inner sep=-0.01cm},
  bluedummy/.style = {regular polygon,  regular polygon sides=4, draw=tab20blue,   line width=0.3mm, inner sep=-0.01cm},
  orangedummy/.style = {regular polygon,regular polygon sides=4, draw=tab20orange, line width=0.3mm, inner sep=-0.01cm},
]
\begin{scope}[scale=0.5]
    \node[vertex] (1) at (0,0) {$ 1 $};
    \node[vertex] (2) at (2,0) {$ 2 $};
    \node[vertex] (3) at (3,0) {$ 3 $};
    \node[vertex] (4) at (5,0) {$ 4 $};
    \node[vertex] (5) at (6,0) {$ 5 $};
    \node[vertex] (6) at (8,0) {$ 6 $};
    \node[vertex] (7) at (10,1) {$ 7 $};

    \node[blacked] (int_1) at (1,1) {};
    \node[blacked] (int_2) at (4,1) {};
    \node[blacked] (int_3) at (7,1) {};
    \node[blacked] (int_4) at (2.5,2) {};
    \node[blacked] (int_5) at (8.5,2) {};
    \node[blacked] (int_6) at (5.5,3) {};

    \draw (1) to (int_1);
    \draw (2) to (int_1);
    \draw (3) to (int_2);
    \draw (4) to (int_2);
    \draw (5) to (int_3);
    \draw (6) to (int_3);
    \draw (7) to (int_5);
    \draw (int_1) to (int_4);
    \draw (int_2) to (int_4);
    \draw (int_3) to (int_5);
    \draw (int_4) to (int_6);
    \draw (int_5) to (int_6);
\end{scope}
\end{tikzpicture}}%
			\label{fig:construction-tree-parallel}%
		}%
    \hfill
		\subfloat[][Incremental]{
			\adjustbox{valign=b, scale=0.75}{\begin{tikzpicture}[
  vertex/.style = {shape=circle,draw, inner sep=0.04cm},
  blacked/.style = {shape=circle,draw,fill=black,text=black, inner sep=0.06cm},
  edge/.style = {-, dashed, tab20gray},
  greenvertex/.style = {draw=tab20green,shape=circle, text=black, line width=0.3mm, inner sep=0.04cm},
  bluevertex/.style = {draw=tab20blue,shape=circle, text=black, line width=0.3mm, inner sep=0.04cm},
  orangevertex/.style = {draw=tab20orange,shape=circle, text=black, line width=0.3mm, inner sep=0.04cm},
  greendummy/.style = {regular polygon, regular polygon sides=4, draw=tab20green,  line width=0.3mm, inner sep=-0.01cm},
  bluedummy/.style = {regular polygon,  regular polygon sides=4, draw=tab20blue,   line width=0.3mm, inner sep=-0.01cm},
  orangedummy/.style = {regular polygon,regular polygon sides=4, draw=tab20orange, line width=0.3mm, inner sep=-0.01cm},
]
\begin{scope}[scale=0.5]
    \node[vertex] (1) at (0,0) {$ 1 $};
    \node[vertex] (2) at (1,0) {$ 2 $};
    \node[vertex] (3) at (2,0) {$ 3 $};
    \node[vertex] (4) at (3,0) {$ 4 $};
    \node[vertex] (5) at (3.5,1) {$ 5 $};
    \node[vertex] (6) at (4.5,2) {$ 6 $};
    \node[vertex] (7) at (5.5,3) {$ 7 $};

    \node[blacked] (int_1) at (1.5,1) {};
    \node[blacked] (int_2) at (2.5,2) {};
    \node[blacked] (int_3) at (3.5,3) {};
    \node[blacked] (int_4) at (4.5,4) {};

    \draw (1) to (int_1);
    \draw (2) to (int_1);
    \draw (3) to (int_1);
    \draw (4) to (int_1);
    \draw (5) to (int_2);
    \draw (6) to (int_3);
    \draw (7) to (int_4);
    \draw (int_1) to (int_2);
    \draw (int_2) to (int_3);
    \draw (int_3) to (int_4);
\end{scope}
\end{tikzpicture}}%
			\label{fig:construction-tree-incremental}%
		}%
  \caption[]{ %
    \textbf{Construction Trees.} 
    All three of our proposed construction variants,
    \emph{basic} \subref{fig:construction-tree-basic},
    \emph{parallel} \subref{fig:construction-tree-parallel},
    and \emph{incremental} \subref{fig:construction-tree-incremental}
    construction, can be described via construction trees.
    Shown are example trees for each variant and $ d=7 $ objects.
  }%
  \label{fig:construction-trees}%
\end{figure}
At the heart of both extensions lie 
\emph{leaf-labeled}, \emph{ordered} trees with \emph{label set} $ [d] $,
further referred to as \emph{construction trees}, see 
\cref{fig:construction-trees}.
As described in \cref{sec:construction}, each \texttt{vertex} of a 
construction tree can be associated with a partial solution
(\texttt{vertex.solution}) for the objects labeled by its descendant 
leaves.
While leaves are always associated with the trivial solution
$ \SV^p = \left\{ \{i\} \,\mid\, i\in V^p \right\} $
of the object they label $ V^p $, $ p\in[d] $, 
solutions are generally the result of (matching and)
merging the solutions of their children.
Our \emph{parallel construction} limits itself to \emph{binary} construction
trees (see \cref{fig:construction-trees}),
which is why it can be implemented with only a GM solver,
see \cref{alg:parallel-construction} for pseudocode.
The basic construction \cref{alg:construction} can be identified as the 
special case of binary construction trees with height $ d-1 $,
see \cref{fig:construction-trees}.
It allows for no parallelization as only matchings for solutions of vertices 
at the same level of a construction tree can be computed in parallel.
\emph{Arbitrary} construction trees are the idea behind 
\emph{incremental construction}, where solutions for vertices with 
more than two children require an \emph{MGM solver},
see \cref{fig:construction-trees}.
However, in \cref{sec:construction} and \cref{sec:experiments}, we limit 
ourselves to trees with only one such vertex at level 2,
see \cref{fig:construction-trees}.
A promising use case is to use a (potentially) better MGM solver for the first
few objects to warm-start the construction with a better partial solution.

\begin{algorithm}[H]
  \addtolength{\hsize}{1.5em}%
  \DontPrintSemicolon
  \SetKwFor{ForParallel}{for}{do in parallel}{endfor}
  \SetInd{0.5em}{0.5em}
  \KwIn{Binary Construction Tree \texttt{T}}
  \For{$ l\in\{\text{\texttt{T.height}, \text{\texttt{T.height}-1, \dots, 1}\}} $}{
    \ForParallel{\texttt{vertex} in \texttt{T.level}[$ l $]}{
      $ \SA \leftarrow\text{\texttt{vertex.left.solution}}$ \;
      $ \SB \leftarrow\text{\texttt{vertex.right.solution}} $ \;
      $ E \leftarrow $ Solve GM with costs $ C^{\SA,\SB} $ from \cref{eq:parallel-costs} \;
      \texttt{vertex.solution} $ \leftarrow \merge(\SA, \SB; E)$ \;
    }
  }
  \caption{Parallel solution construction.}
  \label{alg:parallel-construction}
\end{algorithm}

\section{GM Local Search: Details}
\label{suppl:qap-local-search}
\cref{alg:parallel-gm-ls} provides pseudocode for the \emph{parallel} GM local search.

\begin{algorithm}[H]
  \addtolength{\hsize}{1.5em}%
  \DontPrintSemicolon
  \SetKwFor{ForParallel}{for}{do in parallel}{endfor}
  \SetInd{0.5em}{0.5em}
  \KwIn{Solution $ \SQ\in\BQ $}
  \While{stopping criterion not met}{
    \ForParallel{$ p\in[d] $}{
      $ \SQ' \leftarrow \{Q\setminus V^p \,\mid\, Q\in\SQ\} $ \;
      $ E^p \leftarrow$ Solve GM with costs $ C^{p,\SQ'} $ from Eq. \eqref{eq:basic-costs} \;
      $ c^p \leftarrow C(\merge(V^p,\SQ'; E^p)) $ \;
    }
    \For{$ p\in[d] $ in ascending order of $ c^p $}{
      $ \SQ' \leftarrow \{Q\setminus V^p \,\mid\, Q\in\SQ\} $ \;
      $ \SQ'' \leftarrow \merge(V^p, \SQ'; E^p) $ \;
      \If{$ C ( \SQ'' ) < C ( \SQ ) $}{
        $ \SQ \leftarrow \SQ''$
      }
    }
  }
  \caption{Parallel GM local search.}
  \label{alg:parallel-gm-ls}
\end{algorithm}

\section{Swap Local Search: Details}
\label{suppl:swap-local-search}

In the following, we derive the swap energy change formula \cref{eq:single-swap-energy}. The derivation uses a decomposition of a double sum over a subset of the summation domain as specified in \cref{lemma:double-sum-decomposition}.
\begin{lemma} \label{lemma:double-sum-decomposition}
Let $ S $ be a finite set and $ f: S \times S \to \BR $ a real-valued function defined on the Cartesian product $ S \times S $.
For the double sum over $ S $, we have the following decomposition \wrt a subset $ T\subset S $
\begin{equation}
\sum_{a,b \in S} f(a,b) = \sum_{a,b \in T} f(a,b) + \sum_{\substack{a \in T\\ b \in S\setminus T}} \left( f(a,b) + f(b,a) \right) + \sum_{a, b \in S\setminus T} f(a,b)
  \label{eq:double-sum-decomposition}
\end{equation}
\end{lemma}
\begin{proof}
  Fix $ a\in S $, then we have by commutativity and associativity of the sum
\begin{equation}
  \sum_{b\in S} f(a,b) = \sum_{b\in T} f(a,b) + \sum_{b \in S\setminus T} f(a,b) \ .
  \label{eq:double-sum-decomposition-proof-1}
\end{equation}
  Applying this relation twice to the left-hand side of \cref{eq:double-sum-decomposition} and using again the commutativity and associativity of the sum after renaming summation indices yields the right-hand side.
\end{proof}

Now, consider a single swap (\cref{sec:swap-local-search}) performed over a fixed object $p \in [d]$ and two cliques $Q,R \in \SQ$, yielding a new solution $\SQ^\prime$.
First, we rewrite the MGM objective \st the cost terms involving the fixed object $ p\in [d]  $ and cliques $ Q, R \in \SQ $ can be singled out, see \cref{lemma:mgm-objective-for-swap}.
\begin{lemma} \label{lemma:mgm-objective-for-swap}
  Let $ \SQ\in \BQ$ be a solution. Fix an object $ p\in [d] $ and two cliques $ Q, R \in \SQ $.
  We can rewrite the MGM objective \eqref{eq:MGM} as follows 
  \begin{equation}
    C(\SQ) = \sum_{q\in [d] \setminus \left\{ p \right\}} C^{p,q}_{\left\{ Q, R \right\}} (\SQ) + \FC^{p}_{\lnot \left\{ Q,R  \right\}}(\SQ) + \FC^{\lnot p}(\SQ) \ ,
    \label{eq:mgm-objective-for-swap}
  \end{equation}
  where the term $ \FC^{p}_{\lnot \left\{ Q, R \right\}} (\SQ) $ does not depend on the cliques $ Q, R \in \SQ $, the term $ \FC^{\lnot p} (\SQ) $ does not depend on the object $ p \in [d] $, and all terms involving the fixed object or cliques are given by
  \begin{equation}
    C^{p,q}_{\left\{ Q, R \right\}} (\SQ) = \sum_{S, T \in \left\{ Q, R \right\}} C^{p, q} (S, T) + \sum_{\substack{S \in \{Q, R\} \\ T \in \SQ \setminus \{Q, R\}}} \left( C^{p,q}(S,T) + C^{p, q} (T, S) \right)
    \label{eq:mgm-objective-for-swap-1}
  \end{equation}
  with the symmetrized extended costs 
  \begin{equation}
    C^{p, q} (S, T) 
    \coloneqq 
    \left\{ 
      \begin{array}{ ll }
        C^{p,q}_{ S^p S^q, T^p T^q}&; p<q \wedge p, q\in D(S,T) \\
        C^{q,p}_{ S^q S^p, T^q T^p}&; q<p \wedge p, q\in D(S,T) \\
        0&; \text{ otherwise}\\
      \end{array}
    \right\} \ .
    \label{eq:mgm-objective-for-swap-2}
  \end{equation}
\end{lemma}
\begin{proof}
  The MGM objective \cref{eq:MGM} is given by 
  \begin{equation}
    C(\SQ) 
    \coloneqq \sum_{ S, T\in\SQ} \ \sum_{\substack{u,v\in D( S,T) \\ u < v}} C^{u,v}_{ S^u S^v, T^u T^v}  \ .
    \label{eq:mgm-objective-for-swap-proof-1}
  \end{equation}
  We will rewrite it as two double sums over the objects $ [d]  $ and cliques $ \SQ $ \st we can single out the fixed object $ p\in [d] $ and cliques $ Q, R\in \SQ $ by applying \cref{lemma:double-sum-decomposition} twice.
  First, rewrite the MGM objective as two double sums, \ie, 
  \begin{equation}
    C(\SQ) 
    = \sum_{S, T \in \SQ} \ \sum_{u, v \in [d]} \tilde{C}^{u, v} (S, T) 
    \label{eq:mgm-objective-for-swap-proof-2}
  \end{equation}
  where $ \tilde{C}^{u, v} (S, T) $ are extended costs defined by
  \begin{equation}
    \tilde{C}^{u, v} (S, T) \coloneqq
    \left\{ 
      \begin{array}{ ll }
        C^{u,v}_{ S^u S^v, T^u T^v}&; u<v \wedge u, v\in D(S,T) \\
        0&; \text{ else}\\
      \end{array}
    \right\} 
    \label{eq:mgm-objective-for-swap-proof-3}
  \end{equation}
We further define $ C^{u, v} (\SQ) \coloneqq \sum_{S, T\in \SQ} \tilde{C}^{u, v} \left( S, T \right) $ as all costs involving the two objects $ u,v \in [d] $.
Splitting the double sum over $ [d] $ \wrt the subset $ \left\{ p \right\} \subset [d] $ according to \cref{lemma:double-sum-decomposition} yields
\begin{equation}
  C(\SQ) 
= \underbrace{C^{p, p} (\SQ)}_{=0}
+ \sum_{v\in [d]\setminus \left\{ p \right\}} \left( C^{p, v} (\SQ) + C^{v, p} (\SQ) \right) 
+ \underbrace{\sum_{u, v \in [d] \setminus \left\{ p \right\}} C^{u, v} (\SQ)}_{\eqqcolon \FC^{\lnot p} (\SQ)}
  \label{eq:mgm-objective-for-swap-proof-4}
\end{equation}
where $ \FC^{\lnot p} (\SQ) $ only sums terms that do not depend on the object $ p \in [d] $.
Now, the symmetrized extended costs $  C^{u, v} (S, T) $ ($= \tilde{C}^{u,v} (S, T) + \tilde{C}^{v, u} (S, T) $) from \cref{eq:mgm-objective-for-swap-2} allow us to write the symmetric summand of \cref{eq:mgm-objective-for-swap-proof-4} as a double sum over all cliques $ \SQ $.
Again, splitting this double sum \wrt the fixed cliques $ \left\{ Q, R \right\} $ yields according to \cref{lemma:double-sum-decomposition}
\begin{equation}
\begin{aligned}
  C^{p, v} (\SQ) + C^{v, p} (\SQ) 
&= \sum_{S, T \in \SQ} C^{p, v} (S, T) \\
&= 
\underbrace{
\sum_{S, T \in \left\{ Q, R \right\}} C^{p, v} (S, T) 
+ \sum_{\substack{S \in \{Q, R\} \\ T \in \SQ \setminus \{Q, R\}}} \left( C^{p,v}(S,T) + C^{p, v} (T, S) \right)
}_{\eqqcolon C^{p, v}_{\left\{ Q, R \right\}} (\SQ)}
+ \underbrace{\sum_{\substack{S, T \in \SQ \setminus \{Q, R\} \\ \phantom{T \in \SQ \setminus \{Q, R\}}}} C^{p, v} (S, T)}_{\eqqcolon \FC^{p, v}_{\lnot \left\{ Q, R \right\}} (\SQ)}  \ .
  \label{eq:mgm-objective-for-swap-proof-6}
\end{aligned}
\end{equation}
Finally, substituting \cref{eq:mgm-objective-for-swap-proof-6} into \cref{eq:mgm-objective-for-swap-proof-4} yields
\begin{equation}
  C(\SQ) 
= \sum_{v\in [d]\setminus \left\{ p \right\}} C^{p, v}_{\left\{ Q, R \right\}} (\SQ)
+ \underbrace{\sum_{v\in [d]\setminus \left\{ p \right\}} \FC^{p, v}_{\lnot \left\{ Q, R \right\}} (\SQ)}_{\eqqcolon \FC^{p}_{\lnot \left\{ Q, R \right\}} (\SQ)}
+ \FC^{\lnot p} (\SQ) \ ,
  \label{eq:mgm-objective-for-swap-proof-7}
\end{equation}
where $ \FC^{p}_{\lnot \left\{ Q, R \right\}} (\SQ) $ only sums terms that do not depend on the cliques $ Q, R \in \SQ $.
\end{proof}
The decomposition in \cref{lemma:mgm-objective-for-swap} allows us to derive the objective change induced by a single swap as described in \cref{sec:swap-local-search}, see \cref{thm:swap-formula}.
\begin{theorem} \label{thm:swap-formula}
  Consider a single swap performed over a fixed object $p \in [d]$ and two cliques $Q,R \in \SQ$, yielding a new solution $\SQ^\prime$, and assume the definitions of \cref{lemma:mgm-objective-for-swap}. 
  The objective value change induced by this single swap decomposes into a sum of objective value changes between the fixed object $ p\in [d] $ and all other objects $ q\in [d] \setminus \left\{ p \right\} $, \ie,
  \begin{equation}
    C(\SQ^\prime) - C(\SQ) 
    = \sum_{ q \in [d] \setminus \{p\} } \delta C^{p,q}_{\textrm{swap}}(Q, R) 
    \quad \text{with} \quad \delta C^{p,q}_{\textrm{swap}}(Q, R) =  C^{p,q}_{\left\{ Q', R' \right\}}(\SQ') - C^{p,q}_{\left\{ Q, R \right\}}(\SQ)  \ .
    \label{eq:swap-formula}
  \end{equation}
\end{theorem}
\begin{proof}
  According to \cref{lemma:mgm-objective-for-swap}, we have
  \begin{equation}
    C(\SQ^\prime) - C(\SQ) 
    = \sum_{q\in [d] \setminus \left\{ p \right\}} \left( C^{p,q}_{\left\{ Q', R' \right\}} (\SQ') - C^{p,q}_{\left\{ Q, R \right\}} (\SQ) \right)
    + \underbrace{\FC^{p}_{\lnot \left\{ Q',R'  \right\}}(\SQ') - \FC^{p}_{\lnot \left\{ Q,R  \right\}}(\SQ)}_{= 0}
    + \underbrace{\FC^{\lnot p}(\SQ') - \FC^{\lnot p}(\SQ)}_{=0} \ ,
    \label{eq:swap-formula-1}
  \end{equation}
  where the last two terms vanish because the swap affects a cost term only if it involves both the fixed object $p \in [d]$ and cliques $Q,R \in \SQ$ (or $Q',R' \in \SQ'$), which does not hold for $ \FC^{p}_{\lnot \left\{ Q, R \right\}}(\SQ) $ (or $ \FC^{p}_{\lnot \left\{ Q', R' \right\}}(\SQ') $) and $ \FC^{\lnot p} (\SQ) $ (or $ \FC^{\lnot p} (\SQ') $).
\end{proof}

\section{Incomplete to Complete Transformation: Details}
\label{suppl:details_incomplete_to_complete_transformation}

\subsection{Proof of the Polynomial Transformation from Incomplete to Complete MGM}
\label{suppl:incomplete-transformation}
Before we prove the polynomial-time transformation from incomplete
to complete MGM, we formalize complete MGM analogously to incomplete MGM in 
\cref{sec:definition}, denoting complete counterparts of 
incomplete quantities with a tilde, \eg, $ V^p $ and $ \tV^p $.
Instances, hereafter referred to as problems, of incomplete (or complete)
problems are fully defined via their objects $ V^p $, $ p\in[d] $, 
(or $ \tV^p $, $ p\in[\td] $,) and costs $ C^{p,q} $, $ p,q\in[d] $, 
(or $ \tC^{p,q} $, $ p,q\in[\td] $,)
with the difference that objects of the complete instance must have 
equal cardinalities $ |\tV^p|=|\tV^q|\eqqcolon\tn \ \ \forall p,q\in[\td] $.
We denote the set of feasible vertex partitions or solutions to the complete 
problem with $ \tBQ $.
All cliques $ \tQ\in\tSQ $ of solutions $ \tSQ\in\tBQ $ to the complete 
problem must contain \emph{exactly one} element per object 
$ p\in[\td]  $, \ie, $ |\tQ\cap\tV^p|=1  $, which implies that 
cliques are \emph{maximal} $ D(\tQ)=[\td] $.
Furthermore, since all vertices must be matched, solutions $ \tSQ\in\tBQ $
contain exactly $ \tn $ cliques $ |\tSQ|=\tn $.
The objective for complete MGM is defined in analogy to Eq. \eqref{eq:MGM}
\begin{equation}
  \min_{\tSQ\in\tBQ}
  \left[ 
    \tC(\tSQ) \coloneqq
    \sum_{ \tQ, \tR\in\tSQ}
    \sum_{p,q\in [\td]}
    \tC^{p,q}_{ \tQ^p \tQ^q, \tR^p \tR^q}
  \right] \,.
  \label{eq:complete-MGM}
\end{equation}
\begin{proof}[Proof of \cref{thm:transformation}]
As described in Eq. \eqref{eq:complete-objects} and \eqref{eq:complete-costs},
given an incomplete problem with objects $ V^p $, $ p\in[d] $, and costs
$ C^{p,q} $, $ p,q\in[d] $, we construct a complete problem with 
objects 
\begin{equation}
  \widetilde{V}^{p} = V^p \cup \Delta^p, \quad  p \in[d] \,,
  \label{eq:suppl-complete-objects}
\end{equation}
and costs
\begin{equation}
  \widetilde{C}^{p,q}_{is,jt} =
	\begin{cases}
    C^{p,q}_{is,jt} &, \text{if $ i,s,j,t\in V  $ }\\
		0 &, \text{else}
	\end{cases} \,,
  \label{eq:suppl-complete-costs}
\end{equation}
where $ \Delta^p  $, $ p\in[d] $, are \emph{dummy vertices} \st 
$ |\Delta^p|=|V|-|V^p| $, see \cref{fig:transformation}.
\begin{figure*}
  \centering
  \scalebox{0.85}{\begin{tikzpicture}[
  vertex/.style = {shape=circle,draw, inner sep=0.04cm},
  blacked/.style = {shape=circle,draw,fill=black,text=black, inner sep=0.06cm},
  edge/.style = {-, dashed, tab20gray},
  greenvertex/.style = {draw=tab20green,shape=circle, text=black, line width=0.3mm, inner sep=0.04cm},
  bluevertex/.style = {draw=tab20blue,shape=circle, text=black, line width=0.3mm, inner sep=0.04cm},
  orangevertex/.style = {draw=tab20orange,shape=circle, text=black, line width=0.3mm, inner sep=0.04cm},
  greendummy/.style = {regular polygon, regular polygon sides=4, draw=tab20green,  line width=0.3mm, inner sep=-0.01cm},
  bluedummy/.style = {regular polygon,  regular polygon sides=4, draw=tab20blue,   line width=0.3mm, inner sep=-0.01cm},
  orangedummy/.style = {regular polygon,regular polygon sides=4, draw=tab20orange, line width=0.3mm, inner sep=-0.01cm},
]

      \node[bluevertex] (1_p) at  (-0.5,0.5) {$ 1 $};
      \node[bluevertex] (2_p) at  ( 0,0) {$ 2 $};

      \node[greenvertex] (1_r) at  (0  ,1.5) {$ 1 $};
      \node[greenvertex] (2_r) at  (0.5,1.5) {$ 2 $};
      \node[greenvertex] (3_r) at  (1  ,1.5) {$ 3 $};

      \node[orangevertex] (1_q) at  (1.5,0.5) {$ 1 $};
      \node[orangevertex] (2_q) at  (1,0) {$ 2 $};

      \node (V_p) at (-0.7, -0.2) {$V^{p}$};
      \node (V_r) at (0.5, 2) {$V^{r}$};
      \node (V_q) at (1.7, -0.2) {$V^{q}$};

      \node[draw, rectangle, minimum height=0.6cm, minimum width=1.5cm, rotate=0] (box_r) at (0.5,1.5) {};
      \node[draw, rectangle, minimum height=0.6cm, minimum width=1.25cm, rotate=-45]  (box_p) at (-0.25,0.25) {};
      \node[draw, rectangle, minimum height=0.6cm, minimum width=1.25cm, rotate=45] (box_q) at (1.25,0.25) {};

      \node (incomplete_prob_to_sol_top_anchor) at (0.5, -1.15) {};
      \node (incomplete_prob_to_sol_bottom_anchor) at (0.5, -2.35) {};
      \draw[->, line width=0.3mm] (incomplete_prob_to_sol_bottom_anchor) to (incomplete_prob_to_sol_top_anchor);
      \node (SQ) at (0.5, -2.75) {Solution $ \SQ\in\BQ $};

      \begin{scope}[xshift=6.7cm,yshift=-0.75cm]
        \node[greendummy]   (d_p_5)   at  ( 0,0)     {$ \delta_5 $};
        \node[greendummy]   (d_p_4)   at  (-0.5,0.5) {$ \delta_4 $};
        \node[greendummy]   (d_p_3) at  (-1,1)       {$ \delta_3 $};
        \node[orangedummy]  (d_p_2) at  (-1.5,1.5)   {$ \delta_2 $};
        \node[orangedummy]  (d_p_1) at  (-2,2)       {$ \delta_1 $};
        \node[bluevertex]  (2_p) at  (-2.5,2.5)      {$ 2 $};
        \node[bluevertex]  (1_p) at  (-3,3)          {$ 1 $};

        \node[greendummy]   (d_q_5) at  (1,0)     {$ \delta_5 $};
        \node[greendummy]   (d_q_4) at  (1.5,0.5) {$ \delta_4 $};
        \node[greendummy]   (d_q_3) at  (2,1)     {$ \delta_3 $};
        \node[bluedummy]    (d_q_2) at  (2.5,1.5) {$ \delta_2 $};
        \node[bluedummy]    (d_q_1) at  (3,2)     {$ \delta_1 $};
        \node[orangevertex]  (2_q) at  (3.5,2.5) {$ 2 $};
        \node[orangevertex]  (1_q) at  (4,3) {$ 1 $};

        \node[greenvertex]   (1_r)   at  (-1.75,4)   {$ 1 $};
        \node[greenvertex]   (2_r)   at  (-1,4) {$ 2 $};
        \node[greenvertex]   (3_r)   at  (-0.25,4)    {$ 3 $};
        \node[orangedummy]   (d_r_1) at  (0.5,4)  {$ \delta_1 $};
        \node[orangedummy]   (d_r_2) at  (1.25,4)    {$ \delta_2 $};
        \node[bluedummy]     (d_r_3) at  (2,4)  {$ \delta_3 $};
        \node[bluedummy]     (d_r_5) at  (2.75,4)    {$ \delta_4 $};

        \node (tV_p) at (-2, 1) {$\tV^{p}$};
        \node (tV_r) at (0.5, 4.6) {$\tV^{r}$};
        \node (tV_q) at (3, 1) {$\tV^{q}$};

        \node[draw, rectangle, minimum height=0.6cm, minimum width=5.25cm, rotate=0] (box_r) at (0.5,4) {};
        \node[draw, rectangle, minimum height=0.65cm, minimum width=4.9cm, rotate=-45]  (box_p) at (-1.5,1.5) {};
        \node[draw, rectangle, minimum height=0.65cm, minimum width=4.9cm, rotate=45] (box_q) at (2.5,1.5) {};

        \node (dummy_label) at (1.65, 3.2) {$ \Delta^r $};
        \draw [decorate, decoration={brace, amplitude=5pt, mirror}] (0.25,3.6) -- (3,3.6);
        \node (vertex_label) at (-1, 3.2) {$ V^r $};
        \draw [decorate, decoration={brace, amplitude=5pt, mirror}] (-2,3.6) -- (0,3.6);

        \node (complete_prob_to_sol_top_anchor) at (0.5, -0.4) {};
        \node (complete_prob_to_sol_bottom_anchor) at (0.5, -1.6) {};
        \draw[->, line width=0.3mm] (complete_prob_to_sol_top_anchor) to (complete_prob_to_sol_bottom_anchor);
        \node (tSQ) at (0.5, -2) {Solution $ \tSQ\in\tBQ $};
      \end{scope}

      \node (left_anchor_arrow_top) at (1.5, 1.1) {};
      \node (right_anchor_arrow_top) at (3.9, 1.1) {};
      \draw[->, line width=0.3mm] (left_anchor_arrow_top) to (right_anchor_arrow_top);
      \node (arrow_text_top) at (2.7, 1.4) {$ \tV^p=V^p\cup \Delta^p $};

      \node (left_anchor_arrow_bottom) at (1.7, -2.75) {};
      \node (right_anchor_arrow_bottom) at (5.7,-2.75) {};
      \draw[->, line width=0.3mm] (right_anchor_arrow_bottom) to (left_anchor_arrow_bottom);
      \node (arrow_text_bottom) at (3.75, -2.45) {$ \SQ=\{\tQ\cap V \, \mid \, \tQ\in\tSQ\} $};
\end{tikzpicture}}
  \caption{ 
    \textbf{Reducing incomplete to complete MGM.} 
    To reduce an incomplete problem with objects $ V^p, V^q, V^r $ 
    to a complete problem, one adds dummy vertices $ \Delta^p $ to each object 
    $ V^p $.
    Solutions to the incomplete problem can be recovered from solutions
    to the complete one by removing dummy vertices from cliques.
    However, this transformation is, without further reasoning,
    \emph{impractical} since \emph{each} object $ \tV^p $ of the complete
    problem is the same size as the \emph{entire} incomplete problem $ |V| $. 
  }
  \label{fig:transformation}
\end{figure*}

This problem is indeed \emph{complete} 
since  $ |\tV^p| = |V^p| + |V| - |V^p| = |V| \ \ \forall p\in[d] $.
It is also constructed in \emph{polynomial time} 
since $ \sum_{p\in[d]} |V|-|V^p|=(d-1)|V| $ dummy vertices and 
$ \sum_{p\in[d]}\sum_{q\in[d]\setminus\{p\}} (|V|^4-|V^p|^2|V^q|^2)$ costs 
are added to the incomplete problem.
What remains to prove is that the optimal solution of the 
incomplete problem can be determined in polynomial time
by solving the complete one. This follows from the fact that we can
\emph{translate} solutions without changing 
their objective between the incomplete and complete problem by 
\emph{adding} and \emph{removing} dummy vertices.
A solution $ \tSQ\in\tBQ $ for the complete problem is
translated to the solution 
\begin{equation}
   \SQ \coloneqq \{ \tQ\cap V \, \mid \, \tQ\in\tSQ \} \,,
  \label{eq:translation-complete-to-incomplete}
\end{equation}
for the incomplete problem.
We relate the cliques of both solutions
via the \emph{well-defined} function 
$ \tau:\tSQ \rightarrow \SQ, \tQ\mapsto\tQ\cap V $.
The solution to the incomplete problem is actually a solution,
\ie, $ \SQ\in\BQ $, because it only contains \emph{non-dummy vertices} 
by construction and \emph{exhaustiveness}, \emph{pairwise disjointness}, and 
\emph{feasibility} carry over from the solution to the complete problem.
Moreover, both objectives are equal
\renewcommand{\arraystretch}{1.5}
\begin{equation}
  \begin{array}{ lcl }
      \tC(\tSQ) 
      &
      =
      & 
      \sum_{\tQ,\tR\in\tSQ}
      \sum_{p,q\in[d]}
      \tC^{p,q}_{\tQ^p\tQ^q,\tR^p\tR^q}
      \\
      &
      \stackrel{\text{\cref{eq:suppl-complete-costs}}}{=}
      & 
      \sum_{\tQ,\tR\in\tSQ}
      \sum_{p,q\in D(\tau(\tQ),\tau(\tR))} 
      C^{p,q}_{\tau(\tQ)^p \tau(\tQ)^q, \tau(\tR)^p \tau(\tR)^q}
      \\
      &
      =
      & 
      \sum_{Q,R\in\SQ}
      \sum_{p,q\in D(Q,R)}
      C^{p,q}_{Q^pQ^q,R^pR^q}
      \\
      &
      =
      & 
      C(\SQ) \,.
  \end{array}
  \label{eq:translation-equal-costs}
\end{equation}
\begin{algorithm}[t]
\addtolength{\hsize}{1.5em}%
\DontPrintSemicolon
\KwIn{Incomplete Solution $ \SQ\in\BQ $}
  $
    \SQ
    \leftarrow
    \SQ\setminus\{\emptyset\}
  $\;
  \tcp{enumerate cliques}
  $
    \{Q_l\}_{l\in[\,|\SQ|\,]}
    \leftarrow
    \SQ
  $\;
  $
    \tSQ
    \leftarrow
    \{\tQ_l\}_{l\in[\, |V| \,]}
  $
  where 
  $ \tQ_l = Q_l $ if $ l\leq|\SQ| $ and $ \tQ_l=\emptyset $ else \;
  \For{$ p\in [d] $}{
    \tcp{enumerate dummy vertices}
    $
      \{\delta_l\}_{l\in[\,|V|-|V^p|\,]}
      \leftarrow
      \Delta^p
    $\;
    Reorder $ \tSQ $ \st $ |\tQ_l\cap V^p|=1 $ if $ l\leq|V^p| $\;
    \For{$ l\in  [\, |V|-|V^p| \,] $}{
      $ 
        \tQ_{l+|V^p|}
        \leftarrow
        \tQ_{l+|V^p|} \cup \delta_l
      $\;
    }
  }
\caption{Incomplete-to-complete solution translation.}
\label{alg:translation}
\end{algorithm}

For the other direction, \cref{alg:translation} can 
be used to translate an incomplete solution $ \SR\in\BQ $ to a complete 
solution $ \tSR\in\tBQ $, which is indeed a solution because 
each dummy vertex is added to \emph{exactly one} (possibly empty) clique out of 
$ |V| $ in total, 
which means cliques are \emph{maximal} and \emph{exhaustiveness},
\emph{pairwise disjointness}, and \emph{feasibility} 
carry over from the solution to the incomplete problem.
Note, however, that this solution is generally not unique because
the enumeration of dummy vertices in \cref{alg:translation}
is an unspecified degree of freedom.
Now, it holds by construction that 
$ \SR = \{\tR\cap V \, \mid \, \tR\in\tSR\} $, hence, $ C(\SR)=\tC(\tSR) $ by 
Eq. \eqref{eq:translation-equal-costs}.
Therefore, translating any optimal solution to the complete problem
via Eq.~\eqref{eq:translation-complete-to-incomplete} yields an optimal 
solution to the incomplete one, which concludes the proof because 
this translation can be computed in polynomial time. 
\end{proof}
\subsection{Minimality \wrt the Number of Dummy Vertices}
As discussed in \cref{sec:incomplete-transformation}, the number of dummy vertices 
required in \cref{eq:suppl-complete-objects} is $ d-1 $ times the size of the 
entire incomplete problem and yet \emph{minimal} because the incomplete 
solution not matching any vertex can only be obtained through a 
complete solution matching each non-dummy vertex to $(d-1)$ dummy vertices, see 
\cref{fig:minimality} for an example in the setting of \cref{fig:transformation}.
\begin{figure*}
  \centering
  \scalebox{0.85}{\begin{tikzpicture}[
  vertex/.style = {shape=circle,draw, inner sep=0.04cm},
  blacked/.style = {shape=circle,draw,fill=black,text=black, inner sep=0.06cm},
  edge/.style = {-, dashed, tab20gray},
  greenvertex/.style = {draw=tab20green,shape=circle, text=black, line width=0.3mm, inner sep=0.04cm},
  bluevertex/.style = {draw=tab20blue,shape=circle, text=black, line width=0.3mm, inner sep=0.04cm},
  orangevertex/.style = {draw=tab20orange,shape=circle, text=black, line width=0.3mm, inner sep=0.04cm},
  greendummy/.style = {regular polygon, regular polygon sides=4, draw=tab20green,  line width=0.3mm, inner sep=-0.01cm},
  bluedummy/.style = {regular polygon,  regular polygon sides=4, draw=tab20blue,   line width=0.3mm, inner sep=-0.01cm},
  orangedummy/.style = {regular polygon,regular polygon sides=4, draw=tab20orange, line width=0.3mm, inner sep=-0.01cm},
  blacksquare/.style = {draw=black, rectangle, minimum width=0.5cm, minimum height=0.5cm}
]

  \begin{scope}[scale=0.75]
        \node (V^p) at (0, 0)   {$ V^p $};
        \node (V^q) at (1, 0)   {$ V^q $};
        \node (V^r) at (2, 0)   {$ V^r $};
        \node (SQ) at (4.00, 0) {$ \SQ $};

        \node[bluevertex]  (1_p)   at (0, -1.0) {$ 1 $};
        \node[bluevertex]  (2_p)   at (0, -2.0) {$ 2 $};

        \node[orangevertex]  (1_q)   at (1, -3.0) {$ 1 $};
        \node[orangevertex]  (2_q)   at (1, -4.0) {$ 2 $};

        \node[greenvertex]  (1_r)   at (2, -5.0) {$ 1 $};
        \node[greenvertex]  (2_r)   at (2, -6.0) {$ 2 $};
        \node[greenvertex]  (3_r)   at (2, -7.0) {$ 3 $};

        \node[blacksquare]  (Q_1)   at (4.00, -1.0) {$ Q_1 $};
        \node[blacksquare]  (Q_2)   at (4.00, -2.0) {$ Q_2 $};
        \node[blacksquare]  (Q_3)   at (4.00, -3.0) {$ Q_3 $};
        \node[blacksquare]  (Q_4)   at (4.00, -4.0) {$ Q_4 $};
        \node[blacksquare]  (Q_5)   at (4.00, -5.0) {$ Q_5 $};
        \node[blacksquare]  (Q_6)   at (4.00, -6.0) {$ Q_6 $};
        \node[blacksquare]  (Q_7)   at (4.00, -7.0) {$ Q_7 $};

        \draw[dashed, tab20gray] (0,-1) ellipse (0.5cm and 0.38cm);
        \draw[dashed, tab20gray] (0,-2) ellipse (0.5cm and 0.38cm);
        \draw[dashed, tab20gray] (1,-3) ellipse (0.5cm and 0.38cm);
        \draw[dashed, tab20gray] (1,-4) ellipse (0.5cm and 0.38cm);
        \draw[dashed, tab20gray] (2,-5) ellipse (0.5cm and 0.38cm);
        \draw[dashed, tab20gray] (2,-6) ellipse (0.5cm and 0.38cm);
        \draw[dashed, tab20gray] (2,-7) ellipse (0.5cm and 0.38cm);

        \node[draw, rectangle, minimum width=0.65cm, minimum height=5.25cm, rotate=0] (box_p) at (0,-4) {};
        \node[draw, rectangle, minimum width=0.65cm, minimum height=5.25cm, rotate=0] (box_p) at (1,-4) {};
        \node[draw, rectangle, minimum width=0.65cm, minimum height=5.25cm, rotate=0] (box_p) at (2,-4) {};
        \node[draw, rectangle, minimum width=0.85cm, minimum height=5.25cm, rotate=0, dashed] (box_p) at (4.00,-4) {};

        \draw[dashed, shorten <= 5pt, tab20gray] (1_p) to (Q_1);
        \draw[dashed, shorten <= 5pt, tab20gray] (2_p) to (Q_2);
        \draw[dashed, shorten <= 5pt, tab20gray] (1_q) to (Q_3);
        \draw[dashed, shorten <= 5pt, tab20gray] (2_q) to (Q_4);
        \draw[dashed, shorten <= 5pt, tab20gray] (1_r) to (Q_5);
        \draw[dashed, shorten <= 5pt, tab20gray] (2_r) to (Q_6);
        \draw[dashed, shorten <= 5pt, tab20gray] (3_r) to (Q_7);

        \node (left_anchor_arrow_bottom)  at (4.6, -4.00) {};
        \node (right_anchor_arrow_bottom) at (11, -4.00) {};
        \draw[->, line width=0.3mm] (right_anchor_arrow_bottom) to (left_anchor_arrow_bottom);
        \node (arrow_text_bottom) at (7.8, -3.45) {$ \SQ=\{\tQ\cap V \, \mid \, \tQ\in\tSQ\} $};
  \end{scope}

  \begin{scope}[scale=0.75, xshift=12cm]
        \node (tV^p) at (0, 0) {$ \tV^p $};
        \node (tV^q) at (1, 0) {$ \tV^q $};
        \node (tV^r) at (2, 0) {$ \tV^r $};
        \node (tSQ) at (4.00, 0) {$ \tSQ $};

        \node[bluevertex]  (1_p)   at (0, -1.0) {$ 1 $};
        \node[bluevertex]  (2_p)   at (0, -2.0) {$ 2 $};
        \node[orangedummy] (d_p_1) at (0, -3.0) {$ \delta_1 $};
        \node[orangedummy] (d_p_2) at (0, -4.0) {$ \delta_2 $};
        \node[greendummy]  (d_p_3) at (0, -5.0) {$ \delta_3 $};
        \node[greendummy]  (d_p_4) at (0, -6.0) {$ \delta_4 $};
        \node[greendummy]  (d_p_5) at (0, -7.0) {$ \delta_5 $};

        \node[bluedummy]     (d_q_1) at (1, -1.0) {$ \delta_1 $};
        \node[bluedummy]     (d_q_2) at (1, -2.0) {$ \delta_2 $};
        \node[orangevertex]  (1_q)   at (1, -3.0) {$ 1 $};
        \node[orangevertex]  (2_q)   at (1, -4.0) {$ 2 $};
        \node[greendummy]    (d_q_3) at (1, -5.0) {$ \delta_3 $};
        \node[greendummy]    (d_q_4) at (1, -6.0) {$ \delta_4 $};
        \node[greendummy]    (d_q_5) at (1, -7.0) {$ \delta_5 $};

        \node[bluedummy]    (d_r_3) at (2, -1.0) {$ \delta_3 $};
        \node[bluedummy]    (d_r_4) at (2, -2.0) {$ \delta_4 $};
        \node[orangedummy]  (d_r_1) at (2, -3.0) {$ \delta_1 $};
        \node[orangedummy]  (d_r_2) at (2, -4.0) {$ \delta_2 $};
        \node[greenvertex]  (1_r)   at (2, -5.0) {$ 1 $};
        \node[greenvertex]  (2_r)   at (2, -6.0) {$ 2 $};
        \node[greenvertex]  (3_r)   at (2, -7.0) {$ 3 $};

        \node[blacksquare]  (tQ_1)   at (4.00, -1.0) {$ \tQ_1 $};
        \node[blacksquare]  (tQ_2)   at (4.00, -2.0) {$ \tQ_2 $};
        \node[blacksquare]  (tQ_3)   at (4.00, -3.0) {$ \tQ_3 $};
        \node[blacksquare]  (tQ_4)   at (4.00, -4.0) {$ \tQ_4 $};
        \node[blacksquare]  (tQ_5)   at (4.00, -5.0) {$ \tQ_5 $};
        \node[blacksquare]  (tQ_6)   at (4.00, -6.0) {$ \tQ_6 $};
        \node[blacksquare]  (tQ_7)   at (4.00, -7.0) {$ \tQ_7 $};

        \draw[dashed, tab20gray] (1,-1) ellipse (2cm and 0.38cm);
        \draw[dashed, tab20gray] (1,-2) ellipse (2cm and 0.38cm);
        \draw[dashed, tab20gray] (1,-3) ellipse (2cm and 0.38cm);
        \draw[dashed, tab20gray] (1,-4) ellipse (2cm and 0.38cm);
        \draw[dashed, tab20gray] (1,-5) ellipse (2cm and 0.38cm);
        \draw[dashed, tab20gray] (1,-6) ellipse (2cm and 0.38cm);
        \draw[dashed, tab20gray] (1,-7) ellipse (2cm and 0.38cm);

        \node[draw, rectangle, minimum width=0.65cm, minimum height=5.25cm, rotate=0] (box_p) at (0,-4) {};
        \node[draw, rectangle, minimum width=0.65cm, minimum height=5.25cm, rotate=0] (box_p) at (1,-4) {};
        \node[draw, rectangle, minimum width=0.65cm, minimum height=5.25cm, rotate=0] (box_p) at (2,-4) {};
        \node[draw, rectangle, minimum width=0.85cm, minimum height=5.25cm, rotate=0, dashed] (box_p) at (4.00,-4) {};

        \draw[dashed, shorten <= 16pt, tab20gray] (d_r_3) to (tQ_1);
        \draw[dashed, shorten <= 16pt, tab20gray] (d_r_4) to (tQ_2);
        \draw[dashed, shorten <= 16pt, tab20gray] (d_r_1) to (tQ_3);
        \draw[dashed, shorten <= 16pt, tab20gray] (d_r_2) to (tQ_4);
        \draw[dashed, shorten <= 16pt, tab20gray] (1_r)   to (tQ_5);
        \draw[dashed, shorten <= 16pt, tab20gray] (2_r)   to (tQ_6);
        \draw[dashed, shorten <= 16pt, tab20gray] (3_r)   to (tQ_7);
  \end{scope}
\end{tikzpicture}}
  \caption{ 
    \textbf{Requiring all dummy vertices.} 
    In the transformation from \cref{thm:transformation}, to obtain the incomplete
    solution that does not match any vertex (left), 
    all dummy vertices are required by the complete solution (right).
    This example uses the setting of \cref{fig:transformation}.
    Solutions are illustrated identically to \cref{fig:gm-ls}.
  }
  \label{fig:minimality}
\end{figure*}

\FloatBarrier
\clearpage
\section{Hamming Distance as Synchronization Objective}
\label{suppl:hamm-dist}
Let $E \subset \overbar{E}$ be a possibly inconsistent multi-matching. 
Some synchronization works, \eg~\cite{bernard2019synchronisation,pachauri2013solving}, formulate the projection onto the set of cycle-consistent multi-matchings not by maximizing the overlap, or equivalently the intersection, as in Eq.~\eqref{eq:sync}, but instead by minimizing the \emph{Hamming distance}, or equivalently the \emph{symmetric difference}, \ie,
\begin{equation}\label{eq:suppl-sync}
  E^*\in \argmin_{ F } | E\Delta F| \qquad
  \text{s.t.}\  F\subset\overbar{ E} \text{ cycle-consistent } \,.
\end{equation}
As in the overlap formulation from Eq. \eqref{eq:sync-dense-costs}, the Hamming distance formulation from Eq. \eqref{eq:suppl-sync} can be written as a Multi-LAP over all objects $p\in[d]$ with costs
\begin{equation}
  \widetilde{M}^{p,q}_{is,is}
  =
  \left\{
  \begin{array}{ rl }
    -1, & \text{ if } is\in E\,, \\
    +1,  & \text{ otherwise}.       \\
  \end{array}
  \right.
  \label{eq:suppl-sync-dense-costs}
\end{equation}
Indeed, for any feasible cycle-consistent multi-matching $F\subset \overbar{E}$, the corresponding Multi-LAP cost is given by
\begin{equation}
    C_\text{MLAP}(F) 
    = \sum_{e \in F} - I(e \in E) + I(e \notin E) 
    = -|F \cap E| + |F \setminus E|
    ,
\end{equation}
where $I$ is the indicator function. 
Adding the constant $ |E| $, which is independent of $ F $, yields
\begin{equation}
    |E| + C_\text{MLAP}(F) = |E| - |E \cap F| + |F \setminus E| = |E \setminus F| + |F \setminus E| = |E \Delta F| \,.
\end{equation}
Thus, minimizing the Multi-LAP with costs \cref{eq:suppl-sync-dense-costs} is equivalent to minimizing the symmetric difference synchronization objective \cref{eq:suppl-sync}.
In the complete case, the overlap \eqref{eq:sync} and Hamming distance \eqref{eq:suppl-sync} formulations coincide, since on the feasible set their objectives differ only by an additive constant and a negative multiplicative factor, \ie,
\begin{equation}
  |E \Delta F| = |E \setminus F| + |F \setminus E| = |E| - |E \cap F| + |F| - |E \cap F| = d(d-1) n_{\cdot} - 2|E\cap F| = \mathrm{const.} - 2|E\cap F| \, ,
  \label{eq:suppl-overlap-hamming-distance-agree}
\end{equation}
where $ n_{\cdot} = |V^p| $ is the number of keypoints in each object $ p\in[d] $, and $ |E| = |F| = \frac{d(d-1)}{2} n_{\cdot} $, since a complete multi-matching restricted to any of the $ d(d-1)/2 $ pairs of objects is a perfect matching with $ n_{\cdot} $ edges. 

\section{Experiment Setup: Details}\label{suppl:experiment-setup}

All experiments in \cref{sec:comparison} were run with a single core on an AMD Milan EPYC 7513. Parallel methods in \cref{sec:ablation-study} were granted 10 processing cores.
As mentioned in \cref{sec:experiments}, we solve GM subproblems of \emph{GREEDA} using the state-of-the-art GM solver~\cite{hutschenreiter_fusionmoves_2021,fusionmovesprojectpage}.
We run it with the following otherwise default parameters: Batch size $ 10 $, $ 10 $ greedy generations per batch, number of batches $ 10 $. 
When calling the GM solver~\cite{hutschenreiter_fusionmoves_2021,fusionmovesprojectpage} directly from our \emph{pylibmgm} Python package, the stopping criterion $ k $ should be set to a high number, \eg, $ k = 100000 $, to yield equivalent results.
Additionally, we start with a \emph{fixed} randomization seed $ 42 $, which turns the otherwise randomized solver~\cite{hutschenreiter_fusionmoves_2021,fusionmovesprojectpage} into a deterministic one. This is needed to separate the randomization effects of our construction heuristic (see \cref{sec:construction}) from those of the underlying GM solver. It follows that we can attribute any randomization effects in our experiments to \emph{GREEDA} itself.

For all synchronization methods, the $ d(d-1)/2 $ GM subproblems arising in pre-processing are solved using the same solver~\cite{hutschenreiter_fusionmoves_2021,fusionmovesprojectpage} with identical settings. 
To ensure a fair comparison, each problem instance is solved only once, and the resulting solution is used as the common initialization for all algorithms.\\

For \term{MP-T}, we ran our experiments using the default hyperparameters, which can be found in the respective GitHub repository \url{https://github.com/LPMP/LPMP/tree/master/src/multigraph_matching/eval}.
At the time of writing, these were the following command-line options:

\begin{center}
\begin{tabular}{@{}ll@{}}
  \toprule
  Option & Value \\
  \midrule
  \texttt{-{}-standardReparametrization}       & \texttt{uniform} \\
  \texttt{-{}-roundingReparametrization}       & \texttt{uniform:0.5} \\
  \texttt{-{}-tightenReparametrization}        & \texttt{uniform:0.5} \\
  \texttt{-{}-tightenInterval}                 & \texttt{50} \\
  \texttt{-{}-tightenIteration}                & \texttt{300} \\
  \texttt{-{}-tightenConstraintsPercentage}    & \texttt{0.05} \\
  \texttt{-{}-mcfReparametrization}            & --- \\
  \texttt{-{}-primalComputationStart}          & \texttt{1} \\
  \texttt{-{}-primalComputationInterval}       & \texttt{50} \\
  \texttt{-v}                                  & \texttt{2} \\
  \texttt{-{}-lowerBoundComputationInterval}   & \texttt{10} \\
  \texttt{-{}-tighten}                         & --- \\
  \texttt{-{}-multigraphMatchingRoundingMethod} & \texttt{gaec\_KL} \\
  \bottomrule
\end{tabular}
\end{center}

We used these settings for both \term{MP-T} and \term{MP-T (Sync)} on \emph{hotel}, \emph{house}, and \emph{synthetic}.
For \emph{worms}, the repository configuration differs only in \texttt{-{}-standardReparametrization anisotropic}.
Following correspondence with the authors, we additionally reran all experiments with \texttt{-{}-multigraphMatchingRoundingMethod MCF\_PS}.
This substantially improved \term{MP-T} on \emph{worms}, but substantially worsened it on \emph{hotel}, \emph{house}, and \emph{synthetic}, as well as \term{MP-T (Sync)} on all datasets.
We therefore report \term{MP-T} on \emph{worms} with \texttt{MCF\_PS} rounding, and use the repository defaults in all other cases.

\FloatBarrier
\section{Optimizing Accuracy}\label{suppl:accuracy}
As motivated in \cref{sec:experiments}, we compare algorithms solely by their runtime 
and their achieved objective value.
Many works, \eg, \cite{yan2016constrained,yan2013joint,yan2015consistency,swoboda2019convex,bernard2018ds}, 
instead base their comparative analysis on the \emph{accuracy} achieved on the downstream matching task.
However, modeling a downstream task as the MGM problem \cref{eq:MGM} is an abstraction: algorithms see the costs $ C^{p,q} $, but not the original image data.
Accuracy hence measures two errors at once, the \emph{modeling error} between the optima of \cref{eq:MGM} and the ground truth, and the \emph{optimization error} of the algorithm.
Only the latter is a property of the algorithm.


\begin{figure*}[!t]
  \centering
  \resizebox{\textwidth}{!}{\input{floats/accuracy_vs_objective/fig-accuracy-legend-grid-4.pgf}}%
  \\[0.5\baselineskip]
  \input{floats/accuracy_vs_objective/fig-accuracy-grid-4.pgf}%
  \caption{
    \textbf{Accuracy versus objective value on the 4-object datasets.}
    The objective \eqref{eq:MGM} is lower-is-better, the accuracy higher-is-better.
    Results are averaged over all instances of the respective dataset.
  }
  \label{fig:accuracy-vs-objective-4}
\end{figure*}

\begin{figure*}[!t]
  \centering
  \resizebox{\textwidth}{!}{\input{floats/accuracy_vs_objective/fig-accuracy-legend-grid-8.pgf}}%
  \\[0.5\baselineskip]
  \input{floats/accuracy_vs_objective/fig-accuracy-grid-8.pgf}%
  \caption{
    \textbf{Accuracy versus objective value on the 8-object datasets.}
    Results are averaged over all instances of the respective dataset.
  }
  \label{fig:accuracy-vs-objective-8}
\end{figure*}

\begin{figure*}[!t]
  \centering
  \resizebox{\textwidth}{!}{\input{floats/accuracy_vs_objective/fig-accuracy-legend-grid-12.pgf}}%
  \\[0.5\baselineskip]
  \input{floats/accuracy_vs_objective/fig-accuracy-grid-12.pgf}%
  \caption{
    \textbf{Accuracy versus objective value on the 12-object datasets.}
    Results are averaged over all instances of the respective dataset.
  }
  \label{fig:accuracy-vs-objective-12}
\end{figure*}

\begin{figure*}[!t]
  \centering
  \resizebox{\textwidth}{!}{\input{floats/accuracy_vs_objective/fig-accuracy-legend-grid-16.pgf}}%
  \\[0.5\baselineskip]
  \input{floats/accuracy_vs_objective/fig-accuracy-grid-16.pgf}%
  \caption{
    \textbf{Accuracy versus objective value on the 16-object datasets.}
    Results are averaged over all instances of the respective dataset.
  }
  \label{fig:accuracy-vs-objective-16}
\end{figure*}

\begin{figure*}[!t]
  \centering
  \resizebox{\textwidth}{!}{\input{floats/accuracy_vs_objective/fig-accuracy-legend-worms-3-10.pgf}}%
  \\[0.5\baselineskip]
  \input{floats/accuracy_vs_objective/fig-accuracy-worms-3-10.pgf}%
  \caption{
    \textbf{Accuracy versus objective value on the \emph{worms} datasets with 3 to 10 objects.}
    Results are averaged over all instances of the respective problem size.
    The vertical dotted line marks the accuracy of the ground truth, \ie the maximum attainable one.
    It falls short of 100\% because the heuristic that sparsifies the \emph{worms} instances, and thereby keeps them tractable, forbids some ground-truth matchings, which no feasible solution can recover.
    Accuracy and objective values are evaluated on the annotated keypoints only, see \cref{suppl:accuracy}.
  }
  \label{fig:accuracy-vs-objective-worms-3-10}
\end{figure*}

\begin{figure*}[!t]
  \centering
  \resizebox{\textwidth}{!}{\input{floats/accuracy_vs_objective/fig-accuracy-legend-worms-20-30.pgf}}%
  \\[0.5\baselineskip]
  \input{floats/accuracy_vs_objective/fig-accuracy-worms-20-30.pgf}%
  \caption{
    \textbf{Accuracy versus objective value on the \emph{worms} datasets with 20 and 30 objects.}
    Results are averaged over all instances of the respective problem size.
    As in \cref{fig:accuracy-vs-objective-worms-3-10}, the vertical dotted line marks the maximum attainable accuracy.
    Accuracy and objective values are evaluated on the annotated keypoints only, see \cref{suppl:accuracy}.
  }
  \label{fig:accuracy-vs-objective-worms-20-30}
\end{figure*}


\begin{figure*}[!t]
  \centering
  \includegraphics[width=\textwidth]{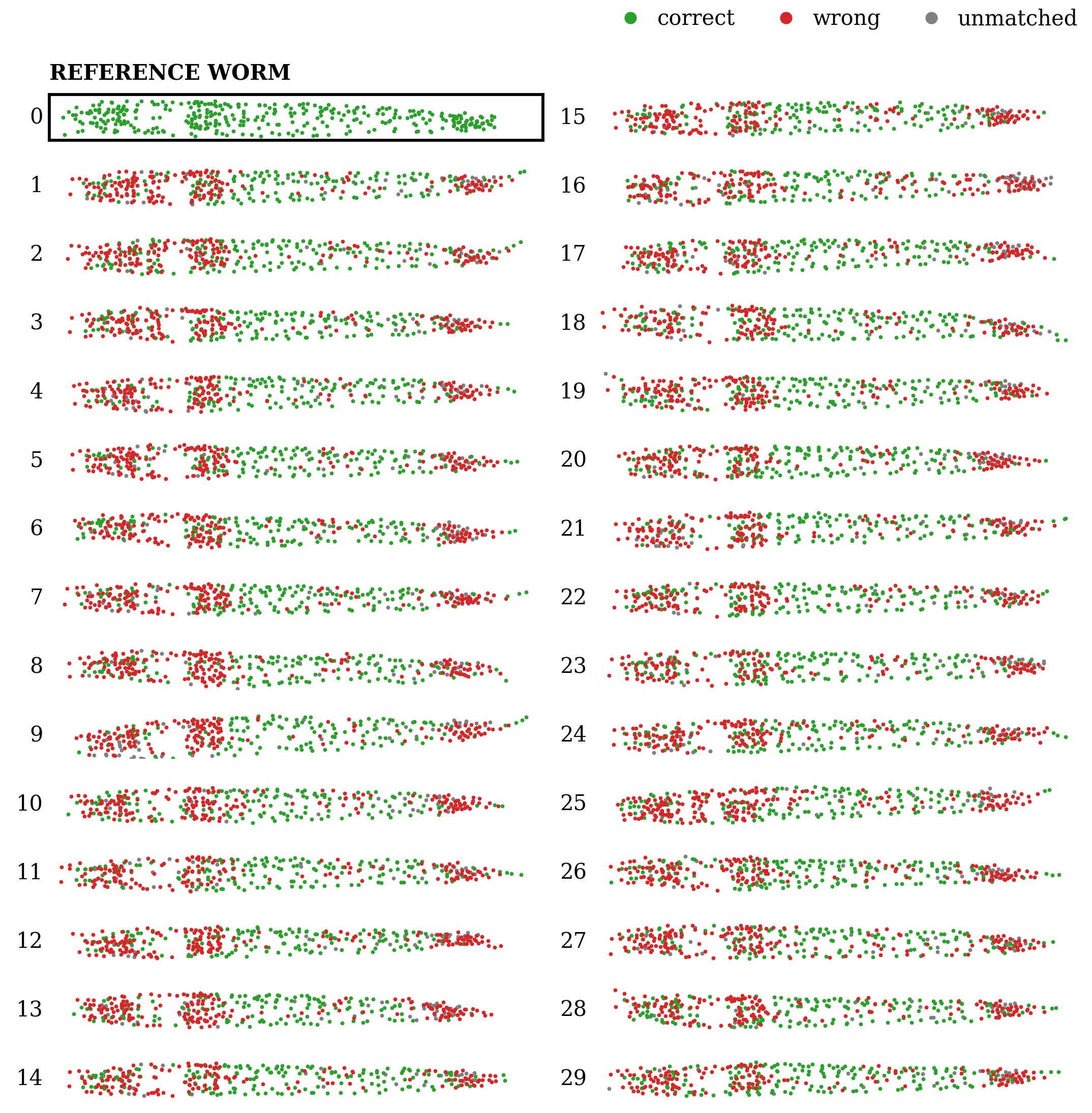}
  \caption{
    \textbf{Per-keypoint correctness of the best \emph{GREEDA} solution on \term{worms-30}.}
    Every keypoint is colored by whether it is matched to the correct keypoint of the boxed reference object.
    Note that this shows only the matchings to the reference object, and thus only a part of the full multi-matching, see \cref{suppl:accuracy}.
    The reference was fixed a priori and not chosen to favor our result.
  }
  \label{fig:worms30-correctness-greeda}
\end{figure*}


\begin{figure*}[!t]
  \centering
  \includegraphics[width=\textwidth]{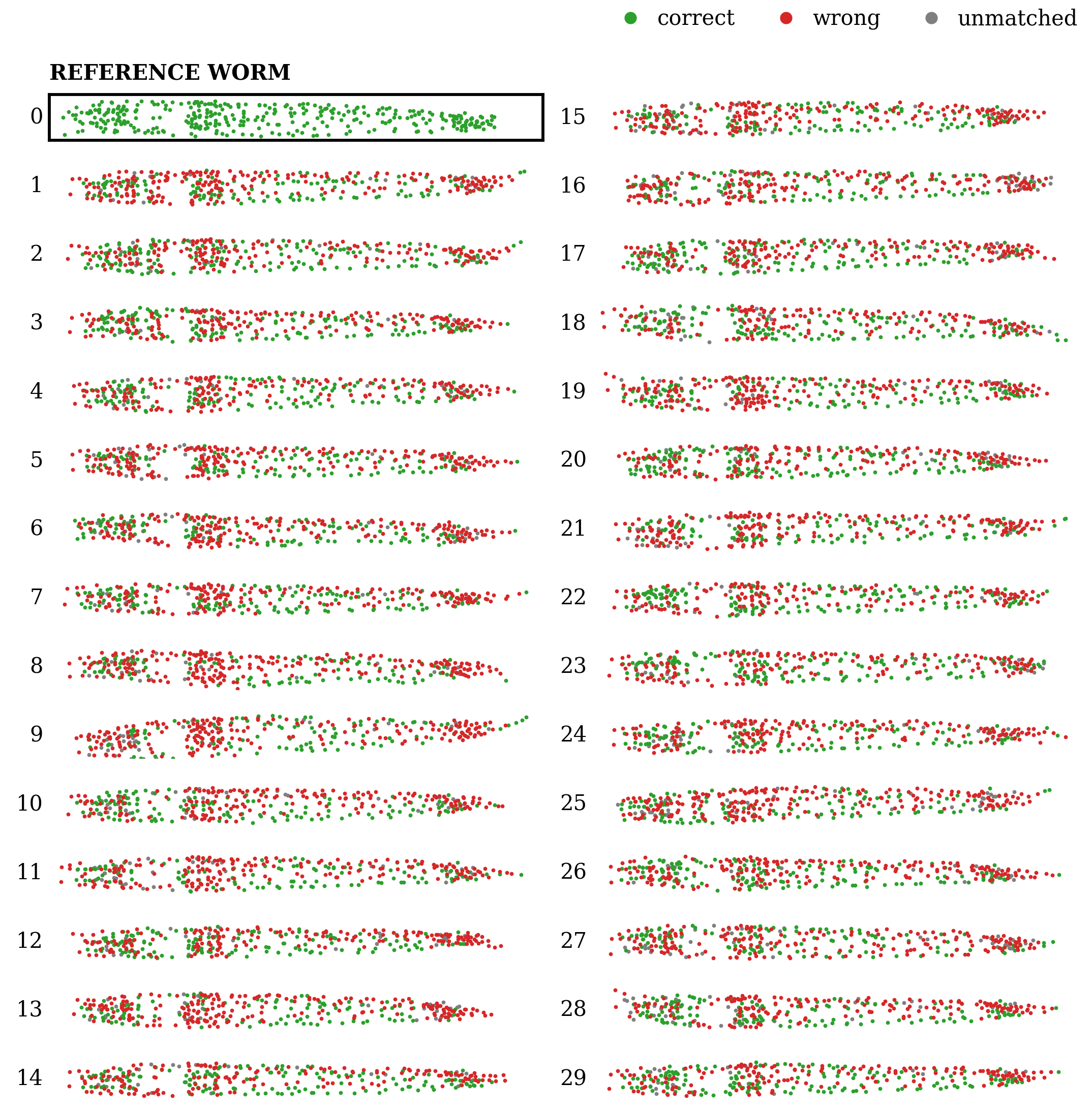}
  \caption{
    \textbf{Per-keypoint correctness of the \term{MP-T} solution on \term{worms-30}.}
    Every keypoint is colored by whether it is matched to the correct keypoint of the boxed reference object.
    Note that this shows only the matchings to the reference object, and thus only a part of the full multi-matching, see \cref{suppl:accuracy}.
    The reference was fixed a priori and not chosen to favor our result.
  }
  \label{fig:worms30-correctness-mpt}
\end{figure*}

Although objective value and accuracy correlate, the correspondence between them is not order-preserving, see \cref{fig:accuracy-vs-objective-4,fig:accuracy-vs-objective-8,fig:accuracy-vs-objective-12,fig:accuracy-vs-objective-16,fig:accuracy-vs-objective-worms-3-10,fig:accuracy-vs-objective-worms-20-30}, so a better objective value need not yield a better accuracy.
Comparing accuracies therefore favors algorithms whose optimization error by chance happens to compensate for the modeling error of the given dataset.

On the \emph{worms} datasets, a ground truth is available only for a subset of the keypoints.
\cref{fig:accuracy-vs-objective-worms-3-10,fig:accuracy-vs-objective-worms-20-30} therefore evaluate accuracy on the annotated keypoints only.
The solvers, however, optimize the full problem, which comprises all keypoints.
To compare their objective values to that of the ground truth, we reevaluate their solutions on the problem restricted to the annotated keypoints.
This comparison is only approximate, however. 
The solvers did not optimize the restricted problem, and their objective values would differ had they done so.

\cref{fig:worms30-correctness-greeda,fig:worms30-correctness-mpt} visualize the keypoint accuracy of \emph{GREEDA (best)} and \term{MP-T} on \term{worms-30}.
Visualizing a complete multi-matching is impractical at this scale, as it comprises $ d(d-1)/2 = 435 $ pairwise matchings.
The figures therefore show only the matchings to a single reference object.
The correspondences among the remaining objects stay invisible in this view, although they contribute to the overall accuracy.
A visually better result need therefore not translate into a better overall solution.
Substantial differences, however, can in general be assumed to be systematic and should be visible for any choice of reference object.

\FloatBarrier

\section{Ablation Study: Influence of Randomization}
\label{suppl:randomization}

\begin{figure*}[!t]
	\resizebox{\textwidth}{!}{\input{floats/pgf/cost-distribution.pgf}}%
	\caption{
		\textbf{Box plots of cost distributions} from 100 runs averaged over 10 \term{worms-10} instances after our \emph{sequential}, \emph{parallel}, or \emph{incremental} construction (see \cref{sec:construction}), and after additionally applying our \emph{sequential} or \emph{parallel} GM local search ($+$~LS) to each (see \cref{sec:qap-local-search}).
	}
	\label{fig:results-cost-distribution}%
\end{figure*}

To assess the influence of randomization inherent to GREEDA, we run each variant of our algorithm 100 times, using different object orderings, on each of the first 10 \term{worms-10} instances. 
\cref{fig:results-cost-distribution} shows box plots of the resulting cost distributions, averaged over all instances. 
Although the cost distributions significantly differ for different variants of the construction method, this difference vanishes after the local search. 
Hence, the extra runtime of \emph{incremental} construction (see \cref{fig:results-ablation-ls}) is hard to justify. 
In contrast, \emph{parallel} construction's improved scaling can be fully utilized because the local search compensates for the cost difference to the \emph{sequential} construction.

Drawing $k$ of the 100 runs uniformly without replacement, 
the expected cost of the best-of-$k$ run is worse than the best-of-100 run by as little as
0.59\,\% ($k{=}1$) and 0.22\,\% ($k{=}10$) after sequential construction, and by 
0.27\,\% and 0.10\,\% after sequential GM-LS.
Restarts thus add little, especially once local search is applied. 
We therefore do not recommend restarts in general.
\FloatBarrier

\section{Detailed Averaged Synchronization Results}
\label{suppl:detailed-sync-results}
We provide here detailed results of the synchronization algorithms considered in \cref{sec:synchronization-exp} for every problem instance.
We report the Multi-LAP (\term{M-LAP}), \cref{eq:sync-dense-costs}, and \emph{MGM}, \cref{eq:MGM}, objectives, as well as the number of forbidden matchings (\emph{\#forb.}) in the final solution and the runtime of each algorithm.
The \emph{pre-processing time}, the time required to calculate the cycle-inconsistent solution by solving the $ d(d-1)/2 $ independent GM subproblems, is given in the caption of each table.
On the \emph{worms} and \term{worms-large} datasets, most synchronization algorithms were evaluated only on \term{worms-10}, as they in general produce infeasible solutions with infinite cost.

\begin{table*}
\centering
\caption{\textbf{Synchronization Results for synthetic complete-4}. Averaged over all instances. Pre-processing time: 0.04 s.}
\begin{tabular}{lCCCC}
\toprule
Method & M-LAP obj. \textdownarrow & \#Forb.\textdownarrow & MGM obj. \textdownarrow & Runtime Sync \textdownarrow \\
\midrule
GREEDA (Sync)(best) & \bfseries -60.00 & \bfseries 0.00 & \bfseries -552.83 & 0.007\\
GREEDA (Sync)(worst) & \bfseries -60.00 & \bfseries 0.00 & \bfseries -552.83 & 0.007\\
MP-T (Sync) & \bfseries -60.00 & \bfseries 0.00 & \bfseries -552.83 & \bfseries 0.003\\
MST & \bfseries -60.00 & \bfseries 0.00 & \bfseries -552.83 & 0.087\\
MST-Star & \bfseries -60.00 & \bfseries 0.00 & \bfseries -552.83 & 0.086\\
PPI & \bfseries -60.00 & \bfseries 0.00 & \bfseries -552.83 & 0.012\\
Stiefel & \bfseries -60.00 & \bfseries 0.00 & \bfseries -552.83 & 0.086\\
Spectral & \bfseries -60.00 & \bfseries 0.00 & \bfseries -552.83 & 0.275\\
\bottomrule
\end{tabular}
\end{table*}

\begin{table*}
\centering
\caption{\textbf{Synchronization Results for synthetic complete-8}. Averaged over all instances. Pre-processing time: 0.15 s.}
\begin{tabular}{lCCCC}
\toprule
Method & M-LAP obj. \textdownarrow & \#Forb.\textdownarrow & MGM obj. \textdownarrow & Runtime Sync \textdownarrow \\
\midrule
GREEDA (Sync)(best) & \bfseries -280.00 & \bfseries 0.00 & \bfseries -2580.57 & 0.012\\
GREEDA (Sync)(worst) & \bfseries -280.00 & \bfseries 0.00 & \bfseries -2580.57 & 0.012\\
MP-T (Sync) & \bfseries -280.00 & \bfseries 0.00 & \bfseries -2580.57 & \bfseries 0.009\\
MST & \bfseries -280.00 & \bfseries 0.00 & \bfseries -2580.57 & 0.393\\
MST-Star & \bfseries -280.00 & \bfseries 0.00 & \bfseries -2580.57 & 0.391\\
PPI & \bfseries -280.00 & \bfseries 0.00 & \bfseries -2580.57 & 0.014\\
Stiefel & \bfseries -280.00 & \bfseries 0.00 & \bfseries -2580.57 & 0.085\\
Spectral & \bfseries -280.00 & \bfseries 0.00 & \bfseries -2580.57 & 0.256\\
\bottomrule
\end{tabular}
\end{table*}

\begin{table*}
\centering
\caption{\textbf{Synchronization Results for synthetic complete-12}. Averaged over all instances. Pre-processing time: 0.34 s.}
\begin{tabular}{lCCCC}
\toprule
Method & M-LAP obj. \textdownarrow & \#Forb.\textdownarrow & MGM obj. \textdownarrow & Runtime Sync \textdownarrow \\
\midrule
GREEDA (Sync)(best) & \bfseries -660.00 & \bfseries 0.00 & \bfseries -6076.54 & 0.019\\
GREEDA (Sync)(worst) & \bfseries -660.00 & \bfseries 0.00 & \bfseries -6076.54 & 0.019\\
MP-T (Sync) & \bfseries -660.00 & \bfseries 0.00 & \bfseries -6076.54 & 0.018\\
MST & \bfseries -660.00 & \bfseries 0.00 & \bfseries -6076.54 & 0.916\\
MST-Star & \bfseries -660.00 & \bfseries 0.00 & \bfseries -6076.54 & 0.914\\
PPI & \bfseries -660.00 & \bfseries 0.00 & \bfseries -6076.54 & \bfseries 0.014\\
Stiefel & \bfseries -660.00 & \bfseries 0.00 & \bfseries -6076.54 & 0.096\\
Spectral & \bfseries -660.00 & \bfseries 0.00 & \bfseries -6076.54 & 0.278\\
\bottomrule
\end{tabular}
\end{table*}

\begin{table*}
\centering
\caption{\textbf{Synchronization Results for synthetic complete-16}. Averaged over all instances. Pre-processing time: 0.61 s.}
\begin{tabular}{lCCCC}
\toprule
Method & M-LAP obj. \textdownarrow & \#Forb.\textdownarrow & MGM obj. \textdownarrow & Runtime Sync \textdownarrow \\
\midrule
GREEDA (Sync)(best) & \bfseries -1200.00 & \bfseries 0.00 & \bfseries -11048.86 & 0.029\\
GREEDA (Sync)(worst) & \bfseries -1200.00 & \bfseries 0.00 & \bfseries -11048.86 & 0.029\\
MP-T (Sync) & \bfseries -1200.00 & \bfseries 0.00 & \bfseries -11048.86 & 0.031\\
MST & \bfseries -1200.00 & \bfseries 0.00 & \bfseries -11048.86 & 1.624\\
MST-Star & \bfseries -1200.00 & \bfseries 0.00 & \bfseries -11048.86 & 1.634\\
PPI & \bfseries -1200.00 & \bfseries 0.00 & \bfseries -11048.86 & \bfseries 0.015\\
Stiefel & \bfseries -1200.00 & \bfseries 0.00 & \bfseries -11048.86 & 0.088\\
Spectral & \bfseries -1200.00 & \bfseries 0.00 & \bfseries -11048.86 & 0.269\\
\bottomrule
\end{tabular}
\end{table*}

\begin{table*}
\centering
\caption{\textbf{Synchronization Results for synthetic deform-4}. Averaged over all instances. Pre-processing time: 0.05 s.}
\begin{tabular}{lCCCC}
\toprule
Method & M-LAP obj. \textdownarrow & \#Forb.\textdownarrow & MGM obj. \textdownarrow & Runtime Sync \textdownarrow \\
\midrule
GREEDA (Sync)(best) & \bfseries -38.50 & \bfseries 0.00 & -284.85 & \bfseries 0.009\\
GREEDA (Sync)(worst) & -37.10 & \bfseries 0.00 & -281.61 & 0.010\\
MP-T (Sync) & -36.90 & \bfseries 0.00 & -243.67 & 0.120\\
MST & -35.40 & \bfseries 0.00 & -308.42 & 0.086\\
MST-Star & -35.30 & \bfseries 0.00 & \bfseries -308.73 & 0.086\\
PPI & -32.60 & \bfseries 0.00 & -255.42 & 0.013\\
Stiefel & -37.60 & \bfseries 0.00 & -280.14 & 0.117\\
Spectral & -36.10 & \bfseries 0.00 & -294.34 & 0.259\\
\bottomrule
\end{tabular}
\end{table*}

\begin{table*}
\centering
\caption{\textbf{Synchronization Results for synthetic deform-8}. Averaged over all instances. Pre-processing time: 0.19 s.}
\begin{tabular}{lCCCC}
\toprule
Method & M-LAP obj. \textdownarrow & \#Forb.\textdownarrow & MGM obj. \textdownarrow & Runtime Sync \textdownarrow \\
\midrule
GREEDA (Sync)(best) & \bfseries -144.90 & \bfseries 0.00 & -1303.02 & 0.064\\
GREEDA (Sync)(worst) & -138.00 & \bfseries 0.00 & -1227.35 & 0.066\\
MP-T (Sync) & -139.80 & \bfseries 0.00 & -1156.88 & 0.692\\
MST & -117.70 & \bfseries 0.00 & -1345.97 & 0.390\\
MST-Star & -120.00 & \bfseries 0.00 & \bfseries -1374.90 & 0.393\\
PPI & -131.40 & \bfseries 0.00 & -1201.48 & \bfseries 0.016\\
Stiefel & -135.20 & \bfseries 0.00 & -1226.57 & 0.134\\
Spectral & -127.20 & \bfseries 0.00 & -1190.27 & 0.262\\
\bottomrule
\end{tabular}
\end{table*}

\begin{table*}
\centering
\caption{\textbf{Synchronization Results for synthetic deform-12}. Averaged over all instances. Pre-processing time: 0.45 s.}
\begin{tabular}{lCCCC}
\toprule
Method & M-LAP obj. \textdownarrow & \#Forb.\textdownarrow & MGM obj. \textdownarrow & Runtime Sync \textdownarrow \\
\midrule
GREEDA (Sync)(best) & \bfseries -329.50 & \bfseries 0.00 & -3160.21 & 0.279\\
GREEDA (Sync)(worst) & -317.50 & \bfseries 0.00 & -3044.68 & 0.129\\
MP-T (Sync) & -322.30 & \bfseries 0.00 & -3088.39 & 1.782\\
MST & -246.00 & \bfseries 0.00 & -3072.96 & 0.912\\
MST-Star & -253.50 & \bfseries 0.00 & \bfseries -3208.23 & 0.919\\
PPI & -308.60 & \bfseries 0.00 & -2998.04 & \bfseries 0.021\\
Stiefel & -318.00 & \bfseries 0.00 & -3136.55 & 0.150\\
Spectral & -304.40 & \bfseries 0.00 & -3018.41 & 0.269\\
\bottomrule
\end{tabular}
\end{table*}

\begin{table*}
\centering
\caption{\textbf{Synchronization Results for synthetic deform-16}. Averaged over all instances. Pre-processing time: 0.81 s.}
\begin{tabular}{lCCCC}
\toprule
Method & M-LAP obj. \textdownarrow & \#Forb.\textdownarrow & MGM obj. \textdownarrow & Runtime Sync \textdownarrow \\
\midrule
GREEDA (Sync)(best) & \bfseries -559.70 & \bfseries 0.00 & -5794.40 & 0.588\\
GREEDA (Sync)(worst) & -541.60 & \bfseries 0.00 & -5620.65 & 0.375\\
MP-T (Sync) & -552.60 & \bfseries 0.00 & -5713.54 & 3.485\\
MST & -419.60 & \bfseries 0.00 & -5659.78 & 1.623\\
MST-Star & -404.90 & \bfseries 0.00 & \bfseries -5798.54 & 1.632\\
PPI & -525.30 & \bfseries 0.00 & -5516.54 & \bfseries 0.024\\
Stiefel & -533.20 & \bfseries 0.00 & -5653.30 & 0.143\\
Spectral & -517.30 & \bfseries 0.00 & -5536.65 & 0.278\\
\bottomrule
\end{tabular}
\end{table*}

\begin{table*}
\centering
\caption{\textbf{Synchronization Results for synthetic density-4}. Averaged over all instances. Pre-processing time: 0.04 s.}
\begin{tabular}{lCCCC}
\toprule
Method & M-LAP obj. \textdownarrow & \#Forb.\textdownarrow & MGM obj. \textdownarrow & Runtime Sync \textdownarrow \\
\midrule
GREEDA (Sync)(best) & \bfseries -42.40 & \bfseries 0.00 & -189.88 & 0.009\\
GREEDA (Sync)(worst) & -41.90 & \bfseries 0.00 & -184.95 & \bfseries 0.008\\
MP-T (Sync) & -41.70 & \bfseries 0.00 & -184.50 & 0.104\\
MST & -39.20 & \bfseries 0.00 & -198.15 & 0.086\\
MST-Star & -39.20 & \bfseries 0.00 & \bfseries -199.97 & 0.087\\
PPI & -39.30 & \bfseries 0.00 & -179.14 & 0.013\\
Stiefel & -41.60 & \bfseries 0.00 & -184.99 & 0.102\\
Spectral & -40.90 & \bfseries 0.00 & -192.39 & 0.265\\
\bottomrule
\end{tabular}
\end{table*}

\begin{table*}
\centering
\caption{\textbf{Synchronization Results for synthetic density-8}. Averaged over all instances. Pre-processing time: 0.18 s.}
\begin{tabular}{lCCCC}
\toprule
Method & M-LAP obj. \textdownarrow & \#Forb.\textdownarrow & MGM obj. \textdownarrow & Runtime Sync \textdownarrow \\
\midrule
GREEDA (Sync)(best) & \bfseries -163.20 & \bfseries 0.00 & -832.43 & 0.040\\
GREEDA (Sync)(worst) & -157.50 & \bfseries 0.00 & -808.54 & 0.049\\
MP-T (Sync) & -158.80 & \bfseries 0.00 & -780.07 & 0.647\\
MST & -144.50 & \bfseries 0.00 & \bfseries -870.00 & 0.388\\
MST-Star & -142.90 & \bfseries 0.00 & -867.31 & 0.390\\
PPI & -143.40 & \bfseries 0.00 & -763.99 & \bfseries 0.016\\
Stiefel & -157.90 & \bfseries 0.00 & -808.39 & 0.123\\
Spectral & -155.40 & \bfseries 0.00 & -810.87 & 0.274\\
\bottomrule
\end{tabular}
\end{table*}

\begin{table*}
\centering
\caption{\textbf{Synchronization Results for synthetic density-12}. Averaged over all instances. Pre-processing time: 0.40 s.}
\begin{tabular}{lCCCC}
\toprule
Method & M-LAP obj. \textdownarrow & \#Forb.\textdownarrow & MGM obj. \textdownarrow & Runtime Sync \textdownarrow \\
\midrule
GREEDA (Sync)(best) & \bfseries -364.40 & \bfseries 0.00 & \bfseries -2056.16 & 0.160\\
GREEDA (Sync)(worst) & -357.10 & \bfseries 0.00 & -1975.51 & 0.134\\
MP-T (Sync) & -357.70 & \bfseries 0.00 & -1989.32 & 1.757\\
MST & -311.10 & \bfseries 0.00 & -1991.32 & 0.913\\
MST-Star & -301.30 & \bfseries 0.00 & -2022.73 & 0.906\\
PPI & -344.60 & \bfseries 0.00 & -1920.39 & \bfseries 0.019\\
Stiefel & -354.50 & \bfseries 0.00 & -1981.91 & 0.133\\
Spectral & -343.20 & \bfseries 0.00 & -1960.10 & 0.257\\
\bottomrule
\end{tabular}
\end{table*}

\begin{table*}
\centering
\caption{\textbf{Synchronization Results for synthetic density-16}. Averaged over all instances. Pre-processing time: 0.73 s.}
\begin{tabular}{lCCCC}
\toprule
Method & M-LAP obj. \textdownarrow & \#Forb.\textdownarrow & MGM obj. \textdownarrow & Runtime Sync \textdownarrow \\
\midrule
GREEDA (Sync)(best) & \bfseries -661.80 & \bfseries 0.00 & \bfseries -3837.84 & 0.218\\
GREEDA (Sync)(worst) & -652.40 & \bfseries 0.00 & -3731.19 & 0.302\\
MP-T (Sync) & -657.40 & \bfseries 0.00 & -3767.17 & 3.423\\
MST & -555.90 & \bfseries 0.00 & -3618.30 & 1.628\\
MST-Star & -531.00 & \bfseries 0.00 & -3672.19 & 1.625\\
PPI & -649.10 & \bfseries 0.00 & -3725.49 & \bfseries 0.022\\
Stiefel & -644.30 & \bfseries 0.00 & -3724.86 & 0.150\\
Spectral & -634.90 & \bfseries 0.00 & -3674.14 & 0.286\\
\bottomrule
\end{tabular}
\end{table*}

\begin{table*}
\centering
\caption{\textbf{Synchronization Results for synthetic outlier-4}. Averaged over all instances. Pre-processing time: 0.04 s.}
\begin{tabular}{lCCCC}
\toprule
Method & M-LAP obj. \textdownarrow & \#Forb.\textdownarrow & MGM obj. \textdownarrow & Runtime Sync \textdownarrow \\
\midrule
GREEDA (Sync)(best) & \bfseries -36.40 & \bfseries 0.00 & 109.10 & 0.007\\
GREEDA (Sync)(worst) & \bfseries -36.40 & \bfseries 0.00 & 109.10 & 0.007\\
MP-T (Sync) & \bfseries -36.40 & \bfseries 0.00 & 100.51 & \bfseries 0.003\\
MST & \bfseries -36.40 & \bfseries 0.00 & \bfseries -43.36 & 0.085\\
MST-Star & \bfseries -36.40 & \bfseries 0.00 & \bfseries -43.36 & 0.085\\
PPI & -34.70 & \bfseries 0.00 & 117.38 & 0.011\\
Stiefel & \bfseries -36.40 & \bfseries 0.00 & 115.14 & 0.086\\
Spectral & \bfseries -36.40 & \bfseries 0.00 & 60.17 & 0.252\\
\bottomrule
\end{tabular}
\end{table*}

\begin{table*}
\centering
\caption{\textbf{Synchronization Results for synthetic outlier-8}. Averaged over all instances. Pre-processing time: 0.16 s.}
\begin{tabular}{lCCCC}
\toprule
Method & M-LAP obj. \textdownarrow & \#Forb.\textdownarrow & MGM obj. \textdownarrow & Runtime Sync \textdownarrow \\
\midrule
GREEDA (Sync)(best) & \bfseries -169.40 & \bfseries 0.00 & 530.29 & 0.013\\
GREEDA (Sync)(worst) & \bfseries -169.40 & \bfseries 0.00 & 530.29 & 0.013\\
MP-T (Sync) & \bfseries -169.40 & \bfseries 0.00 & 491.36 & \bfseries 0.009\\
MST & \bfseries -169.40 & \bfseries 0.00 & \bfseries -202.44 & 0.388\\
MST-Star & -168.90 & \bfseries 0.00 & -202.31 & 0.393\\
PPI & -167.30 & \bfseries 0.00 & 553.44 & 0.013\\
Stiefel & \bfseries -169.40 & \bfseries 0.00 & 516.19 & 0.089\\
Spectral & -169.20 & \bfseries 0.00 & 274.18 & 0.256\\
\bottomrule
\end{tabular}
\end{table*}

\begin{table*}
\centering
\caption{\textbf{Synchronization Results for synthetic outlier-12}. Averaged over all instances. Pre-processing time: 0.35 s.}
\begin{tabular}{lCCCC}
\toprule
Method & M-LAP obj. \textdownarrow & \#Forb.\textdownarrow & MGM obj. \textdownarrow & Runtime Sync \textdownarrow \\
\midrule
GREEDA (Sync)(best) & \bfseries -397.70 & \bfseries 0.00 & 1243.78 & 0.022\\
GREEDA (Sync)(worst) & \bfseries -397.70 & \bfseries 0.00 & 1243.78 & 0.022\\
MP-T (Sync) & \bfseries -397.70 & \bfseries 0.00 & 1020.35 & 0.067\\
MST & -397.50 & \bfseries 0.00 & \bfseries -476.33 & 0.913\\
MST-Star & -396.10 & \bfseries 0.00 & -476.08 & 0.910\\
PPI & -393.50 & \bfseries 0.00 & 1292.49 & \bfseries 0.014\\
Stiefel & -394.80 & \bfseries 0.00 & 1261.04 & 0.091\\
Spectral & -397.10 & \bfseries 0.00 & 722.74 & 0.260\\
\bottomrule
\end{tabular}
\end{table*}

\begin{table*}
\centering
\caption{\textbf{Synchronization Results for synthetic outlier-16}. Averaged over all instances. Pre-processing time: 0.64 s.}
\begin{tabular}{lCCCC}
\toprule
Method & M-LAP obj. \textdownarrow & \#Forb.\textdownarrow & MGM obj. \textdownarrow & Runtime Sync \textdownarrow \\
\midrule
GREEDA (Sync)(best) & \bfseries -721.60 & \bfseries 0.00 & 2278.62 & 0.035\\
GREEDA (Sync)(worst) & \bfseries -721.60 & \bfseries 0.00 & 2278.62 & 0.035\\
MP-T (Sync) & \bfseries -721.60 & \bfseries 0.00 & 1816.98 & 0.287\\
MST & -721.10 & \bfseries 0.00 & \bfseries -866.38 & 1.640\\
MST-Star & -718.50 & \bfseries 0.00 & -865.78 & 1.620\\
PPI & -719.30 & \bfseries 0.00 & 2269.22 & \bfseries 0.014\\
Stiefel & -710.80 & \bfseries 0.00 & 2340.42 & 0.095\\
Spectral & -721.00 & \bfseries 0.00 & 1211.50 & 0.248\\
\bottomrule
\end{tabular}
\end{table*}

\begin{table*}
\centering
\caption{\textbf{Synchronization Results for hotel-4}. Averaged over all instances. Pre-processing time: 0.04 s.}
\begin{tabular}{lCCCC}
\toprule
Method & M-LAP obj. \textdownarrow & \#Forb.\textdownarrow & MGM obj. \textdownarrow & Runtime Sync \textdownarrow \\
\midrule
GREEDA (Sync)(best) & \bfseries -32.50 & \bfseries 0.00 & 181.41 & 0.008\\
GREEDA (Sync)(worst) & -32.10 & \bfseries 0.00 & 178.84 & \bfseries 0.007\\
MP-T (Sync) & \bfseries -32.50 & \bfseries 0.00 & 126.65 & 0.055\\
MST & -29.60 & \bfseries 0.00 & \bfseries -19.05 & 0.086\\
MST-Star & -28.90 & \bfseries 0.00 & -18.36 & 0.086\\
PPI & -27.50 & \bfseries 0.00 & 219.93 & 0.012\\
Stiefel & -31.40 & \bfseries 0.00 & 205.32 & 0.105\\
Spectral & -30.90 & \bfseries 0.00 & 144.78 & 0.276\\
\bottomrule
\end{tabular}
\end{table*}

\begin{table*}
\centering
\caption{\textbf{Synchronization Results for hotel-8}. Averaged over all instances. Pre-processing time: 0.15 s.}
\begin{tabular}{lCCCC}
\toprule
Method & M-LAP obj. \textdownarrow & \#Forb.\textdownarrow & MGM obj. \textdownarrow & Runtime Sync \textdownarrow \\
\midrule
GREEDA (Sync)(best) & -150.30 & \bfseries 0.00 & 729.11 & 0.024\\
GREEDA (Sync)(worst) & -148.80 & \bfseries 0.00 & 753.05 & 0.021\\
MP-T (Sync) & \bfseries -151.10 & \bfseries 0.00 & 399.68 & 0.375\\
MST & -139.70 & \bfseries 0.00 & \bfseries -79.01 & 0.389\\
MST-Star & -127.20 & \bfseries 0.00 & -77.93 & 0.388\\
PPI & -139.50 & \bfseries 0.00 & 886.68 & \bfseries 0.014\\
Stiefel & -140.80 & \bfseries 0.00 & 882.17 & 0.117\\
Spectral & -138.50 & \bfseries 0.00 & 534.18 & 0.273\\
\bottomrule
\end{tabular}
\end{table*}

\begin{table*}
\centering
\caption{\textbf{Synchronization Results for hotel-12}. Averaged over all instances. Pre-processing time: 0.34 s.}
\begin{tabular}{lCCCC}
\toprule
Method & M-LAP obj. \textdownarrow & \#Forb.\textdownarrow & MGM obj. \textdownarrow & Runtime Sync \textdownarrow \\
\midrule
GREEDA (Sync)(best) & -358.70 & \bfseries 0.00 & 1718.04 & 0.080\\
GREEDA (Sync)(worst) & -355.10 & \bfseries 0.00 & 1734.70 & 0.053\\
MP-T (Sync) & \bfseries -360.70 & \bfseries 0.00 & 652.15 & 1.191\\
MST & -334.00 & \bfseries 0.00 & -171.31 & 0.913\\
MST-Star & -297.50 & \bfseries 0.00 & \bfseries -177.42 & 0.904\\
PPI & -335.90 & \bfseries 0.00 & 1977.17 & \bfseries 0.016\\
Stiefel & -329.40 & \bfseries 0.00 & 2087.22 & 0.141\\
Spectral & -327.30 & \bfseries 0.00 & 1363.31 & 0.292\\
\bottomrule
\end{tabular}
\end{table*}

\begin{table*}
\centering
\caption{\textbf{Synchronization Results for hotel-16}. Averaged over all instances. Pre-processing time: 0.61 s.}
\begin{tabular}{lCCCC}
\toprule
Method & M-LAP obj. \textdownarrow & \#Forb.\textdownarrow & MGM obj. \textdownarrow & Runtime Sync \textdownarrow \\
\midrule
GREEDA (Sync)(best) & -652.40 & \bfseries 0.00 & 3115.66 & 0.174\\
GREEDA (Sync)(worst) & -646.80 & \bfseries 0.00 & 3214.51 & 0.145\\
MP-T (Sync) & \bfseries -657.90 & \bfseries 0.00 & 1037.56 & 2.651\\
MST & -601.10 & \bfseries 0.00 & -301.61 & 1.612\\
MST-Star & -524.40 & \bfseries 0.00 & \bfseries -315.56 & 1.614\\
PPI & -634.30 & \bfseries 0.00 & 3272.17 & \bfseries 0.019\\
Stiefel & -611.60 & \bfseries 0.00 & 3588.96 & 0.129\\
Spectral & -606.60 & \bfseries 0.00 & 2214.57 & 0.296\\
\bottomrule
\end{tabular}
\end{table*}

\begin{table*}
\centering
\caption{\textbf{Synchronization Results for house-4}. Averaged over all instances. Pre-processing time: 0.03 s.}
\begin{tabular}{lCCCC}
\toprule
Method & M-LAP obj. \textdownarrow & \#Forb.\textdownarrow & MGM obj. \textdownarrow & Runtime Sync \textdownarrow \\
\midrule
GREEDA (Sync)(best) & \bfseries -37.70 & \bfseries 0.00 & 126.99 & 0.007\\
GREEDA (Sync)(worst) & -36.90 & \bfseries 0.00 & 137.41 & \bfseries 0.007\\
MP-T (Sync) & \bfseries -37.70 & \bfseries 0.00 & 89.06 & 0.044\\
MST & -35.50 & \bfseries 0.00 & \bfseries -24.50 & 0.087\\
MST-Star & -35.30 & \bfseries 0.00 & -23.81 & 0.085\\
PPI & -32.80 & \bfseries 0.00 & 181.14 & 0.012\\
Stiefel & -36.30 & \bfseries 0.00 & 150.10 & 0.103\\
Spectral & -36.20 & \bfseries 0.00 & 101.88 & 0.280\\
\bottomrule
\end{tabular}
\end{table*}

\begin{table*}
\centering
\caption{\textbf{Synchronization Results for house-8}. Averaged over all instances. Pre-processing time: 0.14 s.}
\begin{tabular}{lCCCC}
\toprule
Method & M-LAP obj. \textdownarrow & \#Forb.\textdownarrow & MGM obj. \textdownarrow & Runtime Sync \textdownarrow \\
\midrule
GREEDA (Sync)(best) & -172.70 & \bfseries 0.00 & 618.13 & 0.020\\
GREEDA (Sync)(worst) & -170.80 & \bfseries 0.00 & 620.82 & 0.018\\
MP-T (Sync) & \bfseries -173.70 & \bfseries 0.00 & 274.25 & 0.292\\
MST & -159.30 & \bfseries 0.00 & \bfseries -103.69 & 0.395\\
MST-Star & -157.30 & \bfseries 0.00 & -101.73 & 0.396\\
PPI & -165.90 & \bfseries 0.00 & 702.60 & \bfseries 0.014\\
Stiefel & -166.50 & \bfseries 0.00 & 718.55 & 0.113\\
Spectral & -163.80 & \bfseries 0.00 & 344.20 & 0.281\\
\bottomrule
\end{tabular}
\end{table*}

\begin{table*}
\centering
\caption{\textbf{Synchronization Results for house-12}. Averaged over all instances. Pre-processing time: 0.33 s.}
\begin{tabular}{lCCCC}
\toprule
Method & M-LAP obj. \textdownarrow & \#Forb.\textdownarrow & MGM obj. \textdownarrow & Runtime Sync \textdownarrow \\
\midrule
GREEDA (Sync)(best) & -406.40 & \bfseries 0.00 & 1462.16 & 0.060\\
GREEDA (Sync)(worst) & -402.80 & \bfseries 0.00 & 1442.94 & 0.060\\
MP-T (Sync) & \bfseries -410.10 & \bfseries 0.00 & 457.94 & 1.268\\
MST & -367.60 & \bfseries 0.00 & -244.94 & 0.912\\
MST-Star & -363.40 & \bfseries 0.00 & \bfseries -245.03 & 0.914\\
PPI & -393.50 & \bfseries 0.00 & 1588.60 & \bfseries 0.016\\
Stiefel & -385.10 & \bfseries 0.00 & 1688.97 & 0.141\\
Spectral & -387.80 & \bfseries 0.00 & 843.99 & 0.269\\
\bottomrule
\end{tabular}
\end{table*}

\begin{table*}
\centering
\caption{\textbf{Synchronization Results for house-16}. Averaged over all instances. Pre-processing time: 0.60 s.}
\begin{tabular}{lCCCC}
\toprule
Method & M-LAP obj. \textdownarrow & \#Forb.\textdownarrow & MGM obj. \textdownarrow & Runtime Sync \textdownarrow \\
\midrule
GREEDA (Sync)(best) & -733.80 & \bfseries 0.00 & 2588.66 & 0.139\\
GREEDA (Sync)(worst) & -728.60 & \bfseries 0.00 & 2653.82 & 0.090\\
MP-T (Sync) & \bfseries -742.90 & \bfseries 0.00 & 634.54 & 2.873\\
MST & -658.60 & \bfseries 0.00 & \bfseries -440.94 & 1.623\\
MST-Star & -635.80 & \bfseries 0.00 & -429.01 & 1.616\\
PPI & -719.10 & \bfseries 0.00 & 2771.18 & \bfseries 0.020\\
Stiefel & -706.30 & \bfseries 0.00 & 2895.95 & 0.124\\
Spectral & -708.30 & \bfseries 0.00 & 1681.84 & 0.260\\
\bottomrule
\end{tabular}
\end{table*}

\begin{table*}
\centering
\caption{\textbf{Synchronization Results for worms-3}. Averaged over all instances. Pre-processing time: 39.26 s.}
\begin{tabular}{lCCCC}
\toprule
Method & M-LAP obj. \textdownarrow & \#Forb.\textdownarrow & MGM obj. \textdownarrow & Runtime Sync \textdownarrow \\
\midrule
GREEDA (Sync)(best) & \bfseries -1379.50 & \bfseries 0.00 & \bfseries -209313.79 & 0.284\\
GREEDA (Sync)(worst) & -1377.50 & \bfseries 0.00 & -209081.76 & \bfseries 0.282\\
MP-T (Sync) & -1350.60 & \bfseries 0.00 & -202457.88 & 129.673\\
MST & -- & -- & -- & --\\
MST-Star & -- & -- & -- & --\\
PPI & -- & -- & -- & --\\
Stiefel & -- & -- & -- & --\\
Spectral & -- & -- & -- & --\\
\bottomrule
\end{tabular}
\end{table*}

\begin{table*}
\centering
\caption{\textbf{Synchronization Results for worms-4}. Averaged over all instances. Pre-processing time: 77.57 s.}
\begin{tabular}{lCCCC}
\toprule
Method & M-LAP obj. \textdownarrow & \#Forb.\textdownarrow & MGM obj. \textdownarrow & Runtime Sync \textdownarrow \\
\midrule
GREEDA (Sync)(best) & \bfseries -2687.30 & \bfseries 0.00 & \bfseries -417463.71 & 0.879\\
GREEDA (Sync)(worst) & -2678.20 & \bfseries 0.00 & -416305.14 & \bfseries 0.869\\
MP-T (Sync) & -2655.30 & \bfseries 0.00 & -409829.74 & 269.629\\
MST & -- & -- & -- & --\\
MST-Star & -- & -- & -- & --\\
PPI & -- & -- & -- & --\\
Stiefel & -- & -- & -- & --\\
Spectral & -- & -- & -- & --\\
\bottomrule
\end{tabular}
\end{table*}

\begin{table*}
\centering
\caption{\textbf{Synchronization Results for worms-5}. Averaged over all instances. Pre-processing time: 131.42 s.}
\begin{tabular}{lCCCC}
\toprule
Method & M-LAP obj. \textdownarrow & \#Forb.\textdownarrow & MGM obj. \textdownarrow & Runtime Sync \textdownarrow \\
\midrule
GREEDA (Sync)(best) & \bfseries -4383.60 & \bfseries 0.00 & \bfseries -699850.20 & \bfseries 1.436\\
GREEDA (Sync)(worst) & -4366.80 & \bfseries 0.00 & -699342.74 & 1.491\\
MP-T (Sync) & -4334.40 & \bfseries 0.00 & -688762.44 & 487.991\\
MST & -- & -- & -- & --\\
MST-Star & -- & -- & -- & --\\
PPI & -- & -- & -- & --\\
Stiefel & -- & -- & -- & --\\
Spectral & -- & -- & -- & --\\
\bottomrule
\end{tabular}
\end{table*}

\begin{table*}
\centering
\caption{\textbf{Synchronization Results for worms-6}. Averaged over all instances. Pre-processing time: 191.34 s.}
\begin{tabular}{lCCCC}
\toprule
Method & M-LAP obj. \textdownarrow & \#Forb.\textdownarrow & MGM obj. \textdownarrow & Runtime Sync \textdownarrow \\
\midrule
GREEDA (Sync)(best) & \bfseries -6446.20 & \bfseries 0.00 & \bfseries -1050112.44 & 2.376\\
GREEDA (Sync)(worst) & -6421.30 & \bfseries 0.00 & -1046558.05 & \bfseries 2.169\\
MP-T (Sync) & -6394.00 & \bfseries 0.00 & -1035444.15 & 801.740\\
MST & -- & -- & -- & --\\
MST-Star & -- & -- & -- & --\\
PPI & -- & -- & -- & --\\
Stiefel & -- & -- & -- & --\\
Spectral & -- & -- & -- & --\\
\bottomrule
\end{tabular}
\end{table*}

\begin{table*}
\centering
\caption{\textbf{Synchronization Results for worms-7}. Averaged over all instances. Pre-processing time: 264.08 s.}
\begin{tabular}{lCCCC}
\toprule
Method & M-LAP obj. \textdownarrow & \#Forb.\textdownarrow & MGM obj. \textdownarrow & Runtime Sync \textdownarrow \\
\midrule
GREEDA (Sync)(best) & \bfseries -8899.30 & \bfseries 0.00 & \bfseries -1465515.57 & 3.455\\
GREEDA (Sync)(worst) & -8861.10 & \bfseries 0.00 & -1462942.47 & \bfseries 3.183\\
MP-T (Sync) & -8816.20 & \bfseries 0.00 & -1448013.99 & 1183.435\\
MST & -- & -- & -- & --\\
MST-Star & -- & -- & -- & --\\
PPI & -- & -- & -- & --\\
Stiefel & -- & -- & -- & --\\
Spectral & -- & -- & -- & --\\
\bottomrule
\end{tabular}
\end{table*}

\begin{table*}
\centering
\caption{\textbf{Synchronization Results for worms-8}. Averaged over all instances. Pre-processing time: 352.87 s.}
\begin{tabular}{lCCCC}
\toprule
Method & M-LAP obj. \textdownarrow & \#Forb.\textdownarrow & MGM obj. \textdownarrow & Runtime Sync \textdownarrow \\
\midrule
GREEDA (Sync)(best) & \bfseries -11704.80 & \bfseries 0.00 & \bfseries -1951795.10 & 4.551\\
GREEDA (Sync)(worst) & -11659.20 & \bfseries 0.00 & -1947156.36 & \bfseries 4.277\\
MP-T (Sync) & -11606.70 & \bfseries 0.00 & -1933383.35 & 1677.381\\
MST & -- & -- & -- & --\\
MST-Star & -- & -- & -- & --\\
PPI & -- & -- & -- & --\\
Stiefel & -- & -- & -- & --\\
Spectral & -- & -- & -- & --\\
\bottomrule
\end{tabular}
\end{table*}

\begin{table*}
\centering
\caption{\textbf{Synchronization Results for worms-9}. Averaged over all instances. Pre-processing time: 454.05 s.}
\begin{tabular}{lCCCC}
\toprule
Method & M-LAP obj. \textdownarrow & \#Forb.\textdownarrow & MGM obj. \textdownarrow & Runtime Sync \textdownarrow \\
\midrule
GREEDA (Sync)(best) & \bfseries -14952.80 & \bfseries 0.00 & \bfseries -2512570.77 & 6.433\\
GREEDA (Sync)(worst) & -14901.30 & \bfseries 0.00 & -2510013.86 & \bfseries 5.668\\
MP-T (Sync) & -14832.10 & \bfseries 0.00 & -2490777.10 & 2280.563\\
MST & -- & -- & -- & --\\
MST-Star & -- & -- & -- & --\\
PPI & -- & -- & -- & --\\
Stiefel & -- & -- & -- & --\\
Spectral & -- & -- & -- & --\\
\bottomrule
\end{tabular}
\end{table*}

\begin{table*}
\centering
\caption{\textbf{Synchronization Results for worms-10}. Averaged over all instances. Pre-processing time: 569.02 s.}
\begin{tabular}{lCCCC}
\toprule
Method & M-LAP obj. \textdownarrow & \#Forb.\textdownarrow & MGM obj. \textdownarrow & Runtime Sync \textdownarrow \\
\midrule
GREEDA (Sync)(best) & -18510.40 & \bfseries 0.00 & \bfseries -3136748.13 & 8.357\\
GREEDA (Sync)(worst) & -18450.70 & \bfseries 0.00 & -3132603.41 & 6.961\\
MP-T (Sync) & -18332.50 & \bfseries 0.00 & -3108349.88 & 3008.551\\
MST & -14740.20 & 360.20 & inf & 113.575\\
MST-Star & -16349.00 & 213.30 & inf & 113.726\\
PPI ($\alpha$=0) & -18427.80 & 791.40 & inf & \bfseries 3.600\\
PPI ($\alpha$=1) & -18390.40 & 274.50 & inf & 4.956\\
PPI ($\alpha$=10) & -17575.70 & 303.20 & inf & 6.216\\
PPI ($\alpha$=50) & -16585.50 & 388.70 & inf & 5.910\\
Stiefel ($\alpha$=0) & -18521.00 & 712.60 & inf & 25.089\\
Stiefel ($\alpha$=1) & -3913.00 & 15426.10 & inf & 88.420\\
Stiefel ($\alpha$=10) & -1118.40 & 17744.60 & inf & 87.839\\
Stiefel ($\alpha$=50) & -1060.50 & 17730.90 & inf & 87.645\\
Spectral ($\alpha$=0) & -18461.40 & 579.50 & inf & 6.105\\
Spectral ($\alpha$=1) & -14212.50 & 2492.90 & inf & 5.379\\
Spectral ($\alpha$=10) & -10530.80 & 3662.50 & inf & 4.286\\
Spectral ($\alpha$=50) & -10700.00 & 3643.50 & inf & 4.384\\
MP-T (Sync)($\alpha$=0) & -18496.60 & 598.80 & inf & 30591.339\\
GREEDA (Sync)(best)($\alpha$=0) & \bfseries -18669.20 & 691.50 & inf & 189.399\\
GREEDA (Sync)(worst)($\alpha$=0) & -18621.70 & 755.00 & inf & 202.476\\
\bottomrule
\end{tabular}
\end{table*}

\begin{table*}
\centering
\caption{\textbf{Synchronization Results for worms-20}. Averaged over all instances. Pre-processing time: 879.33 s.}
\begin{tabular}{lCCCC}
\toprule
Method & M-LAP obj. \textdownarrow & \#Forb.\textdownarrow & MGM obj. \textdownarrow & Runtime Sync \textdownarrow \\
\midrule
GREEDA (Sync)(best) & \bfseries -57597.75 & \bfseries 0.00 & \bfseries -11136025.34 & 37.722\\
GREEDA (Sync)(worst) & -57115.25 & \bfseries 0.00 & -11090742.15 & \bfseries 36.263\\
MP-T (Sync) & -56568.00 & \bfseries 0.00 & -10996530.65 & 18413.977\\
MST & -- & -- & -- & --\\
MST-Star & -- & -- & -- & --\\
PPI & -- & -- & -- & --\\
Stiefel & -- & -- & -- & --\\
Spectral & -- & -- & -- & --\\
\bottomrule
\end{tabular}
\end{table*}

\begin{table*}
\centering
\caption{\textbf{Synchronization Results for worms-30}. Averaged over all instances. Pre-processing time: 2025.94 s.}
\begin{tabular}{lCCCC}
\toprule
Method & M-LAP obj. \textdownarrow & \#Forb.\textdownarrow & MGM obj. \textdownarrow & Runtime Sync \textdownarrow \\
\midrule
GREEDA (Sync)(best) & \bfseries -126659.00 & \bfseries 0.00 & \bfseries -25051626.27 & \bfseries 95.565\\
GREEDA (Sync)(worst) & -125874.00 & \bfseries 0.00 & -25021733.44 & 121.404\\
MP-T (Sync) & -123966.00 & \bfseries 0.00 & -24666921.25 & 56867.457\\
MST & -- & -- & -- & --\\
MST-Star & -- & -- & -- & --\\
PPI & -- & -- & -- & --\\
Stiefel & -- & -- & -- & --\\
Spectral & -- & -- & -- & --\\
\bottomrule
\end{tabular}
\end{table*}

\begin{table*}
\centering
\caption{\textbf{Synchronization Results for worms-50}. Averaged over all instances. Pre-processing time: 7667.12 s.}
\begin{tabular}{lCCCC}
\toprule
Method & M-LAP obj. \textdownarrow & \#Forb.\textdownarrow & MGM obj. \textdownarrow & Runtime Sync \textdownarrow \\
\midrule
GREEDA (Sync)(best) & \bfseries -544317.50 & \bfseries 0.00 & -20893113.44 & \bfseries 234.614\\
GREEDA (Sync)(worst) & -543133.00 & \bfseries 0.00 & -20840082.94 & 242.706\\
MP-T (Sync) & -542699.75 & \bfseries 0.00 & \bfseries -20897573.95 & 241885.890\\
MST & -- & -- & -- & --\\
MST-Star & -- & -- & -- & --\\
PPI & -- & -- & -- & --\\
Stiefel & -- & -- & -- & --\\
Spectral & -- & -- & -- & --\\
\bottomrule
\end{tabular}
\end{table*}

\begin{table*}
\centering
\caption{\textbf{Synchronization Results for worms-100}. Averaged over all instances. Pre-processing time: 30520.96 s.}
\begin{tabular}{lCCCC}
\toprule
Method & M-LAP obj. \textdownarrow & \#Forb.\textdownarrow & MGM obj. \textdownarrow & Runtime Sync \textdownarrow \\
\midrule
GREEDA (Sync)(best) & \bfseries -2184187.00 & \bfseries 0.00 & \bfseries -84047397.13 & \bfseries 1736.458\\
GREEDA (Sync)(worst) & -2179132.00 & \bfseries 0.00 & -83951621.19 & 1798.240\\
MST & -- & -- & -- & --\\
MST-Star & -- & -- & -- & --\\
PPI & -- & -- & -- & --\\
Stiefel & -- & -- & -- & --\\
Spectral & -- & -- & -- & --\\
\bottomrule
\end{tabular}
\end{table*}

\FloatBarrier
\newpage
\section{Detailed Objective Values per Instance}
\label{suppl:detailed-results}

In most cases, the best and second-best results have been obtained by \emph{GREEDA (best)} and \emph{GREEDA (worst)} computed based on 10 runs of our sequential variant of \emph{GREEDA}. 
Note that these runs can be executed in parallel within the same time, and even their sequential execution often takes less time than a single execution of the competing algorithms. 
This, in turn, means that achieving even \emph{GREEDA (best)} results is computationally cheaper than executing the competing algorithms.

In the tables, \term{GREEDA-C} refers to the intermediate result of \emph{GREEDA}, obtained by our sequential construction heuristic of \cref{sec:construction}, and serves as a point of reference for the influence of the \emph{GM local search} (see \cref{sec:qap-local-search}) and \emph{SWAP local search} (see \cref{sec:swap-local-search}) on the final results of \emph{GREEDA}.

The algorithms \emph{DS*}~\cite{bernard2018ds} and the \emph{pygmtools}~\cite{wang2024pygmtools} implementations of \emph{Cao}~\cite{yan2015multi} and \emph{Floyd}~\cite{jiang2020unifying} are designed for complete MGM problems and were therefore evaluated only on the \emph{synthetic complete/density/deform} datasets. 
On all other datasets, transforming the problem into a complete MGM as described in \cref{sec:incomplete-transformation,suppl:details_incomplete_to_complete_transformation} makes these algorithms impractical, both in terms of runtime and memory required.
Furthermore, \emph{Cao} and \emph{Floyd} do not guarantee cycle-consistent solutions. Results that are cycle-inconsistent are marked with a dagger ($\dagger$) and excluded from the selection of the best result, which is indicated in bold.

On the \emph{worms} and \term{worms-large} datasets, most synchronization algorithms were evaluated only on \term{worms-10}, as they in general produce infeasible solutions with infinite cost.

We report out of memory (OOM) if an algorithm exceeds the evaluation machine's memory limit of $236$\,GB and out of time (OOT) if an algorithm did not finish within $96$\,h.

\bgroup
\setlength{\tabcolsep}{5pt}

\begin{table*}
	\centering
	\caption{\textbf{Synthetic complete (4 objects)}. Objective values per instance.}
	\begin{adjustbox}{width=\textwidth}
\begin{tabular}{lCCCCCCCCCC}
\toprule
& \multicolumn{10}{c}{Instance} \\
\cmidrule{2-11}
Method & 1 & 2 & 3 & 4 & 5 & 6 & 7 & 8 & 9 & 10 \\
\midrule
GREEDA (best) & \bfseries -558.48 & \bfseries -546.60 & \bfseries -552.41 & \bfseries -559.55 & \bfseries -558.49 & \bfseries -553.06 & \bfseries -543.75 & \bfseries -547.87 & \bfseries -553.29 & \bfseries -554.76 \\
GREEDA (worst) & \bfseries -558.48 & \bfseries -546.60 & \bfseries -552.41 & \bfseries -559.55 & \bfseries -558.49 & \bfseries -553.06 & \bfseries -543.75 & \bfseries -547.87 & \bfseries -553.29 & \bfseries -554.76 \\
GREEDA (Sync)(best) & \bfseries -558.48 & \bfseries -546.60 & \bfseries -552.41 & \bfseries -559.55 & \bfseries -558.49 & \bfseries -553.06 & \bfseries -543.75 & \bfseries -547.87 & \bfseries -553.29 & \bfseries -554.76 \\
GREEDA (Sync)(worst) & \bfseries -558.48 & \bfseries -546.60 & \bfseries -552.41 & \bfseries -559.55 & \bfseries -558.49 & \bfseries -553.06 & \bfseries -543.75 & \bfseries -547.87 & \bfseries -553.29 & \bfseries -554.76 \\
GREEDA-C (best) & \bfseries -558.48 & \bfseries -546.60 & \bfseries -552.41 & \bfseries -559.55 & \bfseries -558.49 & \bfseries -553.06 & \bfseries -543.75 & \bfseries -547.87 & \bfseries -553.29 & \bfseries -554.76 \\
GREEDA-C (worst) & \bfseries -558.48 & \bfseries -546.60 & \bfseries -552.41 & \bfseries -559.55 & \bfseries -558.49 & \bfseries -553.06 & \bfseries -543.75 & \bfseries -547.87 & \bfseries -553.29 & \bfseries -554.76 \\
\hline
\bottomrule
\end{tabular}
	\end{adjustbox}
\end{table*}

\begin{table*}
	\centering
	\caption{\textbf{Synthetic complete (8 objects)}. Objective values per instance.}
	\begin{adjustbox}{width=\textwidth}
		\begin{tabular}{lCCCCCCCCCC}
\toprule
& \multicolumn{10}{c}{Instance} \\
\cmidrule{2-11}
Method & 1 & 2 & 3 & 4 & 5 & 6 & 7 & 8 & 9 & 10 \\
\midrule
GREEDA (best) & \bfseries -2587.65 & \bfseries -2565.68 & \bfseries -2585.59 & \bfseries -2562.17 & \bfseries -2610.60 & \bfseries -2593.88 & \bfseries -2558.20 & \bfseries -2583.81 & \bfseries -2592.51 & \bfseries -2565.65 \\
GREEDA (worst) & \bfseries -2587.65 & \bfseries -2565.68 & \bfseries -2585.59 & \bfseries -2562.17 & \bfseries -2610.60 & \bfseries -2593.88 & \bfseries -2558.20 & \bfseries -2583.81 & \bfseries -2592.51 & \bfseries -2565.65 \\
GREEDA (Sync)(best) & \bfseries -2587.65 & \bfseries -2565.68 & \bfseries -2585.59 & \bfseries -2562.17 & \bfseries -2610.60 & \bfseries -2593.88 & \bfseries -2558.20 & \bfseries -2583.81 & \bfseries -2592.51 & \bfseries -2565.65 \\
GREEDA (Sync)(worst) & \bfseries -2587.65 & \bfseries -2565.68 & \bfseries -2585.59 & \bfseries -2562.17 & \bfseries -2610.60 & \bfseries -2593.88 & \bfseries -2558.20 & \bfseries -2583.81 & \bfseries -2592.51 & \bfseries -2565.65 \\
GREEDA-C (best) & \bfseries -2587.65 & \bfseries -2565.68 & \bfseries -2585.59 & \bfseries -2562.17 & \bfseries -2610.60 & \bfseries -2593.88 & \bfseries -2558.20 & \bfseries -2583.81 & \bfseries -2592.51 & \bfseries -2565.65 \\
GREEDA-C (worst) & \bfseries -2587.65 & \bfseries -2565.68 & \bfseries -2585.59 & \bfseries -2562.17 & \bfseries -2610.60 & \bfseries -2593.88 & \bfseries -2558.20 & \bfseries -2583.81 & \bfseries -2592.51 & \bfseries -2565.65 \\
\hline
\bottomrule
\end{tabular}
	\end{adjustbox}
\end{table*}

\begin{table*}
	\centering
	\caption{\textbf{Synthetic complete (12 objects)}. Objective values per instance.}
	\begin{adjustbox}{width=\textwidth}
		\begin{tabular}{lCCCCCCCCCC}
\toprule
& \multicolumn{10}{c}{Instance} \\
\cmidrule{2-11}
Method & 1 & 2 & 3 & 4 & 5 & 6 & 7 & 8 & 9 & 10 \\
\midrule
GREEDA (best) & \bfseries -6092.97 & \bfseries -6039.32 & \bfseries -6107.78 & \bfseries -6052.55 & \bfseries -6121.11 & \bfseries -6109.68 & \bfseries -6012.00 & \bfseries -6082.33 & \bfseries -6089.19 & \bfseries -6058.51 \\
GREEDA (worst) & \bfseries -6092.97 & \bfseries -6039.32 & \bfseries -6107.78 & \bfseries -6052.55 & \bfseries -6121.11 & \bfseries -6109.68 & \bfseries -6012.00 & \bfseries -6082.33 & \bfseries -6089.19 & \bfseries -6058.51 \\
GREEDA (Sync)(best) & \bfseries -6092.97 & \bfseries -6039.32 & \bfseries -6107.78 & \bfseries -6052.55 & \bfseries -6121.11 & \bfseries -6109.68 & \bfseries -6012.00 & \bfseries -6082.33 & \bfseries -6089.19 & \bfseries -6058.51 \\
GREEDA (Sync)(worst) & \bfseries -6092.97 & \bfseries -6039.32 & \bfseries -6107.78 & \bfseries -6052.55 & \bfseries -6121.11 & \bfseries -6109.68 & \bfseries -6012.00 & \bfseries -6082.33 & \bfseries -6089.19 & \bfseries -6058.51 \\
GREEDA-C (best) & \bfseries -6092.97 & \bfseries -6039.32 & \bfseries -6107.78 & \bfseries -6052.55 & \bfseries -6121.11 & \bfseries -6109.68 & \bfseries -6012.00 & \bfseries -6082.33 & \bfseries -6089.19 & \bfseries -6058.51 \\
GREEDA-C (worst) & \bfseries -6092.97 & \bfseries -6039.32 & \bfseries -6107.78 & \bfseries -6052.55 & \bfseries -6121.11 & \bfseries -6109.68 & \bfseries -6012.00 & \bfseries -6082.33 & \bfseries -6089.19 & \bfseries -6058.51 \\
\hline
\bottomrule
\end{tabular}
	\end{adjustbox}
\end{table*}

\begin{table*}
	\centering
	\caption{\textbf{Synthetic complete (16 objects)}. Objective values per instance.}
	\begin{adjustbox}{width=\textwidth}
		\begin{tabular}{lCCCCCCCCCC}
\toprule
& \multicolumn{10}{c}{Instance} \\
\cmidrule{2-11}
Method & 1 & 2 & 3 & 4 & 5 & 6 & 7 & 8 & 9 & 10 \\
\midrule
GREEDA (best) & \bfseries -11080.96 & \bfseries -11021.58 & \bfseries -11059.12 & \bfseries -10953.02 & \bfseries -11142.22 & \bfseries -11079.59 & \bfseries -10965.60 & \bfseries -11096.36 & \bfseries -11095.05 & \bfseries -10995.11 \\
GREEDA (worst) & \bfseries -11080.96 & \bfseries -11021.58 & \bfseries -11059.12 & \bfseries -10953.02 & \bfseries -11142.22 & \bfseries -11079.59 & \bfseries -10965.60 & \bfseries -11096.36 & \bfseries -11095.05 & \bfseries -10995.11 \\
GREEDA (Sync)(best) & \bfseries -11080.96 & \bfseries -11021.58 & \bfseries -11059.12 & \bfseries -10953.02 & \bfseries -11142.22 & \bfseries -11079.59 & \bfseries -10965.60 & \bfseries -11096.36 & \bfseries -11095.05 & \bfseries -10995.11 \\
GREEDA (Sync)(worst) & \bfseries -11080.96 & \bfseries -11021.58 & \bfseries -11059.12 & \bfseries -10953.02 & \bfseries -11142.22 & \bfseries -11079.59 & \bfseries -10965.60 & \bfseries -11096.36 & \bfseries -11095.05 & \bfseries -10995.11 \\
GREEDA-C (best) & \bfseries -11080.96 & \bfseries -11021.58 & \bfseries -11059.12 & \bfseries -10953.02 & \bfseries -11142.22 & \bfseries -11079.59 & \bfseries -10965.60 & \bfseries -11096.36 & \bfseries -11095.05 & \bfseries -10995.11 \\
GREEDA-C (worst) & \bfseries -11080.96 & \bfseries -11021.58 & \bfseries -11059.12 & \bfseries -10953.02 & \bfseries -11142.22 & \bfseries -11079.59 & \bfseries -10965.60 & \bfseries -11096.36 & \bfseries -11095.05 & \bfseries -10995.11 \\
\hline
\bottomrule
\end{tabular}
	\end{adjustbox}
\end{table*}

\begin{table*}
	\centering
	\caption{\textbf{Synthetic deform (4 objects)}. Objective values per instance.}
	\begin{adjustbox}{width=\textwidth}
		\begin{tabular}{lCCCCCCCCCC}
\toprule
& \multicolumn{10}{c}{Instance} \\
\cmidrule{2-11}
Method & 1 & 2 & 3 & 4 & 5 & 6 & 7 & 8 & 9 & 10 \\
\midrule
GREEDA (best) & \bfseries -342.83 & \bfseries -321.97 & \bfseries -292.49 & \bfseries -337.38 & \bfseries -320.10 & \bfseries -324.16 & \bfseries -328.66 & \bfseries -342.18 & \bfseries -315.66 & \bfseries -333.06 \\
GREEDA (worst) & -335.20 & \bfseries -321.97 & -289.34 & -325.09 & -310.72 & -317.34 & -320.85 & -331.55 & -307.92 & -309.55 \\
GREEDA (Sync)(best) & -276.88 & -312.33 & -238.74 & -301.58 & -237.11 & -262.74 & -301.77 & -292.78 & -293.50 & -331.06 \\
GREEDA (Sync)(worst) & -277.00 & -289.34 & -214.11 & -301.58 & -265.82 & -278.22 & -294.38 & -299.98 & -264.62 & -331.06 \\
GREEDA-C (best) & \bfseries -342.83 & \bfseries -321.97 & -291.64 & -334.25 & -313.84 & \bfseries -324.16 & -327.82 & -336.27 & -314.52 & \bfseries -333.06 \\
GREEDA-C (worst) & -323.13 & -303.98 & -288.49 & -321.49 & -307.63 & -315.68 & -308.08 & -331.31 & -301.24 & -309.55 \\
\hline
\bottomrule
\end{tabular}
	\end{adjustbox}
\end{table*}

\begin{table*}
	\centering
	\caption{\textbf{Synthetic deform (8 objects)}. Objective values per instance.}
	\begin{adjustbox}{width=\textwidth}
		\begin{tabular}{lCCCCCCCCCC}
\toprule
& \multicolumn{10}{c}{Instance} \\
\cmidrule{2-11}
Method & 1 & 2 & 3 & 4 & 5 & 6 & 7 & 8 & 9 & 10 \\
\midrule
GREEDA (best) & \bfseries -1523.81 & \bfseries -1456.37 & \bfseries -1370.43 & \bfseries -1458.59 & \bfseries -1510.40 & \bfseries -1478.13 & \bfseries -1501.71 & \bfseries -1544.77 & \bfseries -1536.48 & \bfseries -1484.89 \\
GREEDA (worst) & -1387.04 & \bfseries -1456.37 & -1329.65 & -1405.29 & -1378.37 & -1426.58 & \bfseries -1501.71 & -1459.14 & -1406.44 & -1442.26 \\
GREEDA (Sync)(best) & -1338.05 & -1305.97 & -1203.33 & -1274.75 & -1240.10 & -1188.23 & -1352.27 & -1291.82 & -1461.57 & -1374.09 \\
GREEDA (Sync)(worst) & -1223.47 & -1264.77 & -1060.45 & -1325.60 & -1206.95 & -1123.49 & -1264.02 & -1220.99 & -1337.67 & -1246.10 \\
GREEDA-C (best) & \bfseries -1523.81 & -1450.39 & -1342.89 & -1456.61 & \bfseries -1510.40 & \bfseries -1478.13 & -1488.93 & -1540.45 & -1510.82 & -1440.19 \\
GREEDA-C (worst) & -1381.87 & -1319.96 & -1305.38 & -1368.23 & -1365.10 & -1421.98 & -1377.28 & -1436.76 & -1395.48 & -1348.48 \\
\hline
\bottomrule
\end{tabular}
	\end{adjustbox}
\end{table*}

\begin{table*}
	\centering
	\caption{\textbf{Synthetic deform (12 objects)}. Objective values per instance.}
	\begin{adjustbox}{width=\textwidth}
		\begin{tabular}{lCCCCCCCCCC}
\toprule
& \multicolumn{10}{c}{Instance} \\
\cmidrule{2-11}
Method & 1 & 2 & 3 & 4 & 5 & 6 & 7 & 8 & 9 & 10 \\
\midrule
GREEDA (best) & \bfseries -3642.08 & \bfseries -3477.52 & \bfseries -3336.99 & \bfseries -3467.06 & \bfseries -3498.74 & \bfseries -3478.99 & \bfseries -3480.48 & \bfseries -3650.50 & \bfseries -3522.58 & \bfseries -3528.56 \\
GREEDA (worst) & -3342.53 & -3156.30 & -3100.01 & -3222.46 & -3118.22 & -3332.47 & \bfseries -3480.48 & -3328.68 & -3239.89 & -3385.65 \\
GREEDA (Sync)(best) & -3151.71 & -3216.88 & -2607.96 & -3013.25 & -3355.47 & -2905.49 & -3360.98 & -3390.49 & -3178.79 & -3421.03 \\
GREEDA (Sync)(worst) & -3283.79 & -3177.51 & -2609.88 & -2795.71 & -3349.43 & -2615.74 & -2993.56 & -3253.29 & -3042.21 & -3325.71 \\
GREEDA-C (best) & -3637.43 & -3474.32 & -3164.47 & -3374.96 & -3423.41 & -3422.84 & -3451.42 & -3644.58 & -3418.37 & -3444.29 \\
GREEDA-C (worst) & -3321.13 & -3123.26 & -3080.38 & -3160.84 & -3070.59 & -3311.89 & -3293.50 & -3262.86 & -3186.95 & -3227.73 \\
\hline
\bottomrule
\end{tabular}
	\end{adjustbox}
\end{table*}

\begin{table*}
	\centering
	\caption{\textbf{Synthetic deform (16 objects)}. Objective values per instance.}
	\begin{adjustbox}{width=\textwidth}
		\begin{tabular}{lCCCCCCCCCC}
\toprule
& \multicolumn{10}{c}{Instance} \\
\cmidrule{2-11}
Method & 1 & 2 & 3 & 4 & 5 & 6 & 7 & 8 & 9 & 10 \\
\midrule
GREEDA (best) & \bfseries -6507.40 & \bfseries -6569.81 & \bfseries -6164.47 & \bfseries -6278.85 & \bfseries -6272.24 & \bfseries -6349.59 & \bfseries -6231.85 & \bfseries -6615.35 & \bfseries -6411.24 & \bfseries -6561.48 \\
GREEDA (worst) & -6491.59 & \bfseries -6569.81 & -5651.26 & -5850.96 & \bfseries -6272.24 & -5989.36 & -5831.64 & \bfseries -6615.35 & -5979.06 & -6557.37 \\
GREEDA (Sync)(best) & -5625.95 & -6476.52 & -4905.23 & -5571.68 & -6021.39 & -5111.74 & -5788.84 & -6186.30 & -5819.38 & -6436.94 \\
GREEDA (Sync)(worst) & -5407.12 & -6548.72 & -4644.90 & -5078.29 & -6021.39 & -4826.56 & -5430.70 & -6341.15 & -5576.68 & -6330.97 \\
GREEDA-C (best) & \bfseries -6507.40 & \bfseries -6569.81 & -6150.18 & -6208.06 & -6205.33 & -6137.90 & -6178.78 & \bfseries -6615.35 & -6400.05 & -6449.65 \\
GREEDA-C (worst) & -5826.41 & -5720.11 & -5572.70 & -5653.17 & -5437.93 & -5961.16 & -5707.70 & -5905.28 & -5862.12 & -5929.07 \\
\hline
\bottomrule
\end{tabular}
	\end{adjustbox}
\end{table*}

\begin{table*}
	\centering
	\caption{\textbf{Synthetic density (4 objects)}. Objective values per instance.}
	\begin{adjustbox}{width=\textwidth}
		\begin{tabular}{lCCCCCCCCCC}
\toprule
& \multicolumn{10}{c}{Instance} \\
\cmidrule{2-11}
Method & 1 & 2 & 3 & 4 & 5 & 6 & 7 & 8 & 9 & 10 \\
\midrule
GREEDA (best) & \bfseries -180.66 & \bfseries -222.06 & \bfseries -202.26 & \bfseries -232.06 & \bfseries -176.74 & \bfseries -213.24 & \bfseries -208.91 & \bfseries -228.06 & \bfseries -210.06 & \bfseries -232.46 \\
GREEDA (worst) & -176.58 & -211.97 & -196.23 & -215.53 & -171.30 & \bfseries -213.24 & -198.67 & -218.64 & -203.48 & \bfseries -232.46 \\
GREEDA (Sync)(best) & -144.55 & \bfseries -222.06 & -182.26 & -206.73 & -131.47 & -199.49 & -186.17 & -221.14 & -172.43 & \bfseries -232.46 \\
GREEDA (Sync)(worst) & -141.05 & \bfseries -222.06 & -176.11 & -206.73 & -126.68 & -199.49 & -186.17 & -221.14 & -154.56 & -215.52 \\
GREEDA-C (best) & \bfseries -180.66 & \bfseries -222.06 & \bfseries -202.26 & -231.14 & -172.88 & \bfseries -213.24 & \bfseries -208.91 & \bfseries -228.06 & -204.68 & \bfseries -232.46 \\
GREEDA-C (worst) & -171.51 & -208.77 & -194.25 & -212.81 & -171.30 & -203.43 & -193.44 & -214.07 & -202.48 & -213.86 \\
\hline
\bottomrule
\end{tabular}
	\end{adjustbox}
\end{table*}

\begin{table*}
	\centering
	\caption{\textbf{Synthetic density (8 objects)}. Objective values per instance.}
	\begin{adjustbox}{width=\textwidth}
		\begin{tabular}{lCCCCCCCCCC}
\toprule
& \multicolumn{10}{c}{Instance} \\
\cmidrule{2-11}
Method & 1 & 2 & 3 & 4 & 5 & 6 & 7 & 8 & 9 & 10 \\
\midrule
GREEDA (best) & \bfseries -793.43 & \bfseries -970.22 & \bfseries -933.74 & \bfseries -920.62 & \bfseries -816.93 & \bfseries -865.71 & \bfseries -956.28 & \bfseries -1062.28 & \bfseries -938.28 & \bfseries -990.28 \\
GREEDA (worst) & -764.35 & -902.21 & -892.44 & -874.62 & -766.18 & -838.54 & -904.83 & \bfseries -1062.28 & \bfseries -938.28 & -894.50 \\
GREEDA (Sync)(best) & -610.99 & -930.85 & -819.80 & -902.99 & -643.42 & -774.07 & -852.69 & \bfseries -1062.28 & -785.49 & -941.73 \\
GREEDA (Sync)(worst) & -608.03 & -930.85 & -790.04 & -877.26 & -634.10 & -691.23 & -790.41 & \bfseries -1062.28 & -816.06 & -885.14 \\
GREEDA-C (best) & -784.74 & -969.81 & -912.03 & -919.80 & -800.45 & -855.02 & -917.57 & \bfseries -1062.28 & -915.77 & -987.18 \\
GREEDA-C (worst) & -746.59 & -879.15 & -846.54 & -832.28 & -761.45 & -826.03 & -870.25 & -987.49 & -850.59 & -890.23 \\
\hline
\bottomrule
\end{tabular}
	\end{adjustbox}
\end{table*}

\begin{table*}
	\centering
	\caption{\textbf{Synthetic density (12 objects)}. Objective values per instance.}
	\begin{adjustbox}{width=\textwidth}
		\begin{tabular}{lCCCCCCCCCC}
\toprule
& \multicolumn{10}{c}{Instance} \\
\cmidrule{2-11}
Method & 1 & 2 & 3 & 4 & 5 & 6 & 7 & 8 & 9 & 10 \\
\midrule
GREEDA (best) & \bfseries -1941.66 & \bfseries -2326.68 & \bfseries -2075.42 & \bfseries -2174.66 & \bfseries -2006.66 & \bfseries -2023.55 & \bfseries -2162.66 & \bfseries -2374.66 & \bfseries -2298.66 & \bfseries -2350.66 \\
GREEDA (worst) & -1825.29 & \bfseries -2326.68 & -1939.62 & -2071.96 & -1855.79 & -1905.93 & -1994.70 & \bfseries -2374.66 & -2029.63 & \bfseries -2350.66 \\
GREEDA (Sync)(best) & -1757.44 & \bfseries -2326.68 & -1761.23 & -2146.04 & -1742.02 & -1788.76 & -2142.07 & \bfseries -2374.66 & -2214.01 & -2308.66 \\
GREEDA (Sync)(worst) & -1613.47 & \bfseries -2326.68 & -1633.19 & -2103.41 & -1692.18 & -1504.31 & -2042.50 & \bfseries -2374.66 & -2214.01 & -2250.67 \\
GREEDA-C (best) & -1891.98 & \bfseries -2326.68 & -2011.05 & -2088.83 & -1954.22 & -1981.40 & -2151.95 & \bfseries -2374.66 & -2229.22 & \bfseries -2350.66 \\
GREEDA-C (worst) & -1795.56 & -2095.35 & -1925.60 & -1947.70 & -1841.22 & -1865.75 & -1941.44 & -2046.83 & -1968.98 & -2187.98 \\
\hline
\bottomrule
\end{tabular}
	\end{adjustbox}
\end{table*}

\begin{table*}
	\centering
	\caption{\textbf{Synthetic density (16 objects)}. Objective values per instance.}
	\begin{adjustbox}{width=\textwidth}
		\begin{tabular}{lCCCCCCCCCC}
\toprule
& \multicolumn{10}{c}{Instance} \\
\cmidrule{2-11}
Method & 1 & 2 & 3 & 4 & 5 & 6 & 7 & 8 & 9 & 10 \\
\midrule
GREEDA (best) & \bfseries -3615.19 & \bfseries -4271.20 & \bfseries -3855.68 & \bfseries -3934.51 & -3587.26 & \bfseries -3655.63 & \bfseries -3878.61 & \bfseries -4235.20 & \bfseries -4319.20 & \bfseries -4203.20 \\
GREEDA (worst) & -3264.99 & \bfseries -4271.20 & -3604.64 & -3706.78 & -3275.98 & -3408.78 & -3743.23 & \bfseries -4235.20 & \bfseries -4319.20 & \bfseries -4203.20 \\
GREEDA (Sync)(best) & -3421.91 & \bfseries -4271.20 & -3476.26 & -3888.44 & -3465.11 & -3249.39 & -3858.79 & \bfseries -4235.20 & -4308.86 & \bfseries -4203.20 \\
GREEDA (Sync)(worst) & -3402.87 & \bfseries -4271.20 & -3321.71 & -3846.02 & -3152.67 & -2711.32 & -3858.79 & \bfseries -4235.20 & -4308.86 & \bfseries -4203.20 \\
GREEDA-C (best) & -3566.60 & -4249.50 & -3689.37 & -3873.16 & -3482.57 & -3572.54 & -3862.13 & \bfseries -4235.20 & -4318.04 & -4182.95 \\
GREEDA-C (worst) & -3225.58 & -3727.70 & -3528.64 & -3450.36 & -3243.37 & -3371.54 & -3452.96 & -3627.49 & -3820.53 & -3681.47 \\
\hline
\bottomrule
\end{tabular}
	\end{adjustbox}
\end{table*}

\begin{table*}
	\centering
	\caption{\textbf{Synthetic outlier (4 objects)}. Objective values per instance.}
	\begin{adjustbox}{width=\textwidth}
		\begin{tabular}{lCCCCCCCCCC}
\toprule
& \multicolumn{10}{c}{Instance} \\
\cmidrule{2-11}
Method & 1 & 2 & 3 & 4 & 5 & 6 & 7 & 8 & 9 & 10 \\
\midrule
GREEDA (best) & \bfseries -43.28 & \bfseries -43.22 & \bfseries -43.22 & \bfseries -43.28 & \bfseries -43.80 & \bfseries -43.93 & \bfseries -43.22 & \bfseries -43.22 & \bfseries -43.22 & \bfseries -43.22 \\
GREEDA (worst) & \bfseries -43.28 & \bfseries -43.22 & \bfseries -43.22 & \bfseries -43.28 & \bfseries -43.80 & \bfseries -43.93 & \bfseries -43.22 & \bfseries -43.22 & \bfseries -43.22 & \bfseries -43.22 \\
GREEDA (Sync)(best) & 114.27 & 110.54 & 103.01 & 101.21 & 126.88 & 84.20 & 123.17 & 110.09 & 116.72 & 100.86 \\
GREEDA (Sync)(worst) & 114.27 & 110.54 & 103.01 & 101.21 & 126.88 & 84.20 & 123.17 & 110.09 & 116.72 & 100.86 \\
GREEDA-C (best) & \bfseries -43.28 & \bfseries -43.22 & \bfseries -43.22 & \bfseries -43.28 & \bfseries -43.80 & \bfseries -43.93 & \bfseries -43.22 & \bfseries -43.22 & \bfseries -43.22 & \bfseries -43.22 \\
GREEDA-C (worst) & -43.22 & \bfseries -43.22 & \bfseries -43.22 & -43.22 & \bfseries -43.80 & \bfseries -43.93 & \bfseries -43.22 & \bfseries -43.22 & \bfseries -43.22 & \bfseries -43.22 \\
\hline
\bottomrule
\end{tabular}
	\end{adjustbox}
\end{table*}

\begin{table*}
	\centering
	\caption{\textbf{Synthetic outlier (8 objects)}. Objective values per instance.}
	\begin{adjustbox}{width=\textwidth}
		\begin{tabular}{lCCCCCCCCCC}
\toprule
& \multicolumn{10}{c}{Instance} \\
\cmidrule{2-11}
Method & 1 & 2 & 3 & 4 & 5 & 6 & 7 & 8 & 9 & 10 \\
\midrule
GREEDA (best) & \bfseries -202.26 & \bfseries -201.70 & \bfseries -202.40 & \bfseries -202.31 & \bfseries -202.67 & \bfseries -204.36 & \bfseries -203.10 & \bfseries -202.33 & \bfseries -201.70 & \bfseries -201.75 \\
GREEDA (worst) & \bfseries -202.26 & \bfseries -201.70 & \bfseries -202.40 & \bfseries -202.31 & \bfseries -202.67 & \bfseries -204.36 & \bfseries -203.10 & \bfseries -202.33 & \bfseries -201.70 & \bfseries -201.75 \\
GREEDA (Sync)(best) & 548.77 & 513.28 & 522.24 & 534.65 & 552.76 & 450.82 & 534.26 & 524.21 & 582.42 & 539.49 \\
GREEDA (Sync)(worst) & 548.77 & 513.28 & 522.24 & 534.65 & 552.76 & 450.82 & 534.26 & 524.21 & 582.42 & 539.49 \\
GREEDA-C (best) & \bfseries -202.26 & \bfseries -201.70 & \bfseries -202.40 & -202.13 & \bfseries -202.67 & -204.21 & \bfseries -203.10 & \bfseries -202.33 & \bfseries -201.70 & -201.70 \\
GREEDA-C (worst) & -202.21 & \bfseries -201.70 & \bfseries -202.40 & -202.13 & -202.28 & -203.75 & \bfseries -203.10 & \bfseries -202.33 & \bfseries -201.70 & -201.70 \\
\hline
\bottomrule
\end{tabular}
	\end{adjustbox}
\end{table*}

\begin{table*}
	\centering
	\caption{\textbf{Synthetic outlier (12 objects)}. Objective values per instance.}
	\begin{adjustbox}{width=\textwidth}
		\begin{tabular}{lCCCCCCCCCC}
\toprule
& \multicolumn{10}{c}{Instance} \\
\cmidrule{2-11}
Method & 1 & 2 & 3 & 4 & 5 & 6 & 7 & 8 & 9 & 10 \\
\midrule
GREEDA (best) & \bfseries -476.28 & \bfseries -475.47 & \bfseries -476.74 & \bfseries -476.47 & \bfseries -477.54 & \bfseries -478.24 & \bfseries -477.09 & \bfseries -476.56 & \bfseries -475.89 & \bfseries -475.49 \\
GREEDA (worst) & \bfseries -476.28 & \bfseries -475.47 & \bfseries -476.74 & \bfseries -476.47 & \bfseries -477.54 & \bfseries -478.24 & \bfseries -477.09 & \bfseries -476.56 & \bfseries -475.89 & \bfseries -475.49 \\
GREEDA (Sync)(best) & 1288.41 & 1264.39 & 1243.46 & 1254.83 & 1262.77 & 1171.50 & 1240.49 & 1277.77 & 1254.45 & 1179.73 \\
GREEDA (Sync)(worst) & 1288.41 & 1264.39 & 1243.46 & 1254.83 & 1262.77 & 1171.50 & 1240.49 & 1277.77 & 1254.45 & 1179.73 \\
GREEDA-C (best) & \bfseries -476.28 & -475.44 & -476.50 & -476.33 & -477.43 & -477.93 & \bfseries -477.09 & \bfseries -476.56 & \bfseries -475.89 & -475.44 \\
GREEDA-C (worst) & -475.95 & -475.44 & -476.14 & -475.84 & -477.04 & -477.49 & -476.84 & \bfseries -476.56 & \bfseries -475.89 & -475.44 \\
\hline
\bottomrule
\end{tabular}
	\end{adjustbox}
\end{table*}

\begin{table*}
	\centering
	\caption{\textbf{Synthetic outlier (16 objects)}. Objective values per instance.}
	\begin{adjustbox}{width=\textwidth}
		\begin{tabular}{lCCCCCCCCCC}
\toprule
& \multicolumn{10}{c}{Instance} \\
\cmidrule{2-11}
Method & 1 & 2 & 3 & 4 & 5 & 6 & 7 & 8 & 9 & 10 \\
\midrule
GREEDA (best) & \bfseries -866.71 & \bfseries -866.75 & \bfseries -866.49 & \bfseries -867.03 & \bfseries -867.90 & \bfseries -867.24 & \bfseries -866.42 & \bfseries -867.69 & \bfseries -864.88 & \bfseries -867.05 \\
GREEDA (worst) & \bfseries -866.71 & \bfseries -866.75 & \bfseries -866.49 & \bfseries -867.03 & \bfseries -867.90 & \bfseries -867.24 & \bfseries -866.42 & \bfseries -867.69 & \bfseries -864.88 & \bfseries -867.05 \\
GREEDA (Sync)(best) & 2331.70 & 2242.22 & 2226.52 & 2250.54 & 2299.76 & 2262.98 & 2345.80 & 2306.28 & 2332.84 & 2187.59 \\
GREEDA (Sync)(worst) & 2331.70 & 2242.22 & 2226.52 & 2250.54 & 2299.76 & 2262.98 & 2345.80 & 2306.28 & 2332.84 & 2187.59 \\
GREEDA-C (best) & -866.60 & -866.70 & -866.27 & -866.93 & -867.79 & -866.78 & \bfseries -866.42 & \bfseries -867.69 & \bfseries -864.88 & -867.00 \\
GREEDA-C (worst) & -866.32 & -865.76 & -865.89 & -866.16 & -866.92 & -866.48 & -865.84 & -866.93 & \bfseries -864.88 & -867.00 \\
\hline
\bottomrule
\end{tabular}
	\end{adjustbox}
\end{table*}

\begin{table*}[!ht]
	\centering
	\caption{\textbf{Hotel (4 objects)}. Objective values per instance.}
	\begin{adjustbox}{width=\textwidth}
		\begin{tabular}{lCCCCCCCCCC}
\toprule
& \multicolumn{10}{c}{Instance} \\
\cmidrule{2-11}
Method & 1 & 2 & 3 & 4 & 5 & 6 & 7 & 8 & 9 & 10 \\
\midrule
GREEDA (best) & \bfseries -13.87 & \bfseries -25.75 & \bfseries -25.06 & \bfseries -25.78 & \bfseries -12.36 & \bfseries -7.89 & \bfseries -20.20 & \bfseries -26.12 & \bfseries -16.67 & \bfseries -26.96 \\
GREEDA (worst) & -13.71 & \bfseries -25.75 & -24.99 & \bfseries -25.78 & -10.20 & -6.92 & \bfseries -20.20 & \bfseries -26.12 & -16.38 & \bfseries -26.96 \\
GREEDA (Sync)(best) & 247.99 & 162.06 & 169.95 & 145.88 & 256.93 & 263.76 & 139.34 & 129.74 & 175.42 & 122.99 \\
GREEDA (Sync)(worst) & 266.43 & 162.06 & 166.81 & 145.88 & 229.98 & 249.72 & 139.34 & 129.74 & 175.42 & 122.99 \\
GREEDA-C (best) & \bfseries -13.87 & \bfseries -25.75 & \bfseries -25.06 & \bfseries -25.78 & \bfseries -12.36 & -7.75 & \bfseries -20.20 & \bfseries -26.12 & -16.38 & \bfseries -26.96 \\
GREEDA-C (worst) & -8.47 & \bfseries -25.75 & -24.99 & -22.36 & -9.29 & -6.49 & -15.31 & -20.17 & -12.38 & \bfseries -26.96 \\
\hline
\bottomrule
\end{tabular}
	\end{adjustbox}
\end{table*}

\begin{table*}[!ht]
	\centering
	\caption{\textbf{Hotel (8 objects)}. Objective values per instance.}
	\begin{adjustbox}{width=\textwidth}
		\begin{tabular}{lCCCCCCCCCC}
\toprule
& \multicolumn{10}{c}{Instance} \\
\cmidrule{2-11}
Method & 1 & 2 & 3 & 4 & 5 & 6 & 7 & 8 & 9 & 10 \\
\midrule
GREEDA (best) & \bfseries -97.33 & \bfseries -147.92 & \bfseries -81.44 & \bfseries -96.33 & \bfseries -61.84 & \bfseries -69.24 & \bfseries -101.97 & \bfseries -81.95 & \bfseries -69.69 & \bfseries -145.83 \\
GREEDA (worst) & -97.25 & \bfseries -147.92 & \bfseries -81.44 & \bfseries -96.33 & -59.91 & -66.84 & \bfseries -101.97 & -65.59 & \bfseries -69.69 & \bfseries -145.83 \\
GREEDA (Sync)(best) & 707.49 & 636.43 & 893.44 & 711.06 & 856.34 & 787.94 & 699.61 & 679.40 & 717.77 & 601.60 \\
GREEDA (Sync)(worst) & 770.43 & 636.43 & 793.40 & 744.61 & 856.34 & 887.12 & 699.61 & 679.40 & 834.52 & 628.61 \\
GREEDA-C (best) & -97.25 & \bfseries -147.92 & -81.37 & \bfseries -96.33 & -60.40 & -68.98 & \bfseries -101.97 & -75.12 & -68.97 & \bfseries -145.83 \\
GREEDA-C (worst) & -62.78 & -147.50 & -68.87 & -80.19 & -38.61 & -49.06 & -65.60 & -48.99 & -43.03 & \bfseries -145.83 \\
\hline
\bottomrule
\end{tabular}
	\end{adjustbox}
\end{table*}

\begin{table*}[!ht]
	\centering
	\caption{\textbf{Hotel (12 objects)}. Objective values per instance.}
	\begin{adjustbox}{width=\textwidth}
		%

	\end{adjustbox}
\end{table*}

\begin{table*}[!ht]
	\centering
	\caption{\textbf{Hotel (16 objects)}. Objective values per instance.}
	\begin{adjustbox}{width=\textwidth}
		%

	\end{adjustbox}
\end{table*}

\begin{table*}
	\centering
	\caption{\textbf{House (4 objects)}. Objective values per instance.}
	\begin{adjustbox}{width=\textwidth}
		%

	\end{adjustbox}
\end{table*}

\begin{table*}
	\centering
	\caption{\textbf{House (8 objects)}. Objective values per instance.}
	\begin{adjustbox}{width=\textwidth}
		%

	\end{adjustbox}
\end{table*}

\begin{table*}
	\centering
	\caption{\textbf{House (12 objects)}. Objective values per instance.}
	\begin{adjustbox}{width=\textwidth}
		%

	\end{adjustbox}
\end{table*}

\begin{table*}
	\centering
	\caption{\textbf{House (16 objects)}. Objective values per instance.}
	\begin{adjustbox}{width=\textwidth}
		%

	\end{adjustbox}
\end{table*}

\begin{table*}
	\centering
	\caption{\textbf{Worms (3 objects)}. Objective values per instance.}
	\begin{adjustbox}{width=\textwidth}
		%

	\end{adjustbox}
\end{table*}

\begin{table*}
	\centering
	\caption{\textbf{Worms (4 objects)}. Objective values per instance.}
	\begin{adjustbox}{width=\textwidth}
		%

	\end{adjustbox}
\end{table*}

\begin{table*}
	\centering
	\caption{\textbf{Worms (5 objects)}. Objective values per instance.}
	\begin{adjustbox}{width=\textwidth}
		%

	\end{adjustbox}
\end{table*}

\begin{table*}
	\centering
	\caption{\textbf{Worms (6 objects)}. Objective values per instance.}
	\begin{adjustbox}{width=\textwidth}
		%

	\end{adjustbox}
\end{table*}
\egroup

\FloatBarrier
\section{List of Mathematical Symbols}

%
\stopcontents[appendix-toc]

\fi

\end{document}